\documentclass{article}

\usepackage{microtype}
\usepackage{graphicx}
\usepackage{subfigure}
\usepackage{booktabs} 
 
\usepackage{hyperref}
\hypersetup{
   colorlinks=true,
   citecolor=cyan
}

\usepackage[table]{xcolor}
\usepackage[accepted]{icml2022}

\usepackage{amsmath}
\usepackage{amssymb}
\usepackage{mathtools}
\usepackage{amsthm}
\usepackage{thm-restate}

\theoremstyle{plain}
\theoremstyle{definition}
\newtheorem{theorem}{Theorem}[section]
\newtheorem{proposition}[theorem]{Proposition}

\newtheorem{example}[theorem]{Example}
\newtheorem{definition}[theorem]{Definition}

\theoremstyle{remark}

\usepackage{mathrsfs}
\usepackage{enumitem}
\usepackage{nicefrac}
\usepackage{bbm}
\usepackage{tikz-cd}
\usepackage[framemethod=TikZ]{mdframed}
\mdfdefinestyle{box}{%
backgroundcolor=black!6,
roundcorner=10pt,
innertopmargin=8pt,
innerleftmargin=8pt}

\setlist[itemize,0]{leftmargin=0.5cm,itemsep=-0.1cm,topsep=-0.2cm}
\setlist[enumerate,0]{leftmargin=0.75cm,itemsep=-0.1cm,topsep=-0.2cm}

\definecolor{lightblue}{HTML}{89CFF0}
\DeclareMathOperator*{\argmax}{arg\,max}
\DeclareMathSymbol{\indexminus}{\mathbin}{AMSa}{"39}

\usepackage{xspace} 
\newcommand{\gsps}{geometric switching policies\xspace}
\newcommand{\GSPs}{GSPs\xspace}
\newcommand{\gsp}{geometric switching policy\xspace}
\newcommand{\GSP}{GSP\xspace}
\newcommand{\switchprob}{\alpha}
\newcommand{\gspcollection}{\Pi}
\newcommand{\ggpi}{GGPI\xspace}
\newcommand{\lltwotd}{LL2TD\xspace}
\newcommand{\logittoprob}{\mu}

\icmltitlerunning{Generalised Policy Improvement with Geometric Policy Composition}

\begin{document}

\twocolumn[
\icmltitle{Generalised Policy Improvement with Geometric Policy Composition}



\icmlsetsymbol{equal}{*}

\begin{icmlauthorlist}
\icmlauthor{Shantanu Thakoor}{equal,dm}
\icmlauthor{Mark Rowland}{equal,dm}
\icmlauthor{Diana Borsa}{dm}
\icmlauthor{Will Dabney}{dm}
\icmlauthor{R\'emi Munos}{dm}
\icmlauthor{Andr\'e Barreto}{dm}
\end{icmlauthorlist}

\icmlaffiliation{dm}{DeepMind, London}

\icmlcorrespondingauthor{Shantanu Thakoor}{thakoor@deepmind.com}
\icmlcorrespondingauthor{Mark Rowland}{markrowland@deepmind.com}

\icmlkeywords{Machine Learning, ICML, Reinforcement Learning, Planning, Generative Models, Policy Improvement}

\vskip 0.3in
]



\printAffiliationsAndNotice{\icmlEqualContribution} 

\begin{abstract}
    We introduce a method for policy improvement that interpolates between the greedy approach of value-based reinforcement learning (RL) and the full planning approach typical of model-based RL. The new method builds on the concept of a geometric horizon model (GHM, also known as a $\gamma$-model), which models the discounted state-visitation distribution of a given policy.
    We show that we can evaluate any non-Markov policy that switches between a set of base Markov policies with fixed probability by a careful composition of the base policy GHMs, \emph{without any additional learning}.
    We can then apply generalised policy improvement (GPI) to collections of such non-Markov policies to obtain a new Markov policy that will in general outperform its precursors. We provide a thorough theoretical analysis of this approach, develop applications to transfer and standard RL, and empirically demonstrate its effectiveness over standard GPI on a challenging deep RL continuous control task. We also provide an analysis of GHM training methods,
    proving a novel convergence result regarding previously proposed methods 
    and showing how to train these models stably in deep RL settings.
\end{abstract}

\section{Introduction}

Policy improvement is at the heart of reinforcement learning (RL). The prototypical approach to policy improvement in value-based RL is to take the Q-function of a policy and act \emph{greedily} with respect to it. In contrast, in model-based RL, planning with a model in principle aims to derive a (near-)optimal policy directly. Choosing between these two extremes involves some trade-offs. While greedy improvement requires estimating only a Q-function, from which it is computationally trivial to derive the greedy policy, this may result in only a weak improvement over the existing policy. Planning, on the other hand, is a computationally intensive process, yet can yield extremely high-quality policies. In this paper, we introduce an approach to policy improvement that interpolates between these two extremes.

\begin{figure}[t]
    \centering
    \includegraphics[keepaspectratio,width=.47\textwidth]{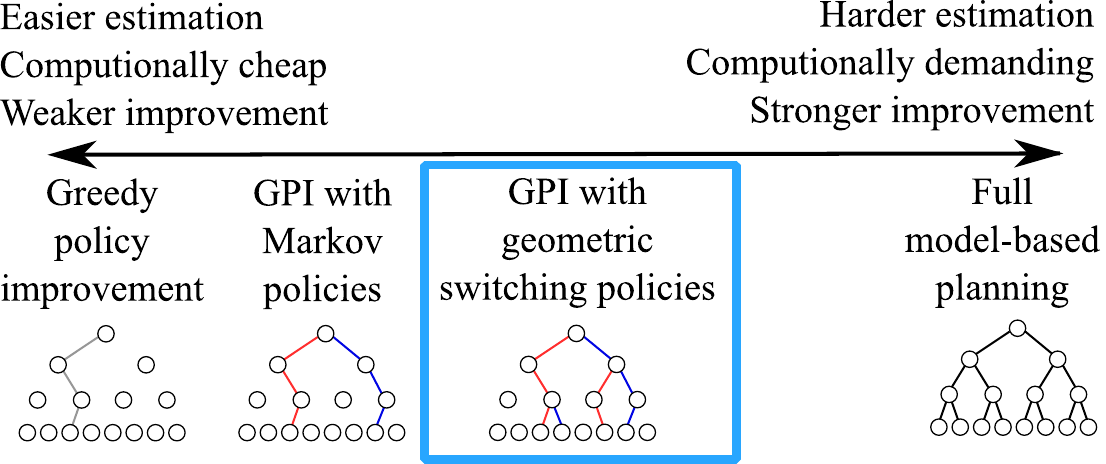}
    \caption{A spectrum of trade-offs in policy improvement. \citet{barreto2016successor} propose generalised policy improvement (GPI) as a means of simultaneously improving over several policies (illustrated with blue and red trajectories), a step from greedy improvement of a single policy towards planning. The central contribution of this paper, GPI with \gsps, moves a step further in this direction, allowing for improvement over non-Markov  \GSPs (illustrated as trajectories that switch between blue and red base policies).
    }
    \label{fig:model-free-based}
\end{figure}

\citet{barreto2016successor} propose generalised policy improvement (GPI), a method that allows for improvement over a \emph{collection} of policies $\{\pi_1,\ldots,\pi_k\}$ simultaneously, generalising the notion of greedy improvement of an individual policy.
We show that GPI can be extended to a much wider class of \emph{non-Markov} policies. These policies, which we call \emph{\gsps} (\GSPs), switch between executing a base set of Markov policies $\{\pi_1,\ldots,\pi_k\}$ with fixed probability. In general, these policies do not ever need to be executed, and can instead be evaluated using information learnt about the base policies, without any further learning required, leading to a stronger improvement guarantee in GPI.
This approach to policy improvement makes statistical and computational trade-offs that interpolate between greedy improvement and full model-based planning, potentially providing benefits of both worlds; Figure~\ref{fig:model-free-based} shows where the proposed approach lies along the spectrum of methods between the conventional model-free and model-based extremes.

Central to our approach is the notion of a geometric horizon model (GHM) \citep{janner2020gamma}, which models the discounted future state-visitation distribution of a given Markov policy. \citet{janner2020gamma} introduced GHMs mainly as a mechanism to compute the value function of a single policy. In this paper we show that GHMs of \emph{distinct} policies can be composed to evaluate a potentially large number of \GSPs
with no additional learning required.
We can then apply GPI over this collection 
of non-Markov policies to obtain a new Markov policy that will in general outperform all of its precursors (base policies \emph{and} switching policies). 

In carrying out the above, we address several technical questions which are contributions in their own right. Specifically, our central technical contributions include:
\begin{itemize}
    \item \emph{\GSP evaluation with GHMs}, a method for evaluating \gsps that only requires learning GHMs for a base class of Markov policies
    (Section~\ref{sec:non-markov-eval}).
    \item \emph{Geometric generalised policy improvement} (\ggpi), a method for deriving a Markov policy that improves over a collection of \gsps, interpolating between greedy improvement and full model-based planning 
    (Section~\ref{sec:nmgpi}).
    \item Convergence analysis of \emph{cross-entropy temporal-difference learning}, an algorithm introduced by \citet{janner2020gamma} for learning GHMs (Section~\ref{sec:learning_ghm}).
    \item New practical methods and insights for training GHMs at scale, including cross-entropy temporal-difference learning with VAE-GHMs (Section~\ref{sec:ghm_training_exp}).
    \item Applications of GHM evaluation and \ggpi
    to both transfer and standard RL settings (Section~\ref{sec:applications_description}), with successful implementation in combination with deep learning in continuous control tasks (Section~\ref{sec:large_scale_ant_exps}).
\end{itemize}

\section{Background}

A Markov decision process (MDP) is specified by a state space $\mathcal{X}$, action space $\mathcal{A}$, transition probabilities $P : \mathcal{X} \times \mathcal{A} \rightarrow \mathscr{P}(\mathcal{X})$, reward distributions $\mathcal{R} : \mathcal{X} \times \mathcal{A} \rightarrow \mathscr{P}_1(\mathbb{R})$, and corresponding expected reward function $r : \mathcal{X} \times \mathcal{A} \rightarrow \mathbb{R}$, defined by $r(x, a) = \mathbb{E}_{R \sim \mathcal{R}(x, a)}[R]$. For ease of presentation, we focus on the case where $\mathcal{X}$ is finite, although much of the material of the paper extends to more general state spaces. 
An agent interacting with the environment using a policy $\pi : \mathcal{X} \rightarrow \mathscr{P}(\mathcal{A})$ generates a trajectory of states, actions, and rewards $(X_t, A_t, R_t)_{t \geq 0}$, and the agent's return along this trajectory is defined by $\sum_{t \geq 0} \gamma^ t R_t$, where $\gamma \in [0,1)$ is the discount factor. The agent's expected return under $\pi$ when beginning in state $x$ and initially taking action $a$ is $Q^\pi_\gamma(x, a) = \mathbb{E}^\pi_{x, a}[\sum_{t \geq 0} \gamma^t R_t]$, where $\mathbb{E}^\pi_{x, a}$ and $\mathbb{P}^\pi_{x, a}$ denote the expectation operator and probability distribution of a trajectory beginning at state-action pair $(x, a)$ and following $\pi$ thereafter.
The goal of \emph{policy evaluation} is to estimate $Q^\pi_\gamma$ for a policy $\pi$ of interest, while the goal of \emph{policy optimisation} is to obtain a policy $\pi^*$ with $Q_\gamma^{\pi^*} \geq Q_\gamma^\pi$ component-wise for all other policies $\pi \in \mathscr{P}(\mathcal{A})^\mathcal{X}$ \citep{sutton2018reinforcement,puterman2014markov,bertsekas1996neuro,szepesvari2010algorithms,meyn2022control}. A fundamental operation in this process is \emph{policy improvement}, described below.

\subsection{Generalised policy improvement}

We first recall a core method for policy improvement in RL. 

\textbf{Greedy policy improvement.} The greedy policy improvement map $\mathcal{G} : \mathbb{R}^{\mathcal{X}\times\mathcal{A}} \rightrightarrows \mathscr{P}(\mathcal{A})^\mathcal{X}$ is a set-valued function that maps Q-functions to the corresponding set of \emph{greedy} policies. Mathematically, we have $\pi' \in \mathcal{G}(Q)$ if and only if
\begin{align*}
    \pi'(a|x) > 0 \implies  a \in \argmax_{a' \in \mathcal{A}} Q(x, a') \, .
\end{align*}
We will overload notation to allow us to pass policies directly to $\mathcal{G}$, writing $\mathcal{G}(\pi)$ for $\mathcal{G}(Q^\pi)$. 
A classical result underpinning policy iteration is that if $\pi' \in \mathcal{G}(\pi)$, 
then $Q^{\pi'} \geq Q^\pi$ element-wise, with equality iff $\pi$ is optimal.

\citet{barreto2016successor} propose \emph{generalised policy improvement}, which provides a means of producing a policy that simultaneously improves over a \emph{set} of base policies.

\textbf{Generalised policy improvement.} The generalised policy improvement (GPI) function $\mathcal{G}$ (overloading notation) takes as input a finite set of Q-functions $\{Q_1,\ldots,Q_n\}$, and returns $\mathcal{G}(\{Q_1,\ldots,Q_n\})$, the set of greedy policies with respect to this set, defined by $\pi' \in \mathcal{G}(\{Q_1,\ldots,Q_n\})$ if and only if
\begin{align}\label{eq:gpi}
    \pi'(a|x) >0 \implies a \in \argmax_{a' \in \mathcal{A}} \max_{i=1}^n Q_i(x,a') \, .
\end{align}

\begin{proposition}[\citealt{barreto2016successor}]\label{prop:gpi}
    If $\pi' \in \mathcal{G}(\{Q^{\pi_1}, \ldots Q^{\pi_n}\})$, then $Q^{\pi'} \geq \max_{i=1}^n Q^{\pi_i}$ element-wise, and equality implies that $\pi'$ is optimal.
\end{proposition}

\subsection{Discounted visitation distributions and geometric horizon models}\label{sec:ghm}

We begin by recalling a key concept in MDPs.

\begin{definition}\label{def:state-visitation}
    The collection of \emph{discounted future state-visitation distributions}
    \footnote{We refer specifically to \emph{future} state-visitation distributions to emphasise that the initial state $x_0$ does not contribute to the distribution.}
    $\mu^\pi_\gamma$ for a policy $\pi$ and discount factor $\gamma$ are indexed by initial state-action pairs $(x_0, a_0) \in \mathcal{X} \times \mathcal{A}$, and are defined by
    \begin{align*}
        \mu^\pi_\gamma(x|x_0, a_0) = (1-\gamma) \sum_{k=0}^\infty \gamma^k \mathbb{P}^\pi_{x_0, a_0}(X_{k+1} = x ) \, ,
    \end{align*}
\end{definition}

A useful interpretation of these distributions is the following.

\begin{restatable}{proposition}{propGeometricInterpretation}\label{prop:geometric-interpretation}
    If $T \sim \text{Geometric}(1-\gamma)$, i.e.
    \begin{align*}
        \mathbb{P}(T = k) = \gamma^{k-1}(1 - \gamma)  \ \ \text{ for } k=1,2,\ldots \, ,
    \end{align*}
    and is independent of the random trajectory $(X_t, A_t, R_t)_{t \geq 0}$ generated by $\pi$ beginning at state-action pair $(x, a)$,
    then the random state $X_T$ is distributed according to $\mu^\pi_\gamma(\cdot|x, a)$.
\end{restatable}

This can also be used as a means of \emph{defining} GHMs over more general state spaces $\mathcal{X}$. \citet{janner2020gamma} introduce $\gamma$-models as generative models of these distributions (in this paper, we will call these objects \emph{geometric horizon models} (GHMs)), and propose to use these models for policy evaluation. The approach is based on well-known identities such as the following \citep{toussaint2005,toussaint2006probabilistic}.

\begin{restatable}{proposition}{propBasicEval}\label{prop:basic-eval}
    For any policy $\pi \in \mathscr{P}(\mathcal{A})^\mathcal{X}$, we have
    \begin{align}\label{eq:gamma-model-eval-1}
        Q^\pi_\gamma(x, a) = r(x, a) + \frac{\gamma}{1-\gamma} \mathbb{E}_{X' \sim \mu^\pi_\gamma(\cdot|x, a)}[r^\pi(X')] \, ,
    \end{align}
    where $r^\pi(x) = \sum_{a \in \mathcal{A}}r(x, a) \pi(a|x)$.
\end{restatable}

This result then naturally suggests a Monte Carlo estimator that can be used for policy evaluation, given a generative model of the distribution $\mu^\pi_\gamma(\cdot|x, a)$, and the reward function $r$. Specifically, if $X'_1,\ldots,X'_n \overset{\text{i.i.d.}}{\sim} \mu^\pi_\gamma(\cdot|x, a)$, then
\begin{align}\label{eq:gamma-model-simple-estimator}
    r(x, a) + \frac{1}{n} \sum_{i=1}^n \frac{\gamma}{1 - \gamma} r^\pi(X'_i)
\end{align}
is an unbiased estimator for $Q^\pi_\gamma(x, a)$.

Note that this expression requires access to the reward function $r$. This function is known in many applications---often in robotics, for example---and when this is not the case it can be learned as a supervised learning problem. Throughout the paper we will assume that $r$ is either given or has been learned.
Note that \citet{janner2020gamma} implicitly use a reward function that depends solely on the \emph{destination state} $x'$ of the transition, leading to slightly different, less general expressions than those above.

\subsection{Composing geometric horizon models for evaluation of Markov policies}\label{sec:composing-markov}\label{sec:ghm-eval}\label{sec:ghm-markov}

As \citet{janner2020gamma} note, a potential disadvantage of using the identity in Equation~\eqref{eq:gamma-model-eval-1} as the basis for policy evaluation is that it requires learning the object $\mu^\pi_\gamma$. When $\gamma \approx 1$, this distribution corresponds to predictions over long time-scales, and is therefore often more difficult to learn than more local predictions. A central observation of \citet{janner2020gamma} is that expressions such as those in  Equation~\eqref{eq:gamma-model-eval-1} can be re-expressed using a geometric horizon model corresponding to a smaller discount factor, $\beta < \gamma$, and composing this model with itself.

\begin{restatable}{proposition}{propMarkovEval}(\citet{janner2020gamma})\label{prop:compose}\label{prop:markov-eval}
    For any policy $\pi \in \mathscr{P}(\mathcal{A})^\mathcal{X}$, $n \geq 1$, and $0 \leq \beta < \gamma$ an unbiased estimator of $Q^\pi_\gamma(x, a)$ is given by
        \begin{align}
            &r(x, a) + \frac{\gamma}{1-\gamma} \times \label{eq:alt-estimator} \\ &\Bigg\lbrack \! \sum_{m=1}^{n-1} \!\frac{1-\gamma}{1 - \beta}\left(\frac{\gamma - \beta}{1-\beta}\right)^{\!m-1}\!\!\!\!\!\!\!\! r^\pi(X^{(m)})
             \!+\! \left(\frac{\gamma - \beta}{1-\beta}\right)^{\!n-1} \!\!\!\!\! r^\pi(X') \Bigg\rbrack \, ,  \nonumber
        \end{align}
        where $X^{(m)} \sim \mu^\pi_\beta(\cdot|X^{(m\indexminus1)}, A^{(m\indexminus1)})$, $A^{(m)} \sim \pi(\cdot|X^{(m)})$, $(X^{(0)}, A^{(0)})=(x,a)$, and $X' \sim \mu^\pi_\gamma(\cdot|X^{(n\indexminus1)}, A^{(n\indexminus1)})$.
\end{restatable}

According to Proposition~\ref{prop:compose}, we can estimate $Q^\pi_\gamma(x, a)$ by sampling the collection of random variables $(X^{(0)}, A^{(0)}, X^{(1)}, \ldots, X^{(n-1)}, A^{(n-1)}, X')$ in the proposition, summarised below:
\begin{equation*}
    \begin{tikzcd}[column sep=small, row sep=scriptsize]
        X^{(0)} \arrow[swap]{r}{\mu^\pi_\beta}  &
        X^{(1)} \arrow[swap]{r}{\mu^\pi_\beta} \arrow[swap]{d}{\pi} &
        X^{(2)} \arrow[swap]{r}{\mu^\pi_\beta} \arrow[swap]{d}{\pi} & 
        \cdots  \arrow[swap]{r}{\mu^\pi_\beta} &
        X^{(n-1)} \arrow[swap]{r}{\mu^\pi_\gamma \ \ } \arrow[swap]{d}{\pi} &
        X' \\%
        A^{(0)} \arrow{ru} &
        A^{(1)} \arrow{ru} &
        A^{(2)} \arrow{ru} &
        \phantom{A^{(2)}} &
        A^{(n-1)} \arrow{ru}
    \end{tikzcd}
\end{equation*}
and then constructing the estimator in Equation~\eqref{eq:alt-estimator}, which the proposition guarantees to be unbiased for $Q^\pi_\gamma(x, a)$; independent estimators can be averaged in the usual manner to reduce variance.

The value of $\beta$ impacts both the mechanics of the process above and the learning of the GHM $\mu^\pi_\beta$ itself. One extreme, $\beta = \gamma$, reverts to the single-sample estimator in Equation~\eqref{eq:gamma-model-simple-estimator}. 
The other extreme, $\beta = 0$, corresponds to estimating the $Q$-function using a single-step transition model
In the first case, predictions are made over potentially long horizons, which alleviates the risk of compounding errors while estimating $Q^\pi_\gamma$. On the other hand, learning the GHM itself becomes more difficult---if we use bootstrapping to do so, as we will discuss shortly, errors might compound when learning $\mu^\pi_\beta$. When $\beta=0$ we observe the opposite trend. In practice, we expect an intermediate value of $\beta$ to offer superior performance to the extremes of $0$ and $\gamma$, since this will trade off errors incurred during the learning of the GHM and the estimation of the $Q$-function \citep{janner2020gamma}. The parameter $n$ offers a trade-off between requiring more compositions of $\mu^\pi_\beta$, and placing a higher weight on samples from the higher-discount, harder-to-train GHM $\mu^\pi_\gamma$.

\section{Composing models for non-Markov policy evaluation}\label{sec:non-markov-eval}

Our first contribution is to extend the estimator appearing in Equation~\eqref{eq:alt-estimator} by modifying the distribution of the random variables $(X^{(0)},A^{(0)}, \ldots, X^{(n-1)}, A^{(n-1)}, X')$ in Proposition~\ref{prop:markov-eval}, composing GHMs for \emph{distinct} policies together.
More precisely, let $(\pi_1,\ldots,\pi_{n})$ be a sequence of policies, $(x, a)$ an initial state action pair, and consider the distribution over state sequences specified by
\begin{equation*}
    \begin{tikzcd}[column sep=small, row sep=scriptsize]
        x_0 \arrow[swap]{r}{\mu^{\pi_1}_\beta}  &
        X^{(1)} \arrow[swap]{r}{\mu^{\pi_2}_\beta} \arrow[swap]{d}{\pi_2} &
        X^{(2)} \arrow[swap]{r}{\mu^{\pi_3}_\beta} \arrow[swap]{d}{\pi_3} & 
        \cdots  \arrow[swap]{r}{\mu^{\pi_{n-1}}_\beta} &
        X^{(n-1)} \arrow[swap]{r}{\mu^{\pi_{n}}_\gamma} \arrow[swap]{d}{\pi_{n}} &
        X' \\%
        a_0 \arrow{ru} &
        A^{(1)} \arrow{ru} &
        A^{(2)} \arrow{ru} &
        \phantom{A^{(2)}} &
        A^{(n-1)} \arrow{ru}
    \end{tikzcd}
\end{equation*}
If we form an expression analogous to Equation~\eqref{eq:alt-estimator}:
\begin{align}\label{eq:multistep-switching-estimator}
    & r(x) + \frac{\gamma}{1-\gamma} \times  \\ &\Bigg\lbrack  \sum_{m=1}^{n-1} \frac{1-\gamma}{1 - \beta}\left(\frac{\gamma - \beta}{1-\beta}\right)^{m-1}\!\!\!\!\!\!\! \bar{r}(X^{(m)})
             + \left(\frac{\gamma - \beta}{1-\beta}\right)^{n-1} \!\! \bar{r}(X') \Bigg\rbrack \nonumber
\end{align}
for some suitable reward function $\bar{r}$, then following the intuition above, we should be able to interpret Expression~\eqref{eq:multistep-switching-estimator} as unbiasedly estimating the value of a (non-Markov) policy that begins each trajectory by following $\pi_1$, before switching to each of $\pi_2, \ldots, \pi_{n\indexminus1}$, and eventually following $\pi_{n}$ for the remainder of the episode. We first formalise this notion of behaviour, and then show that this intuition is correct.

\begin{definition}\label{def:smp}
    Given a sequence $(\pi_1,\ldots,\pi_n)$ of (stationary Markov) policies and a switching probability  $\switchprob \in (0, 1]$,
    the corresponding \emph{\gsp} (\GSP) $\nu$ is a non-Markov policy defined as follows. 
        At the beginning of the episode, the Markov policy $\pi_1$ is followed for  $T_1 \sim \text{Geometric}(\switchprob)$ steps, at which point a switch is made to the Markov policy $\pi_2$. Once a switch from $\pi_i$ to $\pi_{i+1}$ is made, $\pi_{i+1}$ is followed for $T_{i+1} \sim \text{Geometric}(\switchprob)$ steps, at which point the next switching event occurs. Once $\pi_n$ has been selected, it is followed for the remainder of the episode. We write $\pi_1 \overset{\switchprob}{\rightarrow} \cdots \overset{\switchprob}{\rightarrow}\pi_n$ to concisely refer to the \GSP $\nu$. We define  $Q^\nu_\gamma : \mathcal{X} \times \mathcal{A} \rightarrow \mathbb{R}$ for a \GSP $\nu$ by
    \begin{align*}
        Q^\nu_\gamma(x, a) = \mathbb{E}^\nu_{x, a}\Big\lbrack\sum_{t=0}^\infty \gamma^t R_t \Big\rbrack \, ;
    \end{align*}
    precisely, the expectation on the right-hand side is over trajectories beginning at $x$, with actions generated by $\nu$, with the first action overridden to be $a$.
\end{definition}

We now show that the value of \GSPs can be expressed as expectations of expressions such as that in Equation~\eqref{eq:multistep-switching-estimator}.

\begin{mdframed}[style=box]
\begin{restatable}{theorem}{propSMPEval}\label{thm:prop-smp-eval}
    Consider an MDP with reward function $r : \mathcal{X} \rightarrow \mathbb{R}$ and let $\nu = \pi_1 \overset{\switchprob}{\rightarrow} \cdots \overset{\switchprob}{\rightarrow} \pi_n$. With $\beta =  \gamma(1\!-\!\switchprob)$, the following is unbiased for $Q^\nu_\gamma(x, a)$:
    \begin{align}
    & r(x) + \frac{\gamma}{1-\gamma} \times \label{eq:smp-estimator} \\
    & \Bigg\lbrack \! \sum_{m=1}^{n-1} \frac{1-\gamma}{1 - \beta}\!\left(\frac{\gamma - \beta}{1-\beta}\right)^{\!\!m-1}\!\!\!\!\!\!\!\!\! r(X^{(m)})
     \!+ \!\left(\frac{\gamma - \beta}{1-\beta}\right)^{\!\!n-1} \!\!\!\!\!\!r(X') \Bigg\rbrack \, ,  \nonumber
    \end{align}
    where $(X^{(0)}, A^{(0)})=(x, a)$, $X^{(m)} \sim \mu^{\pi_m}_{\beta}(\cdot|X^{(m\indexminus 1)}, A^{(m \indexminus 1)})$, $A^{(m)} \sim \pi_{m+1}(\cdot|X^{(m)})$, $X' \sim \mu^{\pi_{n}}_\gamma(\cdot|X^{(n-1)}, A^{(n-1)})$.
\end{restatable}
\end{mdframed}

\begin{figure*}[t]
    \centering
    \includegraphics[width=.93\textwidth]{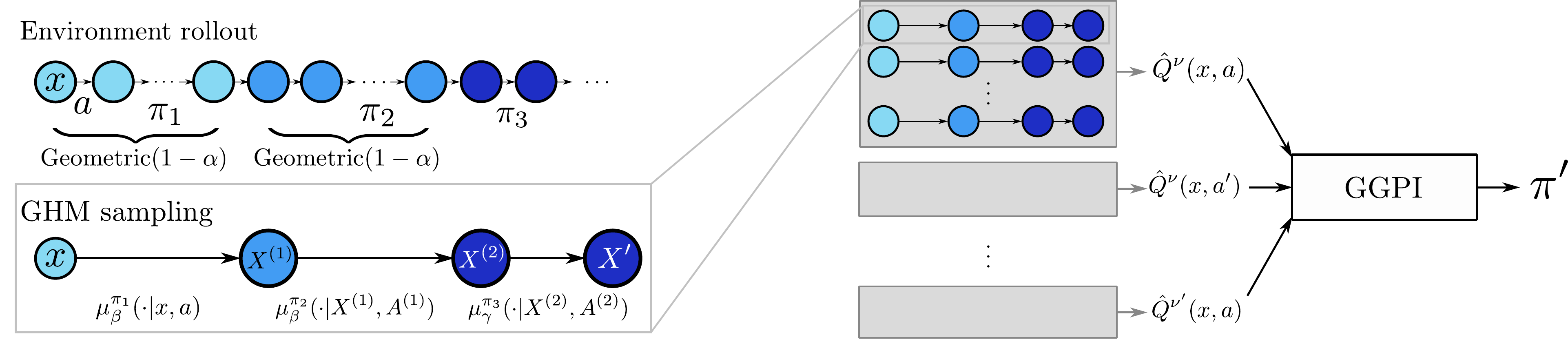}
    \caption{
    \textbf{Left:} A rollout of an example GSP $\nu = \pi_1 \overset{\alpha}{\rightarrow} \pi_2 \overset{\alpha}{\rightarrow} \pi_3$ in the environment, and the GHM sampling procedure that can be used to unbiasedly estimate the value of this policy via Equation~\eqref{eq:smp-estimator}.
    \textbf{Right:} The \ggpi framework. Using the GHM sampling procedure, the action-values of $\nu$ and other GSPs are estimated, and fed into the GPI routine to obtain an improved policy $\pi'$.}
    \label{fig:illustration}
\end{figure*}

We state the result in the case where the reward depends only on state for conciseness here; the slightly more complex formula that incorporates action dependence is given in Appendix~\ref{sec:appendix:extensions}.
The key insight is therefore that we can get an unbiased estimate of the Q-function $Q^\nu_\gamma$ associated with a \gsp $\nu = \pi_1 \overset{\switchprob}{\rightarrow} \cdots \overset{\switchprob}{\rightarrow} \pi_n$ just using the models $\mu^{\pi_i}_{\beta}$ ($i=1,\ldots,n-1$) and $\mu^{\pi_n}_\gamma$ for the base policies. In particular, if we learn these models to evaluate the base policies, we can evaluate all \GSPs arising from these base policies without any additional learning.

\section{Generalised policy improvement with \gsps}\label{sec:nmgpi}

The ability to evaluate a large number of \GSPs without additional learning opens up the possibility of using GPI to improve upon all these policies at once.
Having established how to evaluate \GSPs using GHMs for Markov base policies, the main contribution of this section is to extend GPI to allow for the inclusion of \GSPs into the improvement set. Algorithmically, this is straightforward; the same definition in Equation~\eqref{eq:gpi} can be immediately applied to the Q-functions of \gsps. Note that when applying GPI to the Q-functions of non-Markov \GSPs, the returned greedy  policies are still Markov; this desirable property allows us to embed the proposed approach into the usual RL loop for policy iteration, as discussed below.

What is not immediately clear is whether an improvement guarantee analogous to Proposition~\ref{prop:gpi} still applies when using the Q-functions of \gsps. It turns out, under certain conditions, we can recover such a result. To do so, we need a certain notion of `closedness' amongst the policies to be improved upon.

\begin{definition}\label{def:suffix-closed}
    A collection $\gspcollection$ of \GSPs is \emph{suffix-closed} if whenever $n > 1$ and $\pi_1 \overset{\switchprob}{\rightarrow} \cdots \overset{\switchprob}{\rightarrow}\pi_n$ lies in $\gspcollection$, the suffix policy $\pi_2 \overset{\switchprob}{\rightarrow} \cdots \overset{\switchprob}{\rightarrow}\pi_n$ also lies in $\gspcollection$.
\end{definition}

\begin{mdframed}[style=box]
\begin{restatable}{theorem}{propnmgpi}\label{prop:gnmpi}\label{prop:open-loop-improve}
    Consider a suffix-closed collection of \GSPs $\gspcollection$. Then if $\pi' \in \mathcal{G}(\gspcollection)$, we have
    \begin{align*}
        Q^{\pi'}_\gamma(x,a) \geq \max_{\nu \in \gspcollection} Q^{\nu}_\gamma(x, a) \, , \  \text{ for all } (x, a) \in \mathcal{X} \times \mathcal{A} \, .
    \end{align*}
    Further, if equality holds for all state-action pairs, then $\pi'$ is optimal. 
\end{restatable}
\end{mdframed}

We refer to the procedure of computing $\pi' \in \mathcal{G}(\Pi)$ for a set of \GSPs $\Pi$ as \emph{geometric generalised policy improvement} (\ggpi). A rigorous proof of Theorem~\ref{prop:open-loop-improve} is given in Appendix~\ref{sec:appendix:proofs}, but for some intuition for the suffix-closed condition, consider the two possibilities after a single step of executing $\nu = \pi_1 \overset{\switchprob}{\rightarrow} \cdots \overset{\switchprob}{\rightarrow} \pi_n$: either the first switch has not occurred (in which case it is as though we execute $\nu$ from scratch from the next time step), or the switch has occurred, in which case it is as though we execute the suffix policy $\nu' = \pi_2 \overset{\switchprob}{\rightarrow} \cdots \overset{\switchprob}{\rightarrow} \pi_n$ from the next time step. In fact, this observation yields a Bellman equation
\begin{align*}
    & Q^\nu_\gamma(x, a) = r(x, a) + \\
    & \quad\gamma \mathbb{E}_{\substack{X' \sim P(\cdot|x, a) \\ A_1 \sim \pi_1(\cdot|X') \\ A_2 \sim \pi_2(\cdot|X')}}[(1-\alpha)Q^\nu_\gamma(X', A_1) + \alpha Q^{\nu'}_\gamma(X', A_2)] \, .
\end{align*}
Thus, the suffix-closedness condition is a way of ensuring we can reason about both of these possibilities within the \ggpi process.
Perhaps surprisingly, the suffix-closedness condition in Theorem~\ref{prop:gnmpi} really is necessary; some care needs to be taken when applying the ideas associated with GPI to non-Markov policies. A counterexample when the closure condition is removed is provided in Appendix~\ref{sec:gnmpi-counterexamples}, along with several other examples.

In summary, GHM policy evaluation and \ggpi allow us to derive Markov policies that improve over a wide range of \GSPs, while only requiring learnt GHMs for the base Markov policies under consideration; see Figure~\ref{fig:illustration}.

\section{Applications: transfer and policy iteration}\label{sec:applications_description}

We now detail two central applications of GHM evaluation and \ggpi to reinforcement learning.

\subsection{Transfer and zero-shot learning}\label{sec:transfer}

In the transfer setting, we have a collection of known policies $\pi_1,\ldots,\pi_k$, and a reward function $r$ for which we wish to find a good policy. The policies $\pi_{1:k}$ may have been obtained in a variety of ways: learnt by maximising other reward signals, exploration objectives, from imitation learning, etc. The reward function $r$ is assumed to either be known (as is common in many robotics applications, for example), or learnt from data.

One simple approach to implementing GPI is to learn GHMs $(\mu_\gamma^{\pi_i})_{i=1}^k$, and use these in combination with the given reward function $r$ to estimate $Q^{\pi_1},\ldots,Q^{\pi_k}$, and perform generalised policy improvement over these Q-functions, as justified by Proposition~\ref{prop:gpi}.

With the concepts introduced above, we can improve on this by additionally learning GHMs $(\mu_\beta^{\pi_i})_{i=1}^k$, composing these to evaluate a collection of \GSPs, and then using the \ggpi procedure to improve over all such switching policies.
A pseudocode summary of the approach is provided in Appendix~\ref{sec:appendix:transfer-alg}.
Given a base set $\Pi = \{\pi_1,\ldots,\pi_k\}$ of Markov policies and a switching probability, we can define a variety of different sets of GSPs. A natural choice to consider, which we adopt in the experiments, is \emph{the set of depth-$m$ compositions}, $\Pi_m = \{ \pi^{(1)} \overset{\alpha}{\rightarrow} \cdots \overset{\alpha}{\rightarrow} \pi^{(m)} : \pi^{(1)}, \ldots, \pi^{(m)} \in \Pi \}$, consisting of all \GSPs that switch between exactly $m$ (not necessarily distinct) base policies.
We refer to \ggpi on $\Pi_m$ as \emph{depth-$m$ \ggpi}. The following result shows that \ggpi over $\Pi_m$ guarantees an improvement, thanks to Theorem~\ref{prop:open-loop-improve}.

\begin{restatable}{proposition}{propPiMClosed}\label{prop:pimclosed}
    $\Pi_m$ is suffix-closed.
\end{restatable}

\begin{example}
    Figure~\ref{fig:four-rooms} illustrates an example experiment in the four-rooms environment \citep{sutton1999between}, with a single positive reward at the top-right-most cell, and $\gamma = 0.9$. We consider four base policies $\pi_\text{L}, \pi_\text{D}, \pi_\text{R}, \pi_\text{U}$ that always take the action left/down/right/up in each cell. GHMs are calculated for these policies with discounts $\gamma$ and $\beta = 0.8$. By using \ggpi over \GSPs that make switches between these basic policies, the optimal policy can be recovered in almost all states of the environment, without any additional learning. Figure~\ref{fig:four-rooms} illustrates in which states the optimal policy can be computed when using GPI over the four base policies (left), depth-2 \ggpi (centre), and depth-3 \ggpi (right). Depth-3 \ggpi is able to compute the optimal action in the vast majority of environment states.
\end{example}

\begin{figure}[ht]
    \centering
    \null
    \hfill
    \includegraphics[keepaspectratio,width=.15\textwidth]{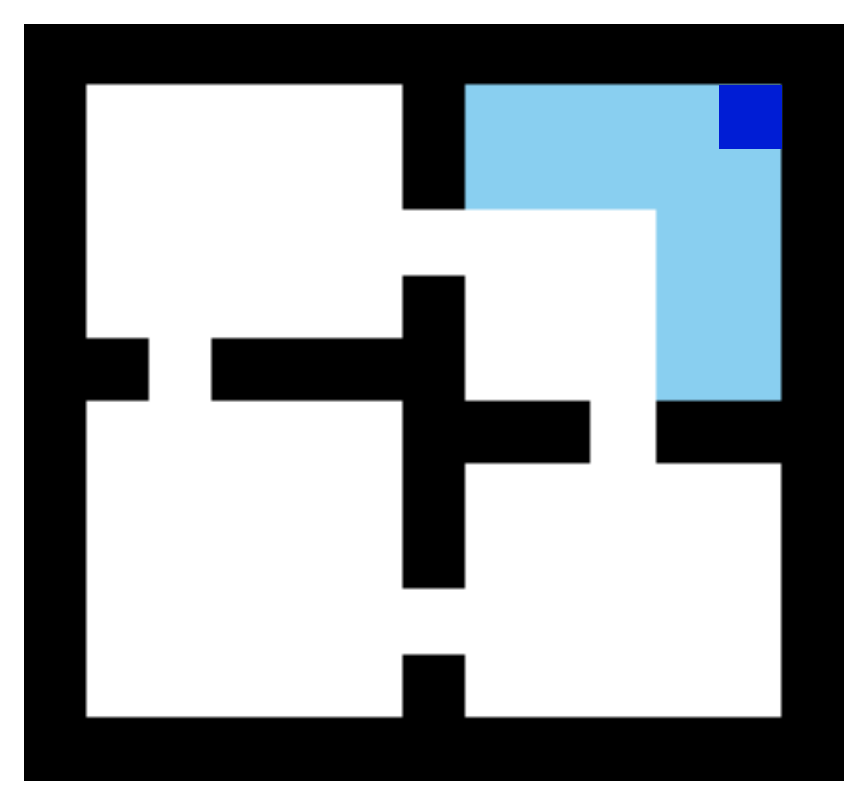}
    \hfill
    \includegraphics[keepaspectratio,width=.15\textwidth]{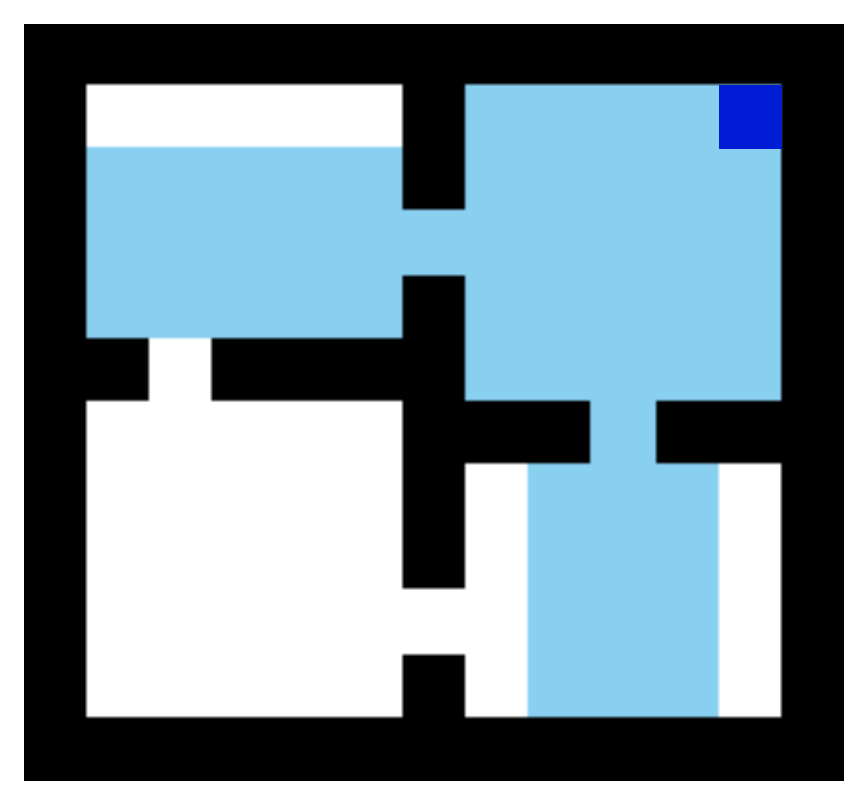}
    \hfill
    \includegraphics[keepaspectratio,width=.15\textwidth]{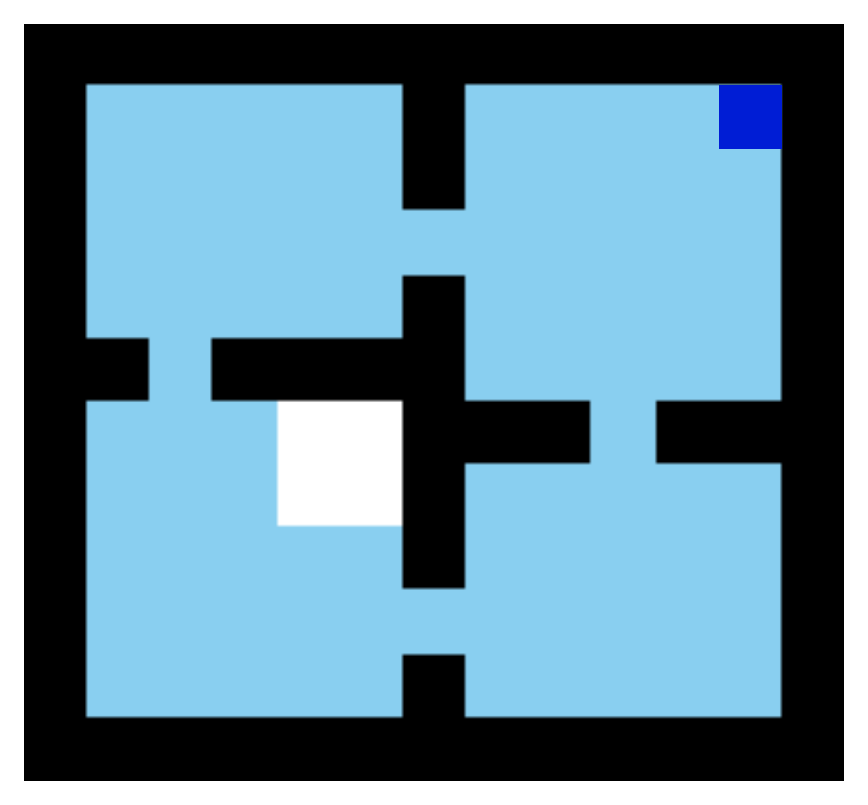}
    \hfill
    \null
    \caption{An illustration of the \ggpi method on the four-rooms environment, with goal state indicated in dark blue.
    The plots illustrate which states (highlighted in light blue) each planning method is able to compute the optimal action for: GPI (left), depth-2 \ggpi (centre), and depth-3 \ggpi (right).}
    \label{fig:four-rooms}
\end{figure}

\subsection{Policy iteration}

Policy iteration is a classical dynamic programming algorithm that computes a sequence of policies $(\pi_k)_{k\geq 0}$ through an iterative process of evaluation and greedy improvement, i.e. $\pi_k \in \mathcal{G}(Q^{\pi_{k\indexminus1}})$, which is guaranteed to reach the optimal policy in a finite number of iterations (for environments with finite state space). A natural question is whether we can use GPI to speed up this iterative process, by leveraging policies from previous iterations to compute even stronger improved policies, e.g. $\pi_{k} \in \mathcal{G}(\pi_{0:k\indexminus1})$. Unfortunately, when using standard GPI the answer to this question is ``no''; since $Q^{\pi_{k-1}} \geq Q^{\pi_l}$ for $l < k-1$, GPI over $\pi_{0:k\indexminus1}$ reduces to standard policy improvement over $\pi_{k-1}$.

However, using \ggpi may enable leveraging policies from older iterations to make larger improvement steps and converge to $\pi^*$ more quickly, for example performing \ggpi over all depth-$m$ compositions over the set of previous policies $\{\pi_0, \ldots \pi_{k-1}\}$.
This has the advantage that any useful behaviour encoded by a prior policy that gets prematurely overwritten by subsequent iterations can still be leveraged to make larger improvement steps.
Appendix~\ref{sec:appendix:policy_iteration_details} contains algorithm pseudocode for applying \ggpi to policy iteration, as well as an illustrative example.

\begin{example}
    In Figure~\ref{fig:four_rooms_policy_iteration} we demonstrate the advantage of using \ggpi for policy iteration in the classic four-rooms environment.
    The number of improvement steps decreases with the \ggpi depth, indicating that the \ggpi improvement step is able to compute stronger improved policies the more past knowledge it is allowed to leverage.
    Note however that although higher depths require fewer iterations, each improvement step is more computationally intensive for higher-depth \ggpi; 
    in this instance, depth-2 \ggpi obtains the optimal trade-off between computational burden and strength of policy improvement, finding the optimal policy with the lowest total number of GHM samples.
    Here we solve the problem for a discount factor of $\gamma=0.95$, switching probability $\alpha=0.1$, compute perfect GHMs obtained using knowledge of the true environment dynamics, and evaluate each \GSP $\nu$ using $1000$ samples from the composed GHM $\mu^{\nu}_\gamma$.
\end{example}

\begin{figure}[ht]
    \centering
    \includegraphics[width=0.47\textwidth]{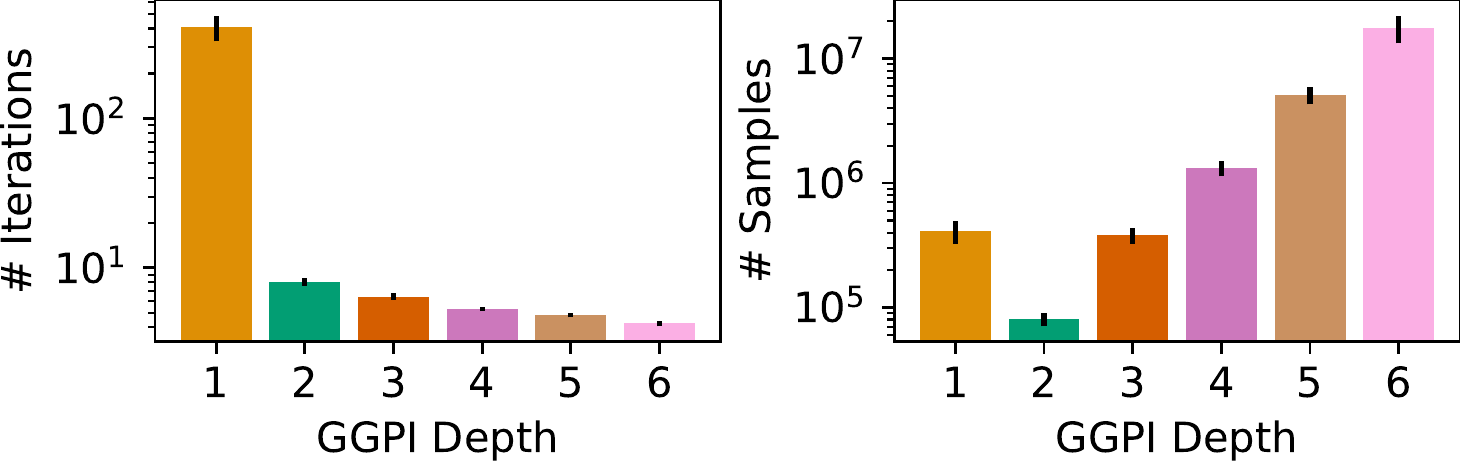}
    \caption{Comparison of \ggpi with various planning depths applied to policy iteration on the standard 4-rooms domain, starting from a random policy.
    \textbf{Left:} Total number of iterations of policy iteration required;
    \textbf{Right:} Total number of GHM samples required, as a proxy to total computation performed.
    Error bars show bootstrapped 95\% confidence intervals over 100 seeds.}
    \label{fig:four_rooms_policy_iteration}
\end{figure}

\section{Learning geometric horizon models}
\label{sec:learning_ghm}

To use geometric horizon models for value estimation in practice, an important question is how to learn such models in the first place. An instructive starting point is to consider supervised learning with samples from $\mu^\pi_\beta$ (obtained by sampling $X_T$ with $T \sim \text{Geometric}(1-\beta)$ by interacting with the environment using $\pi$, for example). The canonical cross-entropy loss can then be used to train a GHM $\mu$, leading to the \emph{cross-entropy Monte Carlo} (CEMC) loss:
\begin{align}\label{eq:cemc}
    \mathbb{E}_{X' \sim \mu^\pi_\beta(\cdot|x, a)}[-\log \mu(X'|x, a)] \, .
\end{align}
As with Monte Carlo learning in value-based RL, this approach is typically difficult to apply with off-policy data, incurring either bias, or potentially high variance updates from off-policy corrections \citep{precup2000eligibility}.
An alternative approach can be motivated by the observation that $\mu^\pi_\beta$ satisfies a Bellman equation involving composed models.

\begin{definition}[Composed geometric horizon models]
    Given two GHMs $\mu_1, \mu_2 \in \mathscr{P}(\mathcal{X})^{\mathcal{X} \times \mathcal{A}}$, and a policy $\pi \in \mathscr{P}(\mathcal{A})^\mathcal{X}$, the \emph{composed model} $\mu_2 \otimes_\pi \mu_1 \in \mathscr{P}(\mathcal{X})^{\mathcal{X} \times \mathcal{A}}$ is the distribution of the random variable $X^{(2)}$, defined by
    \begin{itemize}
        \item $X^{(1)} \sim \mu_1(\cdot|x, a)$,
        \item $A^{(1)}\mid X^{(1)} \sim \pi(\cdot|X^{(1)})$,
        \item $X^{(2)} \mid (X^{(1)}, A^{(1)}) \sim \mu_2(\cdot|X^{(1)}, A^{(1)})$,
    \end{itemize}
    The distributions satisfy the relationship
    \begin{align*}
        (\mu_2 \!\otimes_\pi\! \mu_1)(y|x, a)\! = \!\!
        \sum_{x',a'}\! \mu_1(x'|x, a) \pi(a'|x') \mu_2(y|x', a') \, .
    \end{align*}
\end{definition}

\begin{restatable}{proposition}{propBellman}\label{prop:bellman}
    Defining the Bellman operator $T^\pi_\beta : \mathscr{P}(\mathcal{X})^{\mathcal{X}\times \mathcal{A}} \rightarrow \mathscr{P}(\mathcal{X})^{\mathcal{X}\times \mathcal{A}}$ by
    \begin{align*}
        (T^\pi_\beta \mu)(x'|x, a)\! = \! (1\!-\!\beta) P(x'|x, a)\! +\! \beta ( \mu\! \otimes_\pi\! P)(x'|x, a) \, ,
    \end{align*}
    then $\mu^\pi_\beta$ is the unique solution to $\mu = T^\pi_\beta \mu$.
\end{restatable}
This motivates a loss in which the Monte Carlo target in Equation~\eqref{eq:cemc} is replaced by the `bootstrapped' distribution $T^\pi_\beta \mu$, leading to the \emph{cross-entropy temporal-difference} (CETD) loss, briefly mentioned by \citet{janner2020gamma}:
\begin{align}\label{eq:expected-cetd}
    \mathbb{E}_{X' \sim (T^\pi_\beta \bar{\mu})(\cdot|x, a)}[-\log \mu(X'|x, a)] \, ,
\end{align}
where $\bar{\mu}$ denotes a stop-gradient on $\mu$. Intuitively, $(T^\pi_\beta \mu)(\cdot|x, a)$ is the distribution obtained by sampling a next state $\tilde{x}$ from $P(\cdot|x, a)$, independently deciding whether to stop (with probability $1-\beta$) and output this state, or to sample an action $\tilde{a} \sim \pi(\cdot|\tilde{x})$ and instead return a sample from $\mu(\cdot|\tilde{x}, \tilde{a})$. This also describes a method by which sample-based approximations to Equation~\eqref{eq:expected-cetd} can be derived, leading to an algorithm that can be used at scale. 

However, while sample-based minimisation of Equation~\eqref{eq:cemc} can be understood through stochastic gradient descent and convex optimisation theory, it is less clear that following sample-based gradients of the CETD loss in Equation~\eqref{eq:expected-cetd} will lead to $\mu^\pi_\beta$, due to the presence of bootstrapping. Next, we show that, under certain conditions, convergence to $\mu^\pi_\beta$ can be guaranteed, and additionally we show how the CETD loss can be applied at scale.
Note that the prior approach to training GHMs at scale proposed by \citet{janner2020gamma} instead focused on a biased $L^2$ loss between log-densities; we show that CETD typically outperforms this approach in Appendix~\ref{sec:appendix:ghm_training_exps}, and note that it has the further advantage of not requiring access to single-step transition densities.

\subsection{Convergence analysis of CETD}

Consider a finite state space $\mathcal{X}$, and a tabular parametrisation of each distribution $\mu(\cdot|x, a)$ by a vector of logits $\phi(x, a) \in \mathbb{R}^\mathcal{X}$, so that $\mu(\cdot|x, a) = \text{softmax}(\phi(x, a))$. We show that with this parametrisation convergence to $\mu^\pi_\beta$ is obtained following CETD updates under mild conditions.
To describe the precise algorithm we study, 
let $\phi_0 \in \mathbb{R}^{\mathcal{X} \times \mathcal{A}\times \mathcal{X}}$ be the initial values of the logits in the parametrisation described above.
We then consider generating a sequence of logits $(\phi_k)_{k \geq 0}$ and corresponding distributions $(\mu_k)_{k \geq 0}$ by iteratively applying \emph{synchronous} CETD updates; at algorithm time $k$, for each state-action pair $(x, a)$, we observe a transition $(x, a, x')$ and perform the update:
\begin{align}
    &\phi_{k+1}(x, a) = \phi_k(x, a) + \varepsilon_k \Big( \!\underbrace{ (\hat{T}^\pi\mu_k)(x, a)- \mu_k(\cdot|x, a) }_{\text{stochastic cross-entropy gradient}}\! \Big) \, , \nonumber \\
    & \ \ \ \ \text{ with } 
    (\hat{T}^\pi \mu_k)(x, a) = (1-\gamma) e_{x'} + \gamma e_{x''} \label{eq:cetd-updates-for-thm} \, ,
\end{align}
where $x''$ is generated by sampling $a' \sim \pi(\cdot|x')$ and $x'' \sim \mu_k(\cdot|x', a')$, and where $e_{y} \in \mathbb{R}^{\mathcal{X}}$ is the one-hot vector at state $y$. Here, $(\varepsilon_k)_{k =0}^\infty$ is a sequence of step sizes.

\begin{mdframed}[style=box]
\begin{theorem}\label{thm:cetd}
    The CETD algorithm specified by the updates in Equation~\eqref{eq:cetd-updates-for-thm} produces sequences of
    distributions $(\mu_k)_{k \geq 0}$ such that $\mu_k(\cdot|x, a) \rightarrow \mu^\pi_\gamma(\cdot|x, a)$ for all $(x, a)$, as long as 
    $
    \sum_{k\geq 0} \varepsilon_k = \infty \, , \ \ \sum_{k\geq 0} \varepsilon^2_k < \infty
    $.
\end{theorem}
\end{mdframed}

The proof, relying on a discrete Lyapunov argument based on the Robbins-Siegmund theorem \citep{robbins1971convergence}, is in Appendix~\ref{sec:appendix:td-mle-proof} with several illustrative examples.

\subsection{Learning VAE-GHMs with CETD updates}\label{sec:vae-ghm}

We propose a novel scalable means of learning GHMs $\mu^\pi_\beta$ with VAEs \citep{kingma2014auto,rezende2014stochastic} based on the CETD loss, and emphasise that these methods also apply equally well to learning GHM densities for MDPs with continuous state spaces, as is often of interest in deep RL. Specifically, we use a conditional VAE architecture $\mu_\theta(\cdot|x, a, z)$ \citep{sohn2015learning} with latent variable $z$, and approximate posterior $q_\psi(z|x, a, X')$. The CETD loss is then the negative log-marginal likelihood of the training state under the VAE model, leading to the following evidence lower-bound (ELBO) on the negative CETD loss:
\begin{align*}
    \mathbb{E}_{X' \sim (T^\pi_\beta \bar{\mu}_\theta)(\cdot|x, a)}\left\lbrack \mathbb{E}_{z \sim q_\psi(\cdot|x, a, X')}\big\lbrack\log\big(\tfrac{\mu_\theta(X'|x, a, z)}{q_\psi(z|x, a, X')}\big)\big\rbrack  \right\rbrack
\end{align*}
which is then jointly optimised over $\theta$ and $\psi$ via stochastic gradient descent with the reparametrisation trick~\citep{kingma2014auto}. Using VAEs offers several advantages, such as allowing low latent dimensionality in non-stochastic environments, and connection to the theoretically-justified CETD loss; see Section~\ref{sec:ghm_training_exp} for further commentary.

\section{Deep reinforcement learning experiments}
\label{sec:large_scale_ant_exps}

To understand how \GSP evaluation using GHMs and \ggpi perform at scale, we test them on a deep RL transfer task. Full details and further results are given in Appendix~\ref{sec:appendix:experiment_details}.

\textbf{Environment details.}
We consider a
continuous control task inspired from the moving-target arena in \citet{barreto2019optionkeyboard}, which we call \emph{sparse-reward ant}. The agent is a quadrupedal ``ant'', and the environment observation is a 35-dimensional representation of the agent state, including position, velocity, and joint angles. The agent interacts with the environment via an 8-dimensional action space controlling the torque applied to its various joints.
At the beginning of each episode, the agent's is initialised at rest at a location sampled from a uniform distribution over a square centred at the origin, and a target location is sampled from a smaller region surrounding the agent's initialisation (see Appendix~\ref{sec:appendix:experiment_details_ant} for details).
The reward is $1$ for transitions that terminate in a region around the target and $0$ elsewhere.

\textbf{Experiment setup.}
Similar to \citet{barreto2019optionkeyboard}, we first pretrain four base policies $\Pi = \{\pi_{\text{up}}, \pi_{\text{down}}, \pi_{\text{left}}, \pi_{\text{right}}\}$ that aim to move along each of the 4 directions. The policies are stochastic and implemented as a 2-layer MLP outputting the mean/variance of a Gaussian torque to be applied at each of the ant's 8 joints. These policies are pretrained using \citeauthor{abdolmaleki2018maximum}'s (\citeyear{abdolmaleki2018maximum}) MPO with the reward calculated based on the component of the ant's velocity in the desired direction.

Next, we train GHMs for the base policies using the approach described in Section~\ref{sec:vae-ghm}. The model is implemented as a standard conditional $\beta$-VAE \citep{higgins2016beta,sohn2015learning} with a single latent dimension, as this is sufficient to model the probabilistic horizon in the deterministic environment. We train these GHMs from transitions using the CETD bound described in Section~\ref{sec:vae-ghm}, for GHM $\beta$ values of $0.8$ and $0.9$, and consider the task with discount factor $\gamma=0.9$. This corresponds to performing \ggpi for \GSPs for a switching probability of $\alpha = \nicefrac{1}{9}$.
Further details, observations, and recommendations are provided in Appendix~\ref{sec:appendix:experiment_details}; for example, we found off-policy training of GHMs important to obtain sufficient state coverage.

Once the GHMs have been learned, the agent is evaluated on new episodes without additional learning.
In each test episode, the agent must plan to optimise for a new revealed reward function $r$ associated with the randomly-generated target region, leveraging the learnt GHMs for the set of base policies $\Pi$ above. We consider two approaches: (i) GPI \citep{barreto2016successor} on $\Pi$ using GHM evaluation \citep{janner2020gamma}, a natural baseline for this task (equivalent to depth-1 \ggpi), and (ii) depth-2 \ggpi (Section~\ref{sec:transfer}).

\begin{figure}[h]
    \centering
    \includegraphics[keepaspectratio,width=.48\textwidth]{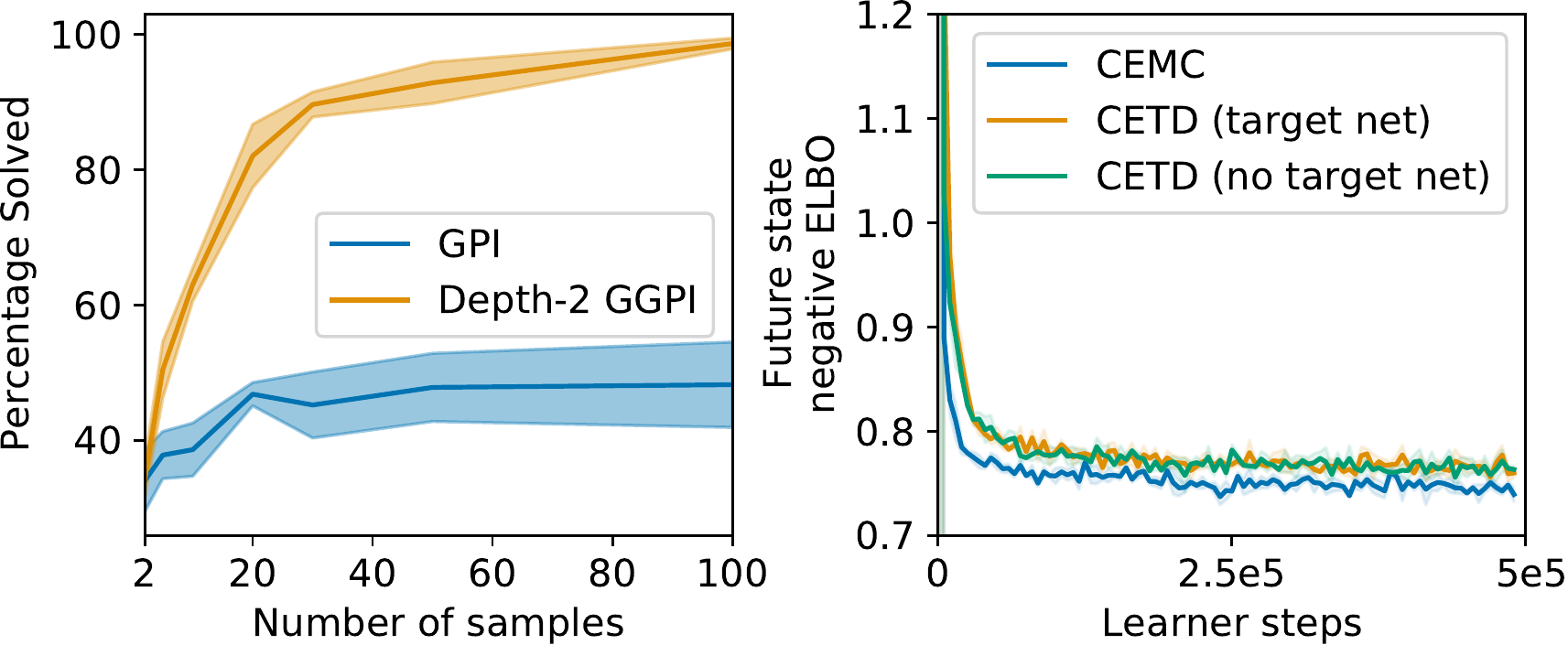}
    \caption{\textbf{Left:} Comparing \ggpi at different depths $m$ in terms of total episodes succeeded, for various sampling budgets. Agents are evaluated across 5 random seeds for GHM training and 100 test episodes with bootstrapped 90\% confidence intervals.
    \textbf{Right:} Comparison of GHMs trained using various losses measured in terms of negative ELBO (lower is better)
    of samples from the true future state visitation distribution obtained by sampling states from on-policy trajectories, averaged over 5 random seeds.
    }
    \label{fig:ant_results_sampling_and_vaes}
\end{figure}

\textbf{Results.}
Figure~\ref{fig:ant_results_sampling_and_vaes} (left) shows the proportion of test episodes successfully solved by GPI and by depth-2 \ggpi, varying the sample budget $n_{\text{samples}}$ used to estimate each Q-value.
We see that depth-2 \ggpi outperforms GPI for  $n_{\text{samples}} \ge 2$, eventually reaching a success rate close to $100\%$.
In Table~\ref{tab:ant_results_depth_1_2} we take a finer-grained look at the results when using 100 samples. We see that standard GPI manages to solve the task roughly 50\% of the time, while depth-2 \ggpi (where the agent is able to model changing directions) not only solves almost all the remaining goal locations but also almost always solves the task faster than standard GPI when both are capable of reaching the goal.
Agent behaviour is visualised in Figure~\ref{fig:combined_viz} and Appendix~\ref{sec:agent-visualisations}.

\begin{table}[h]
    \centering
    \caption{Results of comparing agent performance for GPI and depth-2 \ggpi on the \emph{sparse-reward ant} environment. Agents are evaluated across 5 random seeds for GHM training and 100 random environment and target initialisations.}
    \label{tab:ant_results_depth_1_2}
    \begin{tabular}{l|l}
        \toprule
        Case                           & Frequency \\
        \midrule
        Depth-2 \ggpi succeeds, GPI fails        & $50.8 \pm 5.7$        \\ 
        Both succeed, depth-2 \ggpi is faster & $45.0 \pm 7.0$        \\ 
        Both fail                      & $1.0 \pm 0.9$         \\ 
        Both succeed but GPI is faster & $2.8 \pm 1.2$         \\ 
        GPI succeeds, depth-2 \ggpi fails        & $0.4 \pm 0.8$ \\
        \bottomrule
    \end{tabular}
\end{table}

\begin{figure}[h]
    \centering
    \includegraphics[width=.43\textwidth]{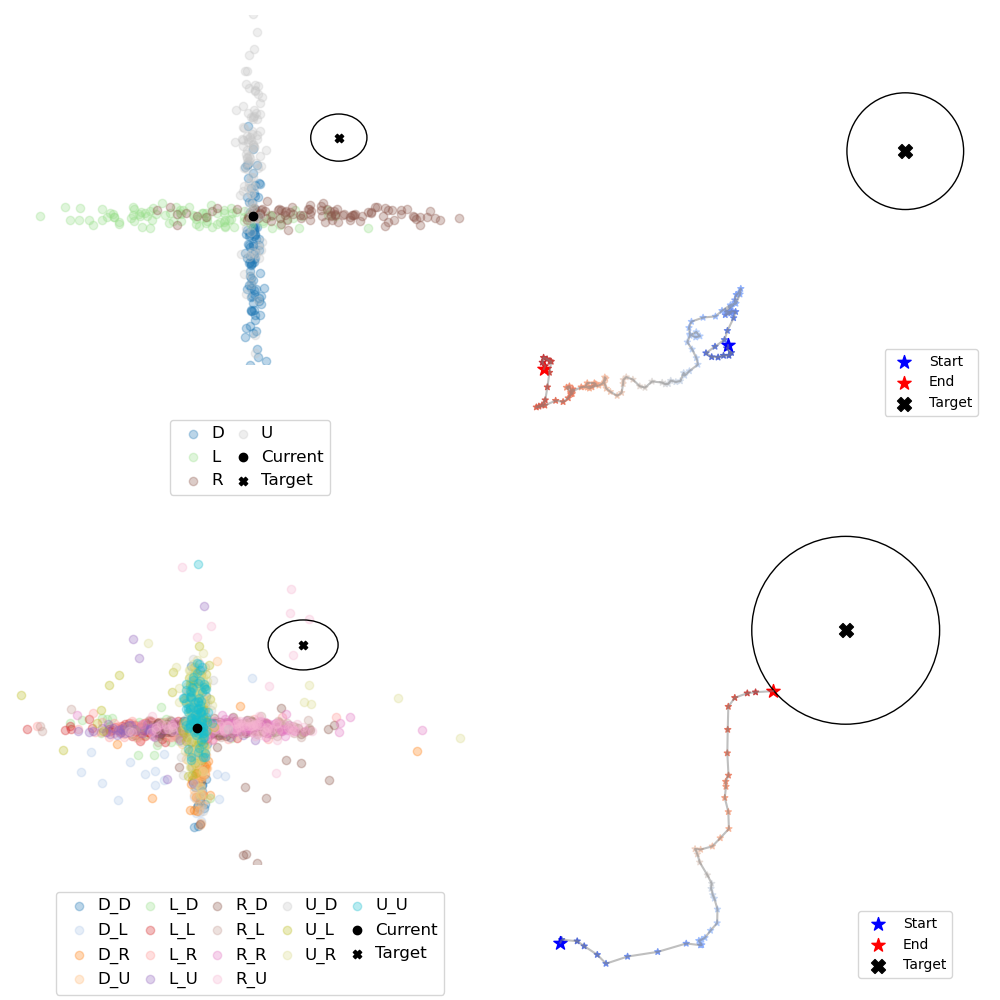}
    \caption{Visualisation of an agent performing GPI (top) and depth-2 \ggpi (bottom) for the same test episode. GHM samples for use in \ggpi at the first episode time step are shown on the left, while the entire episode trajectory coloured on a gradient from blue to red with time is shown on the right. We show the ant centre of mass, target, and reward signal boundary.
    }
    \label{fig:combined_viz}
\end{figure}

\textbf{GHM training experiments.}
\label{sec:ghm_training_exp}
GHM training is an important component of the deep RL results above. We found combining VAEs with the CETD loss to work particularly well; Figure~\ref{fig:ant_results_sampling_and_vaes} (right) shows negative ELBO loss curves for the VAE-GHM of the policy $\pi_{\text{right}}$ on the sparse-reward ant domain, with $\beta=0.8$.
The loss is similar to that of a VAE-GHM trained via supervised learning with the
CEMC loss from Equation~\eqref{eq:cemc}, and we found target networks unnecessary for stable training. 
We show in Appendix~\ref{sec:appendix:ghm_training_exps} that this combination of VAEs with the CETD loss is remarkably stable compared to normalising flows using either CETD or the log-$L^2$ loss previously considered. Appendix~\ref{sec:appendix:ghm_vs_one_step} compares GHMs with multi-step compositions of VAEs modelling single-step transitions, showing that the latter incur large errors and thus are unsuitable for long-horizon planning.
Finally, in Appendix~\ref{sec:appendix:ghm_training_budget} we examine the performance of \ggpi when using GHMs trained with varying sampling budgets, showing that even GHMs trained for only a few thousand steps whose loss has not yet converged are still useful and result in a strong improved policy.

\section{Related work}

This paper relates to a number of different areas in reinforcement learning; we describe the most closely related works below, with additional discussion in Appendix~\ref{sec:appendix:background}.

In addition to the work of \citet{janner2020gamma} described above, there are several recent contributions studying the task of large-scale learning of discounted visitation distributions and related objects. \citet{blier2021learning} propose several TD-based methods for learning parametric representations, including an approach based on low-rank approximations. Building on this work, \citet{touati2021learning} propose a compact representation of an MDP that in principle allows for the optimal policy associated with any reward function to be computed without planning, in practice relying on a low-dimensional approximation of the visitation distributions. \citet{eysenbach2021c} propose a classification-based approach based on contrastive learning; these works also note a close connection with the domain of goal-conditioned RL \citep{kaelbling1993learning,schaul2015universal,andrychowicz2017hindsight,pong2018temporal}. These ideas go back to the \emph{successor representation} (SR), introduced in the context of representation learning in finite-state MDPs by \citet{dayan1993improving}, who also proposed a TD method for learning the SR; this has also been explored in combination with deep learning \cite{kulkarni2016deep,fujimoto2021deep}.

Relatedly, modelling discounted visitation distributions for evaluation was proposed by \citet{sutton1995td}, who termed such objects $\beta$-models. These models were generalised by \citet{precup1998multi}, who proposed multi-time models, which encompass both $\beta$-models and $n$-step models as special cases. More generally, there is a long-established practice of learning option models \citep{sutton1999between,precup1998theoretical,precup2000temporal}, and using such models in a compositional manner \citep{silver2012compositional}.
A central difference between these option models and this work is that
the use of geometric switching times (or, in the language of options, constant termination probabilities) means we do not need to model accumulated return or the taken executing each base policy,
making applications to transfer possible.
In this regard, the approach of this paper may be viewed as a generative approach to learning a certain class of universal option models \citep{yao2014universal}, which also disentangle reward and transition structure; constant termination probabilities facilitate sample-based composition of such models.

Our application to transfer learning in RL is motivated by \emph{successor features} and generalised policy improvement, introduced by \citet{barreto2016successor,barreto2020fast}.
Subsequent work in this direction includes algorithmic innovations in combination with deep learning \citep{pmlr-v80-barreto18a,borsa2019universal},
reward-free learning \citep{grimm2019disentagled,hansen2020fast},
and addressing questions concerning the influence of the policy set on improvements in GPI \citep{zahavy2021discovering,alver2022constructing,lehnert2020successor,nemecek2021policy}. A notable approach that also interpolates between greedy improvement and computation of optimal policies is multi-step policy improvement \citep{efroni2018beyond,efroni2018multiple,tomar2020multi}.

\section{Conclusions, limitations, and future work}\label{sec:planning}

In this paper, we have proposed using geometric horizon models for the evaluation of non-Markov \gsps, and for doing policy improvement over collections of such policies. We have shown that this pair of techniques can be applied to both transfer and policy iteration, extending existing techniques based on successor features and generalised policy improvement. We have also demonstrated that it is possible to combine these ideas with deep learning architectures to arrive at novel approaches to deep RL, and in the course have additionally provided theoretical analyses of these methods.

We foresee several key considerations in further extending the applicability of this approach.
First, the method relies on constructing models over environment state; as with many other model-based methods, a key question is how to learn such models efficiently in high-dimensional settings. 
Additionally, the use of geometric switching times in \GSPs is key to decoupling rewards from learnt models, but limits the expressivity of the non-Markov policies considered; can this restriction be lifted?
In addition to these questions, there are several natural directions for future work. These include further development of theoretical convergence analyses for learning GHMs and improving over \GSPs, as well as further developing combinations of these techniques with deep learning.
We believe that combining \ggpi with recent advances in adaptive planning techniques is a particularly promising direction for further work.

\section*{Acknowledgements}

We thank the anonymous reviewers for useful comments and suggestions, and gratefully acknowledge support from our colleagues in the course of this work.
Thanks in particular to 
Mohammad Gheshlaghi Azar,
Gheorghe Comanici, 
Hamza Merzic,
Doina Precup,
Yunhao Tang, 
and to Th{\'e}ophane Weber for detailed feedback on an earlier draft.

\clearpage

\bibliography{main}
\bibliographystyle{icml2022}

\clearpage
\appendix
\onecolumn

\section*{\centering Generalised Policy Improvement with Geometric Policy Composition: Appendices}

We briefly summarise the contents of the appendices here for convenience.
\begin{itemize}[topsep=0cm]
    \item Appendix~\ref{sec:appendix:background} provides
    further discussion of related work, as well as
    additional context for geometric horizon models and their precise connection with concepts such as the successor representation.
    \item Appendix~\ref{sec:appendix:proofs} provides proofs for the results in the main paper concerning evaluating and improving over \gsps.
    \item Appendix~\ref{sec:appendix:cetd-convergence} provides a proof of the CETD convergence result presented in the main paper, and illustrations of an implementation of the algorithm.
    \item Appendix~\ref{sec:appendix:examples} provides further examples and illustrations to complement the findings of the main paper, including counterexamples illustrating the necessity of several conditions in our results and algorithm pseudocode for application of \ggpi to transfer and policy iteration.
    \item Appendix~\ref{sec:appendix:experiment_details} provides further experimental details and results.
    \item Appendix~\ref{sec:appendix:extensions} provides a generalisation of the core policy evaluation result in the main paper.
\end{itemize}

\section{Additional background, related work and context}
\label{sec:appendix:background}

\subsection{Related work}

Below, we discuss connections of this work to several sub-fields of reinforcement learning.

\textbf{Other generalisations of greedy policy improvement.} Our proposed approach is one way of interpolating between greedy improvement and full planning.  \citet{efroni2018beyond,efroni2018multiple,tomar2020multi} consider multi-step improvement as a different means of achieving such a trade-off, both analysing the approach theoretically, and empirically investigating the approach in combination with deep reinforcement learning.
More generally, recent developments in Monte Carlo tree search and related ideas in planning \citep{busoniu2012optimistic,busoniu2012survey,feldman2013monte,feldman2014simple,munos2014bandits,szorenyi2014optimistic,feldman2014mabs,efroni2018beyond,efroni2019combine,dalal2021improve} can all be viewed as sitting between greedy improvement and computation of the exact optimal policy, and have the potential to be profitably combined with GHMs and \GSPs.

\textbf{Option models.} Modelling discounted visitation distributions was proposed by \citet{sutton1995td}, who termed them $\beta$-models. These models were generalised by \citet{precup1998multi}, who proposed multi-time models, which encompass both $\beta$-models and $n$-step models as special cases. More generally, there is a long-established practice of learning option models \citep{sutton1999between,precup1998theoretical,precup2000temporal}, and using such models in a compositional manner \citep{silver2012compositional}.
A central difference between option models and this work is that
the use of geometric switching times (or in the language of options, constant termination probabilities)
means we do not need to model accumulated return obtained by each base policy, or the time taken executing each base policy,
making applications to transfer possible.
In this regard, the approach of this paper is related to universal option models \citep{yao2014universal}, which also disentangle reward and transition structure; constant termination probabilities more easily facilitate sample-based composition of such models.
Although orthogonal to the direction of this work, the problem of \emph{option discovery} is central to hierarchical RL \cite{mcgovern2001automatic,menache2002q,csimcsek2004using,brunskill2014pac,kulkarni2016hierarchical,machado2017laplacian,harb2018waiting,harutyunyan2019termination,wulfmeier2021data}, and is clearly relevant here too, essentially posing the question of where the base policies supplied to \ggpi should come from.

\textbf{The successor representation and visitation distributions.}
Discounted visitation distributions are closely related to the successor representation (SR), introduced by \citet{dayan1993improving}, who also proposed a temporal-difference method for learning the SR.
As discussed above, \citet{janner2020gamma} introduce several methods for learning approximate discounted visitations on continuous state spaces, among other contributions. Several other recent works also target this problem.  \citet{blier2021learning} propose several methods for learning parametric approximations to discounted visitation distributions, including an approach based on low-rank approximations. Building on this work, \citet{touati2021learning} propose a compact representation of an MDP that in principle allows for the optimal policy associated with any reward function to be computed without planning, in practice relying on a 
low-dimensional
approximation 
of the visitation
distributions.
\citet{eysenbach2021c} propose an approach based on contrastive learning; these works also note a close connection with the domain of
goal-conditioned RL \citep{kaelbling1993learning,schaul2015universal,andrychowicz2017hindsight,pong2018temporal}.

\textbf{Successor features and GPI.} \citet{barreto2016successor} introduced successor features, a generalisation of the successor representation, and GPI, in the context of transfer; later \citet{pmlr-v80-barreto18a} discussed the practicalities involved in combining the approach with deep learning. The same conceptual machinery was then used by \citet{barreto2019optionkeyboard} to promote temporal abstraction in RL. \citet{borsa2019universal} introduced a generalised form of successor features that has a representation of a policy as one of their inputs, thus allowing generalisation along the space of policies. \citet{hunt2019composing} extended successor features to entropy-regularized RL and addressed some of the challenges involved in applying GPI to continuous action spaces. \citet{grimm2019disentagled} and \citet{hansen2020fast} propose approaches that allow the features used in successor features to be learned from data in the absence of a reward signal. \citet{zahavy2021discovering} and \citet{alver2022constructing} studied the problem of how to construct a good set of policies to be used with GPI.
\citet{lehnert2020successor} showed how successor features can be seen as a link between model-free and model-based RL. \citet{nemecek2021policy} studied a related problem: given a set of successor features and a reward function, they showed how to estimate the performance of the associated GPI policy and use this estimate to decide whether to add new successor features to the set. Recently, \citet{barreto2020fast} presented a comprehensive account of GPI and successor features in which the latter are cast as a special case of a more general concept called \emph{generalised policy evaluation} (GPE). We believe GHMs can be understood as an alternative form of GPE.

\textbf{Non-Markov policies.} Non-Markov/homogeneous policies are used in several other sub-fields of reinforcement learning in MDPs. \citet{scherrer2012use,lesner2015non} consider approximate value iteration, policy iteration, and modified policy iteration algorithms, proposing the use of non-homogeneous policies that repeatedly cycle through a sequence of recent greedy Markov policies, and showing that such policies obtain improved performance bounds. In contrast, \ggpi always produces a Markov policy, but one which improves upon non-Markov policies. Non-Markov policies are also commonly-encountered in exploration, for example via action repetition \citep{dabney2020temporally}, and Thompson sampling and its approximations and variations  \citep{strens2000bayesian,osband2013more,osband2016deep,agrawal2017optimistic,russo2018tutorial}.

\subsection{Successor features, the successor representation, and geometric horizon models}\label{sec:appendix:sr}

We provide some additional discussion regarding the relationship between the successor representation, successor features, and geometric horizon models in the case of finite state spaces $\mathcal{X}$. For ease of comparison, we phrase all three concepts in terms of variants that condition on an initial state-action pair, although the successor representation was originally introduced as a state-indexed quantity.

\citet{dayan1993improving} introduced the successor representation in reinforcement learning. In the context of discounted MDPs, the definition is as follows.

\begin{definition}
    For a given policy $\pi : \mathcal{X} \rightarrow \mathscr{P}(\mathcal{A})$, 
    the corresponding \emph{successor representation} of a state-action pair $(x, a) \in \mathcal{X} \times \mathcal{A}$ is the vector
    \begin{align*}
        \lambda^\pi(x, a) = \mathbb{E}^\pi_{x, a} \Big \lbrack\sum_{k=0}^\infty \gamma^k e_{X_k} \Big\rbrack \in \mathbb{R}^\mathcal{X} \, ,
    \end{align*}
    where $e_{x'} \in \mathbb{R}^\mathcal{X}$ is the one-hot vector for the coordinate $x'$.
\end{definition}

We can view $\lambda^\pi(x, a)$ as an unnormalised probability distribution; scaling by a factor of $1-\gamma$ yields a probability distribution that corresponds to sampling a time $T \sim \text{Geometric}(1-\gamma)$, and then sampling $T-1$ transition steps in the environment under $\pi$, initialised at the state-action pair $(x, a)$.

\citet{barreto2016successor} introduced successor features as a generalisation of the successor representation.

\begin{definition}
    Consider a base feature map $\phi : \mathcal{X} \times \mathcal{A} \times \mathcal{X} \rightarrow \mathbb{R}^K$. For a given policy $\pi : \mathcal{X} \rightarrow \mathscr{P}(\mathcal{A})$, 
    the corresponding vector of \emph{successor features} of a state $x \in \mathcal{X}$ is the vector
    \begin{align*}
        \psi^\pi(x, a) = \mathbb{E}^\pi_{x, a}\Big \lbrack \sum_{t \geq 0} \gamma^t \phi(X_t, A_t, X_{t+1}) \Big \rbrack \in \mathbb{R}^K \, .    
    \end{align*}
\end{definition}

The successor representation is subsumed as a special case of successor features when $\phi(x, a, x') \in \mathbb{R}^{\mathcal{X}}$ is taken to be the basis vector for state $x$. The following result relates the discounted future state-visitation distributions of Definition~\ref{def:state-visitation} with successor features.

\begin{proposition}
\label{thm:sfs_ghms}
    The discounted future state-visitation distribution $\mu^\pi_\gamma$ is an instance of successor features, with the base feature map $\phi(x, a, x') = (1-\gamma)e_{x'} \in \mathbb{R}^{\mathcal{X}}$, where $e_{x'}$ is the one-hot vector for the coordinate $x'$.
\end{proposition}
\begin{proof}
    We directly calculate the $x'$ coordinate of $\psi^\pi(x, a)$ as:
    \begin{align*}
        \mathbb{E}^\pi_{x, a}\Big \lbrack \sum_{t \geq 0} \gamma^t \mathbbm{1}\{ X_{t+1} = x' \} \Big \rbrack
        & = \mathbb{E}^\pi_{x, a}\Big \lbrack \sum_{t \geq 0} \gamma^t (1-\gamma) \mathbbm{1}\{ X_{t+1} = x' \} \Big \rbrack \\
        & \overset{(a)}{=} \sum_{t \geq 0} \gamma^t (1-\gamma) \mathbb{E}^\pi_{x, a}\Big \lbrack  \mathbbm{1}\{ X_{t+1} = x' \}  \Big \rbrack \\
        & = (1-\gamma) \sum_{t \geq 0} \gamma^t \mathbb{P}^\pi_{x_,a}(  X_{t+1} = x' ) \\
        & = \mu^\pi_\gamma(x'|x, a) \, ,
    \end{align*}
    where the swapping of summation and expectation in (a) is justified by the dominated convergence theorem, since the integrand is bounded.
\end{proof}

Proposition~\ref{thm:sfs_ghms} sheds light on the relationship between successor features and GHMs in the case of a finite state space $\mathcal{X}$. When using the features $\phi(x, a, x') = (1-\gamma)e_{x'}$, the successor features of policy $\pi$ become the $\gamma$-discounted state-visitation distribution of $\pi$---that is, $\psi^\pi(x,a) = \mu^\pi_\gamma(\cdot| x, a)$; the corresponding GHM is a generative model of this distribution.

\section{Proofs relating to geometric horizon models and generalised policy improvement}\label{sec:appendix:proofs}

\subsection{Proofs of results in Section~\ref{sec:ghm}}

{
\renewcommand\footnote[1]{}
\propGeometricInterpretation*
}

\begin{proof}
    We have
    \begin{align*}
        \mathbb{P}^\pi_{x, a}(X_T = x') & = \mathbb{E}[\mathbb{P}^\pi_{x, a}(X_T = x' \mid T)] \\
        & = \sum_{k=1}^\infty \mathbb{P}(T=k)\mathbb{P}^\pi_{x, a}(X_k = x' \mid T=k)\\
        & = \sum_{k=1}^\infty (1-\gamma) \gamma^{t-1} \mathbb{P}^\pi_{x, a}(X_k = x')\\
        & = \mu^\pi_\gamma(x'|x, a) \, ,
    \end{align*}
    as required.
\end{proof}

\propBasicEval*

\begin{proof}
    We have
    \begin{align*}
        Q^\pi_\gamma(x, a)&  = \mathbb{E}^\pi_{x, a}\left\lbrack\sum_{t = 0}^\infty \gamma^t R_t \right\rbrack \\
        & = \mathbb{E}^\pi_{x, a}[R_0] + \mathbb{E}^\pi_{x, a}\Big\lbrack\sum_{t = 1}^\infty \gamma^t R_t\Big\rbrack \\
        & = r(x,a) + \gamma \sum_{t =1}^\infty  \gamma^{t-1} \sum_{x' \in \mathcal{X}} \mathbb{P}^\pi_{x, a}(X_t = x' ) r^\pi(x') \\
        & \overset{(a)}{=} r(x, a) + \gamma \sum_{x' \in \mathcal{X}} \sum_{t=0}^\infty \gamma^t \mathbb{P}^\pi_{x, a}(X_{t+1} = x') r^\pi(x') \\
        & = r(x, a) + \gamma (1-\gamma)^{-1} \sum_{x' \in \mathcal{X}} \mu^\pi_\gamma(x'|x, a) r^\pi(x') \\
        & = r(x, a) +  \gamma (1-\gamma)^{-1} \mathbb{E}_{X' \sim \mu^\pi_\gamma(\cdot|x, a)}[ r^\pi(X') ] \, .
    \end{align*}
    as required. The switching of the order of summation at (a) can be justified, for example, by noting that the double-sum is absolutely convergent:
    \begin{align*}
        \sum_{t=1}^\infty \sum_{x' \in \mathcal{X}} \left| \gamma^{t-1} \mathbb{P}^\pi_{x, a}(X_t = x') r^\pi(x') \right| \leq \sum_{t=1}^\infty \gamma^{t-1}  R^\pi_{\text{max}} = R^\pi_\text{max} (1-\gamma)^{-1} < \infty \, .
    \end{align*}
    where $R^\pi_\text{max} = \max_x |r^\pi(x)| < \infty$, as $|\mathcal{X}|$ is finite.
\end{proof}

\subsection{Proof of result from Section~\ref{sec:ghm-markov}}

Below, we re-derive a result essentially equivalent to Theorem~2 of \citet{janner2020gamma}, stated as Proposition~\ref{prop:markov-eval} in our main paper, with a slightly different proof technique.
The central idea is to develop a different way of sampling the random variable $X_T$ appearing in Proposition~\ref{prop:geometric-interpretation}, using the following results.

\begin{restatable}{lemma}{lemGeomRandomSum}\label{lem:geometric-random-sum}
    Let $(T_i)_{i =1}^\infty \overset{\text{i.i.d.}}{\sim} \text{Geometric}(1-\beta)$, and independently, $N \sim \text{Geometric}(\nicefrac{1 - \gamma}{1 - \beta})$. Then the random sum $\sum_{i=1}^N T_i$ has distribution $\text{Geometric}(1-\gamma)$.
\end{restatable}

\begin{restatable}{lemma}{lemmaGeometricFixedSum}\label{lem:geometric-fixed-sum}
    Let $(T_i)_{i=1}^{n-1} \overset{\text{i.i.d.}}{\sim} \text{Geometric}(1-\beta)$, and independently, $T' \sim \text{Geometric}(1-\gamma)$, and $N'$ a random variable taking variables in $\{1,\ldots,n\}$, with probabilities
    \begin{align*}
        \mathbb{P}(N'=m) = \frac{1 - \gamma}{1-\beta} \left(\frac{\gamma - \beta}{1 - \beta}\right)^{m-1} \text{ for }m=1,\ldots,n-1 \, , \text{ and } \mathbb{P}(N'=n) = \left(\frac{\gamma - \beta}{1 - \beta}\right)^{n-1} \, .
    \end{align*}
    Then the random sum
    \begin{align*}
        \sum_{i=1}^{\min(N', n-1)} T_i + \mathbbm{1}\{N' = n\} T'
    \end{align*}
    has distribution $\text{Geometric}(1-\gamma)$.
\end{restatable}

\begin{restatable}{proposition}{propComposeDistribution}\label{prop:compose-distribution}
    If we define a sequence of states and actions $(X^{(n)}, A^{(n)})_{n \geq 0}$ inductively by $(X^{(0)}, A^{(0)}) = (x_0, a_0)$, $X^{(n+1)} \sim \mu^\pi_\beta(\cdot|X^{(n)}, A^{(n)})$, $A^{(n+1)} \sim \pi(\cdot|X^{(n+1)})$, then 
    $X^{(n)} \overset{\mathcal{D}}{=} X_{\sum_{i=1}^n T_i}$.
\end{restatable}

We also note that using different distributional identities for the random variable $T$ leads to variants of the result given in Proposition~\ref{prop:markov-eval}. For example, directly using the distributional identity in Lemma~\ref{lem:geometric-random-sum} can be used to establish a version of Theorem~1 of \citet{janner2020gamma} using exactly the same proof technique as for Proposition~\ref{prop:markov-eval}.

\begin{proof}[Proof of Lemma~\ref{lem:geometric-random-sum}]
    This is a classical result from elementary probability theory.
    We work with probability generating functions. The probability generating function of a random variable $Z$ taking values in $\mathbb{N}$ is defined as the function $G_Z(s) = \mathbb{E}[s^Z] = \sum_{k=1}^\infty \mathbb{P}(Z=k) s^k$, and clearly characterises the distribution of $Z$.
    
    A standard calculation shows that for $T \sim \text{Geometric}(1-\gamma)$, we have
    \begin{align*}
        G_T(s) = \frac{s(1-\gamma)}{1-s\gamma} \, , \text{ for } |s| < \gamma^{-1} \, .
    \end{align*}
    We also have the following standard relationship for the PGF of a random sum of i.i.d.~terms:
    \begin{align*}
        G_{\sum_{i=1}^N T_i}(s) = \mathbb{E}[s^{\sum_{i=1}^N T_i}] = \mathbb{E}[\mathbb{E}[s^{\sum_{i=1}^N T_i} \mid N]] = \sum_{n=1}^\infty \mathbb{P}(N=n)\mathbb{E}[s^{\sum_{i=1}^n T_i}] = \sum_{n=1}^\infty \mathbb{P}(N=n)G_{T_1}(s)^n = G_N(G_{T_1}(s)) \, .
    \end{align*}
    Since both $N$ and $T_1$ have geometric distributions, we can directly calculate
    \begin{align*}
        G_N(G_{T_1}(s)) = G_N\left( \frac{s(1-\beta)}{1 - s\beta} \right) = \frac{\frac{s(1-\beta)}{1-s\beta}\left( \frac{1-\gamma}{1-\beta}\right)}{1 - \frac{s(1-\beta)}{1-s\beta}\left(1 - \frac{1-\gamma}{1-\beta}\right)} = \frac{s(1-\gamma)}{1-s\gamma} \, , 
    \end{align*}
    for $|s| < \gamma^{-1}$, which is the probability generating function of a $\text{Geometric}(1-\gamma)$ random variable, as required.
\end{proof}

\begin{proof}[Proof of Lemma~\ref{lem:geometric-fixed-sum}]
    This follows as a straightforward corollary of Lemma~\ref{lem:geometric-random-sum}; under the notation of that result, we have $T \overset{\mathcal{D}}{=} \sum_{i=1}^N T_i$. We now decompose this based on whether the event $\{N \geq n\}$ occurs, and use the fact that $\mathbb{P}(N=k) = \mathbb{P}(N'=k)$ for $k=1,\ldots,n-1$:
    \begin{align*}
        T
        \overset{\mathcal{D}}{=} \sum_{i=1}^{\min(N, n-1)} T_i + \mathbbm{1}\{N \geq n\} \sum_{i=n}^N T_i 
        \overset{\mathcal{D}}{=} \sum_{i=1}^{\min(N', n-1)} T_i + \mathbbm{1}\{N' = n\} T'  \, ,
    \end{align*}
    as required. The final equality in distribution holds from the memoryless property of the geometric distribution; on the event $\{N \geq n\}$, we have $N - (n-1) \sim \text{Geometric}(\nicefrac{1-\gamma}{1-\beta})$, and hence $\sum_{i=n}^N T_i \sim \text{Geometric}(1-\gamma)$ on this event.
\end{proof}

\begin{proof}[Proof of Proposition~\ref{prop:compose-distribution}]
    This follows straightforwardly by induction. The case $n=1$ follows from Proposition~\ref{prop:geometric-interpretation}. Now suppose the claim holds for $n=l$. Then we have $X^{(l)} \overset{\mathcal{D}}{=} X_{\sum_{i=1}^n T_i}$. So
    \begin{align*}
        X^{(l+1)} | X^{(l)}, A^{(l)} \sim \mu^\pi_\beta(\cdot|X^{(l)}, A^{(l)}) \, ,
    \end{align*}
    and so by Proposition~\ref{prop:geometric-interpretation} again, we have $X^{(l+1)} | X^{(l)} \overset{\mathcal{D}}{=} X'_{T'}$, with $T' \sim \text{Geometric}(1-\beta)$, and $(X'_t, A'_t, R'_t)_{\geq 0}$ an independent trajectory following $\pi$ with initial state $X^{(l)}$. But since $X^{(l)} \overset{\mathcal{D}}{=} X_{\sum_{i=1}^n T_i}$, by the Markov property we therefore have $X^{(l+1)} \overset{\mathcal{D}}{=} X_{\sum_{i=1}^n T_i + T'}  \overset{\mathcal{D}}{=} X_{\sum_{i=1}^{n+1} T_i}$ as required.
\end{proof}

We now restate and prove Proposition~\ref{prop:markov-eval}.

\propMarkovEval*

\begin{proof}
    We start from the expression for $Q^\pi_\gamma(x, a)$ in Equation~\ref{eq:gamma-model-eval-1}.
    Using the notation of Proposition~\ref{prop:geometric-interpretation}, we have
    \begin{align*}
        \mathbb{E}_{X' \sim \mu^\pi_\gamma(\cdot|x, a)}[r^\pi(X')] = \mathbb{E}^\pi_{x, a}[r^\pi(X_T)] \, .
    \end{align*}
    Now with the notation of Lemma~\ref{lem:geometric-fixed-sum}, we have
    \begin{align*}
        \mathbb{E}^\pi_{x, a}[r^\pi(X_T)] & = \mathbb{E}^\pi_{x, a}[r^\pi(X_{\sum_{i=1}^{\min(N', n-1)} T_i + \mathbbm{1}\{N' = n\} T'})] \\
        & = \mathbb{E}[\mathbb{E}^\pi_{x, a}[ r^\pi(X_{\sum_{i=1}^{\min(N', n-1)} T_i + \mathbbm{1}\{N' = n\} T'}) \mid N'] ] \\
        & = \sum_{m=1}^{n-1} \frac{1-\gamma}{1 - \beta}\left(\frac{\gamma - \beta}{1-\beta}\right)^{m-1} \mathbb{E}^\pi_{x, a}[r^\pi(X_{\sum_{i=1}^m T_i})] + \left(\frac{\gamma - \beta}{1- \beta}\right)^{n-1} \mathbb{E}^\pi_{x, a}[r^\pi(X_{\sum_{i=1}^{n-1} T_i + T'})] \, .
    \end{align*}
    Finally, by Proposition~\ref{prop:compose-distribution}, we have $X_{\sum_{i=1}^m T_i} \overset{\mathcal{D}}{=} X^{(m)}$ as defined above, and $X_{\sum_{i=1}^{n-1} T_i + T'} \overset{\mathcal{D}}{=} X'$, to obtain the desired conclusion.
\end{proof}

\subsection{Proofs of result from Section~\ref{sec:non-markov-eval}}

\propSMPEval*

\begin{proof}
     Just as with Markov policies, we have the basic identity
    \begin{align*}
        Q^\nu_\gamma(x, a) = r(x) + \frac{\gamma}{1-\gamma} \mathbb{E}^\nu_{x, a}[r(X_T) ] \, .
    \end{align*}
    We now show that $\mathbb{E}^{\nu}_{x, a}[r(X_T)]$ has the required form by induction on $n$. The base case $n=1$ follows from Proposition~\ref{prop:basic-eval}. For the inductive step, fix $n=l$, and suppose the required form of the expectation has been demonstrated for all smaller values of $n$.
    
    Let $\nu =\pi_1 \overset{\alpha}{\rightarrow} \cdots \overset{\alpha}{\rightarrow} \pi_l$. We consider the time to switch from the first policy $\pi_1$, to the second sampled policy, $\pi_2$, denoting this time $T_1$, recalling that its distribution is $\text{Geometric}(\switchprob)$. 
    We proceed by considering whether or not the geometric horizon $T \sim \text{Geometric}(1-\gamma)$ is greater than $T_1$:
    \begin{align}\label{eq:decomp1}
        & \mathbb{E}^\nu_{x, a}[r(X_T) ] \\
        = & \mathbb{E}^\nu_{x, a}[r(X_T) \mathbbm{1}_{T \leq T_1} + r(X_T) \mathbbm{1}_{T > T_1} ] \nonumber \\
        = & \mathbb{E}^{\nu}_{x, a}[r(X_T) \mid T \leq T_1] \mathbb{P}(T \leq T_1 \mid X_0 = x, A_0 = a) \nonumber \\
        & \quad\quad+ \mathbb{E}^{\nu}_{x, a}[r(X_T) \mid T > T_1] \mathbb{P}^\nu_{x, a}(T > T_1) \, . \nonumber
    \end{align}
    Since $T$, $T_1$ are independent of the trajectory $(X_t, A_t, R_t)_{t \geq 0}$, we have $\mathbb{P}^\nu_{x, a}(T \leq T_1) = \mathbb{P}(T \leq T_1)$. To compute $\mathbb{P}(T \leq T_1)$, we have
    \begin{align*}
        \mathbb{P}(T \leq T_1) & = \sum_{k=1}^\infty \mathbb{P}(T \leq k) \mathbb{P}(T_1 = k) \\
        & = \sum_{k=1}^\infty (1 - \gamma^k) \switchprob (1-\switchprob)^{k-1} \\
        & = \switchprob \sum_{k=1}^\infty ((1-\switchprob)^{k-1} - \gamma (\gamma (1-\switchprob))^{k-1}) \\
        & = \switchprob \left( \frac{1}{\switchprob} - \frac{\gamma}{1-\gamma(1-\switchprob)}\right) \\
        & = \frac{1 - \gamma}{1-\gamma(1-\switchprob)} \, .
    \end{align*}
    Now, to compute $\mathbb{E}^{\nu}_{x, a}[r(X_T) \mid T \leq T_1]$, we need the marginal distribution of $T$ given the event $\{T \leq T_1\}$, which again is independent of the trajectory $(X_t, A_t, R_t)_{t \geq 0}$. We have
    \begin{align*}
        \mathbb{P}(T = k \mid T \leq T_1) & \propto \mathbb{P}(T = k, T \leq T_1) \\
        & = \sum_{l=k}^\infty \mathbb{P}(T=k) \mathbb{P}(T_1 =l) \\
        & = (1-\gamma) \gamma^{k-1} (1-\switchprob)^k \\
        & \propto (\gamma (1-\switchprob))^k \, ,
    \end{align*}
    which is the probability mass function of a $\text{Geometric}(1-\gamma(1-\switchprob))$ distribution. Hence, conditional on $T \leq T_1$, we have that $T \sim \text{Geometric}(1-\gamma(1-\switchprob))$, and that the policy $\nu$ has not switched from $\pi_1$ on this event, so 
    \begin{align*}
        \mathbb{E}^\nu_{x, a}[r(X_T) \mid T \leq T_1 ] = \mathbb{E}_{X' \sim \mu_{\gamma(1-\switchprob)}^{\pi_1}(\cdot|x, a)}[ r(X') ] \, .
    \end{align*}
    
    We next turn our attention to the second term on the right-hand side of Equation~\eqref{eq:decomp1}. Conditional on $\{T > T_1\}$, we compute the joint distribution of $(T - T_1, T_1)$. For any $k,l > 0$:
    \begin{align*}
        \mathbb{P}(T - T_1=k, T_1=l \mid T > T_1)  \propto \mathbb{P}(T - T_1=k, T_1=l) 
         = \mathbb{P}(T=k+l, T_1 =l) 
         \propto \gamma^{k+l} (1-\switchprob)^l 
         =  \gamma^k (\gamma (1-\switchprob))^l \, ,
    \end{align*}
    which we recognise as the distribution of two independent geometric random variables with parameters $1-\gamma$ and $1-\gamma(1-\switchprob)$. Hence, a sample from $X_T$ on the event $\{T > T_1\}$ can be obtained by first sampling the state $X^{(1)} \sim \mu^{\pi_1}_{\gamma(1-\switchprob)}$ at which the switch from $\pi_1$ to $\pi_2$ occurs. From this point, we require a state sampled $T - T_1 \sim \text{Geometric}(1-\gamma)$ steps into the future, from initial state $X^{(1)}$, and action $A^{(1)} \sim \pi_2(\cdot|X^{(1)})$, following the suffix \GSP $\nu' = \pi_2 \overset{\alpha}{\rightarrow} \cdots \overset{\alpha}{\rightarrow} \pi_l$. By induction, the corresponding expectation can be expressed as
    \begin{align*}
        \mathbb{E}^\nu_{x, a}[r(X_T) \mid T > T_1] = \mathbb{E}\Big \lbrack \sum_{m=1}^{l-2} \frac{1-\gamma}{1-\beta}\left( \frac{\gamma - \beta}{1 - \beta} \right)^{m-1} r(\bar{X}^{(m)}) + \left( \frac{\gamma - \beta}{1 - \beta} \right)^{l-2} r(\bar{X}') \Big\rbrack \, ,
    \end{align*}
    where $\bar{X}^{(0)} \sim \mu_\beta^{\pi_1}(\cdot|x, a)$, $\bar{X}^{(m)} \sim \mu^{\pi_{m+1}}_{\beta}(\cdot|\bar{X}^{(m\indexminus 1)}, \bar{A}^{(m \indexminus 1)})$, $\bar{A}^{(m)} \sim \pi_{m+2}(\cdot|\bar{X}^{(m)})$, $\bar{X}' \sim \mu^{\pi_l}_\gamma(\cdot|\bar{X}^{(l-2)}, \bar{A}^{(l-2)})$.
    
    Rewriting in terms of the original sequence $(X^{(0)}, X^{(1)}, \ldots, X^{(n)}, A^{(n)}, X')$ in the theorem statement, we have
    \begin{align*}
        \mathbb{E}^\nu_{x, a}[r(X_T) \mid T > T_1] = \mathbb{E}\Big \lbrack \sum_{m=1}^{l-2} \frac{1-\gamma}{1-\beta}\left( \frac{\gamma - \beta}{1 - \beta} \right)^{m-1} r(X^{(m+1)}) + \left( \frac{\gamma - \beta}{1 - \beta} \right)^{l-2} r(X') \Big\rbrack \, .
    \end{align*}
    Putting everything together from the decomposition in Equation~\eqref{eq:decomp1}, we therefore have
    \begin{align*}
        & \mathbb{E}^\nu_{x,a}[r(X_T) ]
        \\
        =& \frac{1-\gamma}{1 -\gamma\beta} \mathbb{E}[r(X^{(1)})] + \frac{\gamma - \beta}{1 - \beta}\mathbb{E}\Big \lbrack \sum_{m=1}^{l-2} \frac{1-\gamma}{1-\beta}\left( \frac{\gamma - \beta}{1 - \beta} \right)^{m-1} r(X^{(m+1)}) + \left( \frac{\gamma - \beta}{1 - \beta} \right)^{l-2} r(X') \Big\rbrack \\
         =&  \mathbb{E} \Big\lbrack  \sum_{m=1}^{l-1} \frac{1-\gamma}{1 - \beta}\left(\frac{\gamma - \beta}{1-\beta}\right)^{m-1} r(X^{(m)})
        + \left(\frac{\gamma - \beta}{1-\beta}\right)^{l-1} r(X') \Big\rbrack
    \end{align*}
    as required.
\end{proof}

\subsection{Proof of result from Section~\ref{sec:nmgpi}}

\propnmgpi*

\begin{proof}
    It is sufficient to show that for any policy $\nu \in \Pi$, we have $Q^{\pi'} \geq Q^{\nu}$. If $\nu = \pi$ is Markov, then we have
    \begin{align*}
        Q^\pi_\gamma (x, a) = r(x, a) + \gamma \sum_{x' \in \mathcal{X}} \sum_{a' \in \mathcal{A}} P(x'|x, a) \pi(a'|x') Q^\pi_\gamma(x', a') \, ,
    \end{align*}
    and hence
    \begin{align*}
        Q^\pi_\gamma (x, a) \leq r(x, a) + \gamma \sum_{x' \in \mathcal{X}} \sum_{a' \in \mathcal{A}} P(x'|x, a) \pi'(a'|x') \max_{\bar{\nu} \in \Pi} Q^{\bar{\nu}}_\gamma(x', a') = (T^{\pi'}( \max_{\bar{\nu} \in \Pi} Q^{\bar{\nu}}_\gamma ))(x, a) \, .
    \end{align*}
    Now suppose $\nu= \pi_1 \overset{\alpha}{\rightarrow} \cdots \overset{\alpha}{\rightarrow} \pi_n \in \Pi$ is a non-Markov \gsp.
    Let $\nu'= \pi_2 \overset{\alpha}{\rightarrow} \cdots \overset{\alpha}{\rightarrow} \pi_n$ be the suffix policy of $\pi$. By suffix-closedness of $\Pi$, $\nu' \in \Pi$, and so we have the following observation:
    \begin{align*}
        Q^{\nu}_\gamma(x, a)& =  r(x, a) + \gamma \sum_{x' \in\mathcal{X}} P(x'|x, a) \left\lbrack (1-\switchprob)  \sum_{a' \in \mathcal{A}} \pi_1(a'|x') Q^{\nu}_\gamma(x', a') + \switchprob \sum_{a' \in \mathcal{A}} \pi_2(a'|x') Q^{\nu'}_\gamma(x', a')  \right\rbrack \\
        & \leq  r(x, a) + \gamma \sum_{x' \in\mathcal{X}} P(x'|x, a) \left\lbrack (1-\switchprob)  \sum_{a' \in \mathcal{A}} \pi'(a'|x') \max_{\bar{\nu} \in \Pi} Q^{\bar{\nu}}_\gamma(x', a') + \switchprob\sum_{a' \in \mathcal{A}} \pi'(a'|x') \max_{\bar{\nu} \in \Pi} Q^{\bar{\nu}}_\gamma(x', a')  \right\rbrack \\
        & = (T^{\pi'} (\max_{\bar{\nu} \in \Pi} Q^{\bar{\nu}}_\gamma ))(x, a) \, ,
    \end{align*}
    similarly to the Markov case. By taking a maximum over the policy considering on the left-hand side of the main chain of inequalities above, we get $\max_{\bar{\nu} \in \Pi} Q^{\bar{\nu}} \leq T^{\pi'} (\max_{\bar{\nu} \in \Pi} Q^{\bar{\nu}})$. As in the proof of improvement guarantee for standard GPI, we have that $T^{\pi'}$ is monotone, and contracts to $Q^{\pi'}$. Hence, $Q^\nu \leq \max_{\nu \in \Pi} Q^\nu \leq \lim_{n \rightarrow \infty} (T^{\pi'})^n (\max_{\nu \in \Pi} Q^\nu) = Q^{\pi'}$, as required. For the final statement of the result, observe that if equality holds at all state-action pairs, then we have that $\max_{\nu \in \Pi} Q^\nu_\gamma$ satisfies the Bellman optimality equation $\max_{\nu \in \Pi} Q^\nu_\gamma = T^{\pi'} \max_{\nu \in \Pi} Q^\nu_\gamma = T^* \max_{\nu \in \Pi} Q^\nu_\gamma$, and hence $\max_{\nu \in \Pi} Q^\nu_\gamma = Q^{\pi'} = Q^*$, so $\pi'$ is optimal.
\end{proof}

\subsection{Proof of result from Section~\ref{sec:applications_description}}

\propPiMClosed*

\begin{proof}
    Given a policy $\nu = \pi^{(1)} \overset{\alpha}{\rightarrow} \cdots \pi^{(m)} \in \Pi_m$, its suffix policy is $\nu' = \pi^{(2)} \overset{\alpha}{\rightarrow} \cdots \pi^{(m)}$. On the face of it, this policy appears not to lie in $\Pi_m$, since it contains only $m-2$ switches. However, the key observation is that appending an additional switch from the tail Markov policy to itself does not change the geometric switching policy; that is 
    \begin{align*}
        \pi^{(2)} \overset{\alpha}{\rightarrow} \cdots \pi^{(m-1)} \overset{\alpha}{\rightarrow} \pi^{(m)} = \pi^{(2)} \overset{\alpha}{\rightarrow} \cdots \pi^{(m-1)} \overset{\alpha}{\rightarrow} \pi^{(m)} \overset{\alpha}{\rightarrow} \pi^{(m)} \, .
    \end{align*}
    The right-hand side clearly lies in $\Pi_m$, and hence the proof of suffix-closedness is complete. The improvement guarantee now follows from Theorem~\ref{prop:gnmpi}.
\end{proof}

\subsection{Proof of result from Section~\ref{sec:learning_ghm}}

Here, we provide a proof of Proposition~\ref{prop:bellman}, and note that the (longer) proof of Theorem~\ref{thm:cetd} is given in Appendix~\ref{sec:appendix:cetd-convergence}.

\propBellman*

\begin{proof}
    That $\mu^\pi_\beta$ solves $\mu = T^\pi_\beta \mu$ follows straightforwardly from the Markov property of the environment:
    \begin{align*}
        \mu^\pi_\beta(x'|x, a) = & (1-\beta)\mathbb{E}^\pi_{x, a}\Big\lbrack \sum_{t \geq 0} \beta^t \mathbbm{1}_{X_{t+1} = x'} \Big\rbrack \\
        = & (1-\beta)\mathbb{E}^\pi_{x, a}\Big \lbrack \mathbbm{1}_{X_{t+1} = x'} \Big\rbrack + \beta \mathbb{E}^\pi_{x, a}\Big\lbrack (1- \beta)\mathbb{E}^\pi_{X_1, A_1}\Big\lbrack \sum_{t \geq 1} \beta^t \mathbbm{1}_{X_{t+1} = x'} \Big\rbrack\Big\rbrack \\
        = & (1-\beta)P(x'|x, a) + \beta \mathbb{E}^\pi_{x, a}\Big\lbrack \mu(x'|X_1, A_1)\Big\rbrack \\
        = & (1-\beta)P(x'|x, a) + \beta \sum_{x'' \in \mathcal{X}} \sum_{a'' \in \mathcal{A}} P(x''|x, a) \pi(a''|x'')\mu(x'|x', a')\\
        = & (1-\beta)P(x'|x, a) + \beta (\mu \otimes_\pi P)(x'|x, a)
        \, .
    \end{align*}
    We now show that $T^\pi_\beta$ is a contraction mapping on $\mathscr{P}(\mathcal{X})^{\mathcal{X} \times \mathcal{A}}$.
    Let $\mu, \mu' \in \mathscr{P}(\mathcal{X})^{\mathcal{X} \times \mathcal{A}}$, from which uniqueness of the solution to $\mu = T^\pi_\beta \mu$ immediately follows. We directly calculate
    \begin{align*}
        (T^\pi_\beta \mu - T^\pi_\beta \mu')(x'|x, a) & = \left((1-\beta) P(x'|x, a) + \beta \sum_{x'' \in \mathcal{X}} \sum_{a' \in \mathcal{A}} P(x''|x, a) \pi(a'|x'') \mu(x'|x'', a')\right) -  \\
        & \quad\quad\quad\quad\left((1-\beta) P(x'|x, a) + \beta \sum_{x'' \in \mathcal{X}} \sum_{a' \in \mathcal{A}} P(x''|x, a) \pi(a'|x'') \mu'(x'|x'', a')\right) \\
        & = \beta \sum_{x'' \in \mathcal{X}} P(x''|x, a) \pi(a'|x'') (\mu(x'|x'', a') - \mu'(x'|x'', a')) \, .
    \end{align*}
    Hence,
    \begin{align*}
        \max_{(x, a, x') \in \mathcal{X} \times \mathcal{A} \times \mathcal{X}} |(T^\pi_\beta \mu - T^\pi_\beta \mu')(x'|x, a)| \leq \beta \max_{(x, a, x') \in \mathcal{X} \times \mathcal{A} \times \mathcal{X}} |(\mu -  \mu')(x'|x, a)| \, ,
    \end{align*}
    as required.
\end{proof}

\section{Proof of the convergence of cross-entropy temporal-difference learning}\label{sec:appendix:td-mle-proof}\label{sec:appendix:cetd-convergence}

In this section we prove Theorem~\ref{thm:cetd}, which establishes the convergence of cross-entropy TD learning in the tabular, finite state-space setting, under mild conditions. The broad structure of the proof follows that of many arguments in stochastic approximation: defining a Lyapnuov function, showing convergence of this Lyapunov function to 0 as the algorithm progresses via the Robbins-Siegmund theorem \citep{robbins1971convergence}, and deducing convergence of the algorithm as a consequence; see for example \citet{kushner2003stochastic} for further background. We begin by recalling the details of the theorem.

\textbf{Statement of result.} The algorithm generates a sequence of logits $(\phi_k)_{k \geq 0}$, with $\phi_k \in \mathbb{R}^{\mathcal{X}\times\mathcal{A}\times\mathcal{X}}$, and corresponding estimated geometric horizon models, denoted $\mu_k$, and defined by
\begin{align*}
    \mu_k(x'|x, a) = \frac{ \exp(\phi_k(x'|x, a)) }{ \sum_{x'' \in \mathcal{X}} \exp(\phi_k(x''|x, a)) } \, .
\end{align*}

We work with a synchronous algorithm, for which every state-action pair is updated at every algorithm time step. Thus, $\phi_0 \in \mathbb{R}^{\mathcal{X} \times \mathcal{A} \times \mathcal{X}}$ is initialised in some manner, and for each algorithm time step $k \geq 0$, for each $(x, a)$ we take a transition $(x, a, X')$ generated from the MDP, independent of all other transitions used at time $k$ and earlier, and define $\phi_{k+1}$ via the update
\begin{align}
    \phi_{k+1}(\cdot|x,a) = & \phi_k(\cdot|x, a) - \varepsilon_k \nabla_{\phi_k(\cdot|x, a)} \text{KL}( \texttt{SG}[ (\hat{T}^\pi \mu_k)(\cdot|x, a)]\ || \ \mu_k(\cdot|x, a) )  \, , \label{eq:logit-update}
\end{align}
where $\texttt{SG}$ denotes a stop-gradient, and $(\hat{T}^\pi \mu_k)(\cdot|X_k, A_k)$ is an unbiased approximation error to the Bellman operator application $(T^\pi \mu_k)(\cdot|X_k, A_k)$, given by
\begin{align*}
     (\hat{T}^\pi \mu_k)(\cdot|x, a) = (1-\gamma) e_{X'} + \gamma e_{X''}  \, ,
\end{align*}
where $X''$ is sampled first by sampling $A' \sim \pi(\cdot|X')$, and then $X'' \sim \mu_k(\cdot|X', A')$. Evaluating the gradient above allows us to re-express the update as
\begin{align}\label{eq:phi-update}
    \phi_{k+1}(\cdot|x,a) = & \phi_k(\cdot|x, a) + \varepsilon_k \left( (\hat{T}^\pi \mu_k)(\cdot|x, a) - \mu_k(\cdot|x, a) \right)  \, .
\end{align}
Then the theorem statement is that if the Robbins-Monro conditions for the step sizes $(\varepsilon_k)_{k=0}^\infty$ hold, then we have $\mu_k \rightarrow \mu^\pi_\gamma$ with probability 1.

\textbf{Proof.} The proof of the result is presented below. We include schematic illustrations of some of the key ideas in the proof in Figure~\ref{fig:proof-illustration}.

\begin{figure}[ht]
    \centering
    \includegraphics[keepaspectratio,width=.95\textwidth]{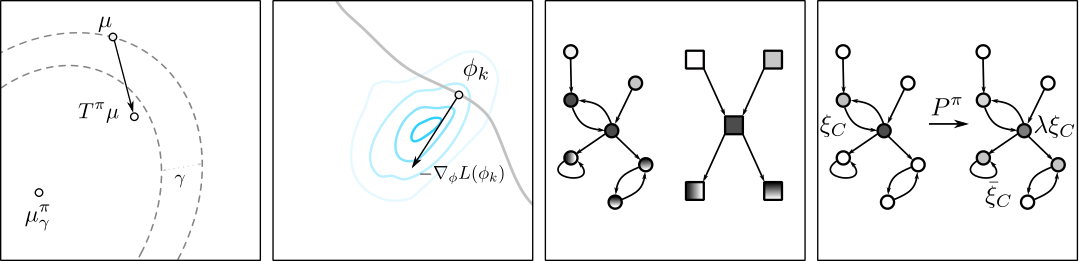}
    \caption{Schematic illustrations of core ideas in the convergence proof for cross-entropy temporal-difference learning.
    \textbf{Left:} Contractivity of the operator $T^\pi$ in the weighted $L^2$ norm $\| \cdot \|_\xi$ towards $\mu^\pi_\gamma$.
    \textbf{Centre-left: } For a given value of $\phi_k$, the corresponding level set of the Lyapunov function $L$ as a grey line, the conditional distribution over $\phi_{k+1}$ illustrated with blue contours, and the negative gradient of the Lyapunov function indicated as a black arrow. The Robbins-Siegmund argument shows that even though $\phi_{k+1}$ may have a higher Lyapunov value than $\phi_k$, in the long term the value of the Lyapunov function must converge to $0$.
    \textbf{Centre-right: } The decomposition of a Markov chain state space into a directed acyclic graph of communicating classes.
    \textbf{Right: } The distribution $\xi_C$ supported on a given communicating class $C$, as constructed via the Perron-Frobenius theorem, and the result of right-multiplying by the Markov chain transition matrix $P^\pi$; on the communicating class $C$, the distribution is scaled by $\lambda$, while descendant communicating classes may now have non-zero probabilities, given by $\bar{\xi}_C$.
    }
    \label{fig:proof-illustration}
\end{figure}

\textbf{The Lyapunov function.} 
Let $\xi$ be a stationary state-action distribution under $\pi$, and suppose initially that it has full support; we will explain how to remove this assumption below. It is useful to introduce the function $\logittoprob : \mathbb{R}^{\mathcal{X} \times \mathcal{A} \times \mathcal{X}} \rightarrow \mathscr{P}(\mathcal{X})^{\mathcal{X} \times \mathcal{A}}$ for the softmax function that maps logits to corresponding collections of probability distributions. We now define the Lyapunov function
\begin{align*}
    L(\phi) = \sum_{x,a} \xi(x, a) \text{KL}(\mu^\pi_\gamma(\cdot|x, a)\ ||\ \logittoprob(\phi)(\cdot|x,a)) \, .    
\end{align*}
The full support condition ensures that $L(\phi) = 0$ implies that $\logittoprob(\phi) = \mu^\pi_\gamma$. Our goal is to show that $L(\phi_k) \rightarrow 0$ almost surely, hence $\sum_{x, a} \xi(x, a) \text{KL}(\mu^\pi_\gamma(\cdot|x, a)\ ||\ \logittoprob(\phi_k)(\cdot|x,a)) \rightarrow 0$, and so $\mu_k \rightarrow \mu^\pi_\gamma$, as required.

\textbf{A supermartingale argument.}  We start by considering a second-order Taylor expansion (with Lagrange remainder) of $L(\phi_{k+1})$ around $\phi_k$ (here, and in the remainder of the proof, it is useful to interpret a probability distribution in $\mathscr{P}(\mathcal{X})$ as a vector in $\mathbb{R}^{\mathcal{X}}$ --- specifically, an element of the simplex $\Delta(\mathcal{X})$, which we will do without further remark):
\begin{align*}
    L(\phi_{k+1}) = L(\phi_k + \varepsilon_k ( \hat{T}^\pi \mu_k - \mu_k)) = L(\phi_k) + \varepsilon_k \langle \nabla_\phi L(\phi_k), \hat{T}^\pi \mu_k - \mu_k \rangle + \varepsilon^2_k \nabla^2_\phi L(\tilde{\phi_k})[\hat{T}^\pi \mu_k - \mu_k, \hat{T}^\pi \mu_k - \mu_k] \, ,
\end{align*}
for some $\tilde{\phi}_{k}$ on the line segment $[\phi_k, \phi_{k+1}]$.
Defining $\mathscr{F}_k$ to be the sigma-algebra generated by all random variables up to, but not including, those defining the update from $\phi_k$ to $\phi_{k+1}$, we have
\begin{align*}
    \mathbb{E}[L(\phi_{k+1}) \mid \mathscr{F}_k] = L(\phi_k) + \varepsilon_k \mathbb{E}[\langle \nabla_\phi L(\phi_k), \hat{T}^\pi \mu_k - \mu_k \rangle \mid \mathscr{F}_k] + \varepsilon^2_k \mathbb{E}[ \nabla^2_\phi L(\tilde{\phi_k})[\hat{T}^\pi \mu_k - \mu_k, \hat{T}^\pi \mu_k - \mu_k] \mid \mathscr{F}_k] \, .
\end{align*}
From the form of the gradient $\nabla_\phi L(\phi)$, the Hessian $\nabla^2_\phi L(\phi)$ is readily seen to be bounded, and the inputs above $\hat{T}^\pi \mu_k - \mu_k$ are also bounded, meaning there is a constant $K > 0$ such that
\begin{align*}
    \mathbb{E}[L(\phi_{k+1}) \mid \mathscr{F}_k] \leq L(\phi_k) + \varepsilon_k \mathbb{E}[\langle \nabla_\phi L(\phi_k), \hat{T}^\pi \mu_k - \mu_k \rangle \mid \mathscr{F}_k] + \varepsilon^2_k K \, .
\end{align*}
To deal with the first-order term, we note that a straightforward calculation gives
\begin{align*}
     [ \nabla_{\phi} L(\phi) ](x'|x, a) = \xi(x, a)  (\logittoprob(\phi)(x'|x, a) - \mu^\pi_\gamma(x'|x, a)) \, .
\end{align*}
We hence have
\begin{align*}
    \mathbb{E}[\langle \nabla_\phi L(\phi_k), \hat{T}^\pi \mu_k - \mu_k \rangle \mid \mathscr{F}_k] = \langle \mu_k - \mu^\pi_\gamma, T^\pi \mu_k - \mu_k  \rangle_\xi \, .
\end{align*}
Now we use a contractivity argument to bound this derivative. We first argue that $T^\pi$ as defined above is a $\gamma$-contraction under the norm $\|\cdot\|_\xi$ defined by $\|\mu \|^2_\xi = \sum_{x, a, x'} \xi(x,a) \mu(x'|x, a)^2$. To see this, note
\begin{align*}
    \| T^\pi \mu - T^\pi \mu' \|^2_\xi = & \| \gamma P^\pi \mu - \gamma P^\pi \mu' \|^2_\xi \\
    = & \gamma^2 \|P^\pi \mu - P^\pi \mu \|^2_\xi \\
    = & \gamma^2 \sum_{x, a, x'} \xi(x, a) \left(\sum_{x'', a''} P(x''|x, a)\pi(a''|x'') (\mu(x'|x'', a'') - \mu'(x'|x'', a''))\right)^2 \\
    \overset{(a)}{\leq} & \gamma^2 \sum_{x, a, x'} \xi(x, a) \sum_{x'', a''} P(x''|x, a)\pi(a''|x'') ((\mu(x'|x'', a'') - \mu'(x'|x'', a'')))^2 \\
    \overset{(b)} = & \gamma^2 \sum_{x, a, x'} \xi(x, a) (\mu(x'|x, a) - \mu'(x'|x, a))^2 \\
    = & \gamma^2 \|\mu - \mu' \|_\xi^2 \, ,
\end{align*}
as required, with (a) following from Jensen's inequality, and (b) from $\xi$ being stationary.

Using this contraction result, we have:
\begin{align*}
    & \| \mu^\pi_\gamma - T^\pi \mu_k \|_\xi^2 \leq \gamma^2 \| \mu^\pi_\gamma - \mu_k \|_\xi^2 \\
    \implies & \| \mu^\pi_\gamma - \mu_k + \mu_k - T^\pi \mu_k \|_\xi^2 \leq \gamma^2 \| \mu^\pi - \mu_k \|_\xi^2 \\
    \implies & \| \mu^\pi_\gamma - \mu_k \|_\xi^2 + \|\mu_k - T^\pi \mu_k \|_\xi^2 + 2\langle \mu^\pi_\gamma - \mu_k, \mu_k - T^\pi \mu_k \rangle_\xi \leq \gamma^2 \| \mu^\pi_\gamma - \mu_k \|_\xi^2 \\
    \implies & \langle \mu_k - \mu^\pi_\gamma, T^\pi \mu_k - \mu_k \rangle_\xi \leq \frac{1}{2} \left( (\gamma^2 - 1) \|\mu^\pi_\gamma - \mu_k\|_\xi^2 - \| \mu_k - T^\pi \mu_k \|_\xi^2  \right) \leq -\frac{1-\gamma^2}{2} \|\mu^\pi_\gamma - \mu_k \|_\xi^2 \, .
\end{align*}

Returning to the Lyapunov function, we therefore have
\begin{align*}
    \mathbb{E}[L(\phi_{k+1}) \mid \mathscr{F}_k] \leq L(\phi_k) - \varepsilon_k\frac{1-\gamma^2}{2} \|\mu^\pi_\gamma - \mu_k\|_\xi^2  + \varepsilon^2_k K \, .
\end{align*}
We now follow the ideas of the Robbins-Siegmund theorem \citep{robbins1971convergence}. Based on the above inequality, $(L(\phi_{k}))_{k \geq 0}$ is almost a positive supermartingale, save for the additive $\varepsilon^2_k K$ terms in the upper bounds on the conditional expectation. However, defining $\tilde{L}_k = L(\phi_k) - \sum_{l=0}^{k-1} \varepsilon_l^2 K + \sum_{l=0}^{k-1} \varepsilon_l \tfrac{1-\gamma^2}{2}\|\mu^\pi_\gamma - \mu_l\|^2_\xi$, we have
\begin{align*}
    \mathbb{E}[\tilde{L}_{k+1} \mid \mathscr{F}_k]
    & \leq \mathbb{E} \Big\lbrack L(\phi_{k+1}) - \sum_{l=0}^k \varepsilon_l^2 K + \sum_{l=0}^{k} \varepsilon_l \frac{1-\gamma^2}{2}\|\mu^\pi_\gamma - \mu_l\|^2_\xi \ \mid \mathscr{F}_k \Big \rbrack \\
    &\leq L(\phi_k) - \sum_{l=0}^{k-1} \varepsilon_l^2 K + \sum_{l=0}^{k-1} \varepsilon_l \frac{1-\gamma^2}{2}\|\mu^\pi_\gamma - \mu_l\|^2_\xi \\
    & \leq \tilde{L}_k \, .
\end{align*}
Hence $(\tilde{L}_k)_{k \geq 0}$ is a supermartingale. However, the subtraction of the $\varepsilon_k^2 K$ terms means that it is not a non-negative supermartingale, so we cannot immediately apply the supermartingale convergence theorem. The approach of \citet{robbins1971convergence} is to define a sequence of stopping times $\tau_\ell = \inf \{ k \geq 0 : \tilde{L}_k \leq - \ell  \}$, for $\ell \in \mathbb{N}$. By the optional stopping theorem, $(\tilde{L}_{k \wedge \tau_\ell})_{k=0}^\infty$ is a supermartingale bounded below, and hence by the supermartingale convergence theorem, converges almost surely. By the second Robbins-Monro step size condition, $\tau_\ell = \infty $ eventually almost surely, and hence $\tilde{L}_k$ converges almost surely, leading to almost-sure convergence of $L(\phi_k)$ too, as well as $\sum_{l=0}^{k} \varepsilon_l \frac{1-\gamma^2}{2}\|\mu^\pi_\gamma - \mu_l\|^2_\xi$. Due to the first Robbins-Monro step size condition $\sum_{k=0}^\infty \varepsilon_k = \infty$, we must have $\|\mu^\pi_\gamma - \mu_k\|^2_\xi \rightarrow 0$, which completes the proof of the theorem in the case where $\xi$ has full support.

\textbf{A chaining argument for invariant distributions without full support.} The previous argument relied on the existence of an invariant distribution $\xi$ for the Markov chain over state-action pairs generated by the interaction of the policy $\pi$ with the MDP in question. We now explain how to generalise this proof technique to remove this restriction on $\xi$.

First, by appending an artificial self-transitioning terminal state if required, there always exists an invariant distribution $\xi$ for the Markov chain concerned, even in episodic settings where trajectories terminate in finite time. The argument above may be applied as-is to obtain the same conclusion $\sum_{l=0}^{k} \varepsilon_l \frac{1-\gamma^2}{2}\|\mu^\pi_\gamma - \mu_l\|^2_\xi < \infty$, and hence $\|\mu^\pi_\gamma - \mu_k\|^2_\xi \rightarrow 0$. The difference now is that this only shows convergence of $\mu_k$ to $\mu^\pi_\gamma$ along the state-action pairs with support under $\xi$.

We begin by recalling some notions from the theory of discrete-time Markov chains on finite sets; see \citet{norris1998markov} for further background. We also clarify that in Markov chain theory, the term \emph{state space} is typically used to refer to the set of states which a Markov chain can take on. For our Markov chain, this \emph{state space} is $\mathcal{X} \times \mathcal{A}$, \emph{not} the usual state space of the MDP. To avoid confusion, we will use the term Markov chain state space (or MCSS) to distinguish the state space of the Markov chain from the set $\mathcal{X}$, and the term Markov chain state to refer to an element of the MCSS.

We can partition the MCSS $\mathcal{X} \times \mathcal{A}$ into \emph{communicating classes}. A communicating class $C \subseteq \mathcal{X} \times \mathcal{A}$ is a set of Markov chain states such that for all $(x, a), (y, b) \in C$, there exists $t > 0$ such that $\mathbb{P}((X_t, A_t) = (x, a) \mid (X_0, A_0) = (y, b)) > 0$ and $\mathbb{P}((X_t, A_t) = (y, b) \mid (X_0, A_0) = (x, a)) > 0$, and further for any $(x, a) \in C$, no Markov chain state outside $C$ has this property. The set of communicating classes of the Markov chain can be given a directed acyclic graph structure, by adding an edge from one class $C$ to a distinct class $C'$ if there exist $(x, a) \in C$, $(y, b) \in C'$ with $\mathbb{P}((X_1, A_1) = (y, b) \mid (X_0,A_0) = (x, a)) > 0$. Let us refer to this directed acyclic graph as $G$. Without loss of generality to what follows, we may assume $G$ is connected (the argument may be applied to each connected component of $G$ separately if $G$ is not connected).

The goal is to recurse backwards through the directed acyclic graph $G$, establishing first for the Markov chain states $(x, a)$ in communicating classes in the leaves of the graph that $\mu_k(\cdot|x, a) \rightarrow \mu^\pi_\gamma(\cdot|x, a)$, and then inductively moving back through the graph. Note that the leaves of $G$ are precisely the \emph{recurrent} communicating classes of the Markov chain: those classes $C$ for which there exists an invariant distribution $\xi_C$ for the Markov chain supported precisely on $C$. The argument above establishes that $\mu_k(\cdot|x, a) \rightarrow \mu^\pi_\gamma(\cdot|x, a)$ for all $(x, a) \in C$, and in fact the stronger conclusion that $\sum_{l=0}^\infty \varepsilon_l \| \mu^\pi_\gamma - \mu_l \|^2_{\xi_C} < \infty$.

Now, for the inductive step of the argument, let $C$ be a non-recurrent communicating class of the Markov chain, and suppose that for every descendant $C'$ of $C$ in the directed acyclic graph $G$, we have established that for some distribution $\xi_{C'}$ supported on $C'$, we have $\sum_{l=0}^\infty \varepsilon_l \| \mu^\pi_\gamma - \mu_l \|^2_{\xi_{C'}} \rightarrow 0$. We now aim to construct a distribution $\xi_C$ supported on $C$, and to demonstrate that $\sum_{l=0}^\infty \varepsilon_l \| \mu^\pi_\gamma - \mu_l \|^2_{\xi_{C}} \rightarrow 0$, so that by induction the theorem is proven.

To do this, we appeal to the Perron-Frobenius theorem \citep{perron1907theorie,frobenius1912ueber}; see \citet{seneta2006non} for a recent account. Specifically, we consider the transition matrix of the Markov chain in question, and consider the sub-matrix obtained by deleting all rows and columns corresponding to Markov chain states outside $C$. The resulting matrix is strictly sub-stochastic (all elements are non-negative, rows sums are less than or equal to 1, with at least one row having row sum strictly less than 1), and hence by the Perron-Frobenius theorem, there exists a left-eigenvector $v \in \mathbb{R}^{C}$ for this matrix with eigenvalue $0 \leq \lambda < 1$, and all elements positive; we may further scale $v$ so that the elements sum to 1. We now set $\xi_C$ to be the distribution over the Markov chain state space that is equal to $v$ on $C$, and 0 elsewhere. We now show that $T^\pi$ still behaves `almost' like a contraction under $\|\cdot\|_{\xi_C}$, which will allow us to re-use the supermartingale argument above. First note that from the structure of the communicating classes, we have that $\xi_C P^\pi$ is equal to $\lambda \xi_C$ on $C$, some other non-negative vector $\bar{\xi}_C$ on $\bar{C}$ the union of descendant communicating classes from $C$, and 0 elsewhere. Now, note that for $\mu, \mu' \in \Delta(\mathcal{X})^{\mathcal{X} \times \mathcal{A}}$, we have
\begin{align*}
    & \| T^\pi \mu - T^\pi \mu' \|^2_{\xi_C} \\
    = & \gamma^2 \|P^\pi \mu - P^\pi \mu \|^2_{\xi_C} \\
    = & \gamma^2 \sum_{x, a, x'} \xi_C(x, a) \left(\sum_{x'', a''} P(x''|x, a)\pi(a''|x'') (\mu(x'|x'', a'') - \mu'(x'|x'', a''))\right)^2 \\
    \leq & \gamma^2 \sum_{x, a, x'} \xi_C(x, a) \sum_{x'', a''} P(x''|x, a)\pi(a''|x'') ((\mu(x'|x'', a'') - \mu'(x'|x'', a'')))^2 \\
    = & \gamma^2 \left( \sum_{(x, a) \in C} \lambda \xi_C(x, a) \sum_{x'} (\mu(x'|x, a) - \mu'(x'|x, a))^2 + \sum_{(x, a) \in \bar{C}} \bar{\xi}_{C}(x, a) \sum_{x'}  (\mu(x'|x, a) - \mu'(x'|x, a))^2 \right) \\
    = & \gamma^2 \lambda \|\mu - \mu' \|_{\xi_C}^2 + \gamma^2 \|\mu - \mu' \|_{\bar{\xi}_{C}}^2  \, .
\end{align*}
The intuition here is that if $\mu \approx \mu'$ on $\bar{C}$, then we have a contraction-like bound for $T^\pi$ as measured by $\xi_C$. From this, we obtain the bound
\begin{align*}
    \langle \mu_k - \mu^\pi_\gamma, T^\pi \mu_k - \mu_k \rangle_{\xi_C} \leq -\frac{1-\gamma^2\lambda}{2} \|\mu^\pi_\gamma - \mu_k \|_{\xi_C}^2 + \frac{\gamma^2}{2} \|\mu_k - \mu^\pi_\gamma \|_{\bar{\xi}_C}^2 \, .
\end{align*}
Defining an alternative Lyapunov function by
\begin{align*}
    L_{\xi_C}(\phi) = \sum_{x,a} \xi_C(x, a) \text{KL}(\mu^\pi_\gamma(\cdot|x, a)\ ||\ \logittoprob(\phi)(\cdot|x,a)) \, ,
\end{align*}
a similar calculation to the above gives
\begin{align*}
    \mathbb{E}\lbrack L_{\xi_C}(\phi_{k+1}) \mid \mathscr{F}_k \rbrack \leq L_{\xi_C}(\phi_k) - \varepsilon_k \frac{1-\gamma^2\lambda}{2} \|\mu^\pi_\gamma - \mu_k\|_{\xi_C}^2 + \varepsilon_k \frac{\gamma^2}{2} \|\mu_k - \mu^\pi_\gamma \|_{\bar{\xi}_C}^2 + \varepsilon^2_k K \, .
\end{align*}
The inductive hypothesis leads to $\sum_{l=0}^\infty\varepsilon_l \nicefrac{\gamma^2}{2}  \|\mu_l - \mu^\pi_\gamma \|_{\bar{\xi}_C}^2 < \infty$, and so defining the modified sequence
\begin{align*}
    \tilde{L}_k^{\xi_C} = L_{\xi_C}(\phi_k) - \sum_{l=0}^{k-1} \left(\varepsilon_l^2 K + \varepsilon_l \frac{\gamma^2}{2} \|\mu_l - \mu^\pi_\gamma \|_{\bar{\xi}_C}^2 \right) + \sum_{l=0}^{k-1} \varepsilon_l \frac{1-\gamma^2\lambda}{2}\|\mu^\pi_\gamma - \mu_l\|^2_{\xi_C} \, ,
\end{align*}
the same Robbins-Siegmund argument yields that $(\tilde{L}_k^{\xi_C})_{k \geq 0}$ is a convergent supermartingale, and hence $\sum_{l=0}^{\infty} \varepsilon_l \tfrac{1-\gamma^2\lambda}{2}\|\mu^\pi_\gamma - \mu_l\|^2_{\xi_C} < \infty$, as required to complete the induction, and hence the proof. \qed

\subsection{Examples of cross-entropy TD learning}

\begin{figure*}
    \centering
    \includegraphics[keepaspectratio,width=.9\textwidth]{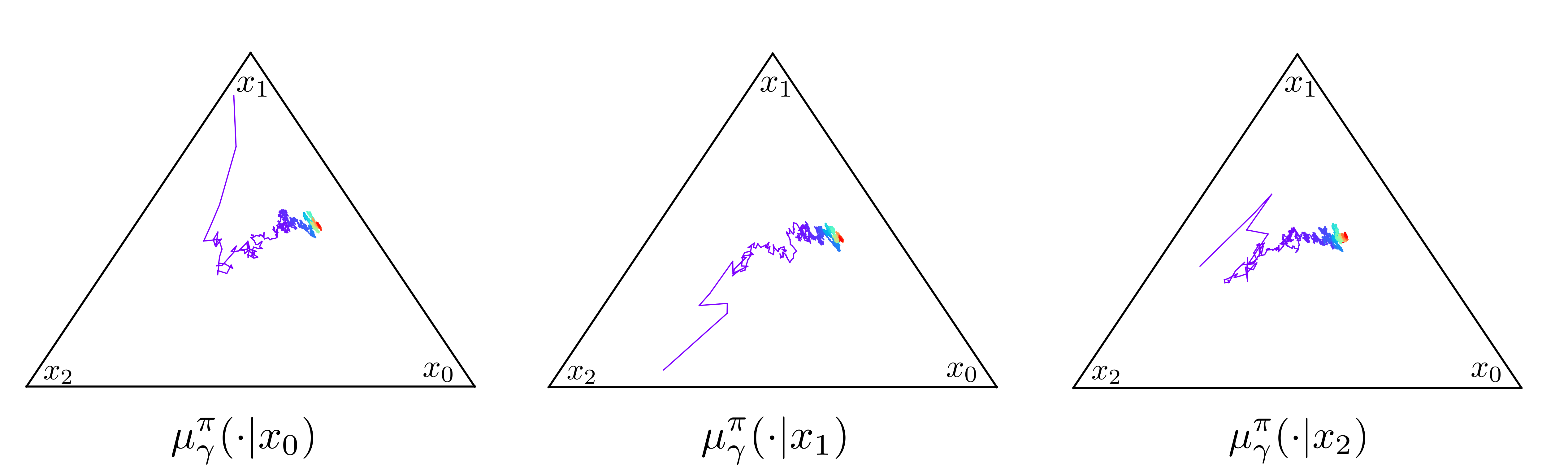}
    \caption{Illustration of CETD dynamics in a three-state MDP. The red dots indicate the fixed point of the operator $T^\pi$ (the true visitation distributions). The coloured lines indicate the path taken by the CETD algorithm.
    }
    \label{fig:cetd-dynamics-main}
\end{figure*}

Figure~\ref{fig:cetd-dynamics-main} shows an example visualisation of the synchronous CETD algorithm in the case of a randomly-generated three-state, one-action MDP. The transition matrix and initial distributions $\mu_0$ used to generated these plots are
\begin{align*}
    P = \begin{pmatrix}
    0.297492728 & 0.702444212 & 0.000063060 \\
    0.584810131 & 0.257810252 & 0.157379617 \\
    0.181511854 & 0.373368720 & 0.445119427
    \end{pmatrix}
    \, , \quad
    \phi_0 =
    \begin{pmatrix}
        -2.3634686 & 1.13534535 & -1.01701414 \\
        0.63736181 & -0.85990661 & 1.77260763 \\
        -1.11036305 & 0.18121427 & 0.56434487
    \end{pmatrix} \, ,
\end{align*}
where as the MDP has a single action, we specify $P$ as a state-by-state transition matrix, and similarly $\phi_0$ is presented a state-by-state matrix, with each row corresponding to the logits of a single estimated future state-visitation distribution. Further, we take $\gamma = 0.9$, and the learning rate schedule used was $\varepsilon_k = 0.75 (k + 1)^{-0.6}$. In all plots presented in this section, we subsample the trajectories generated by a factor of 10 to make trajectories easier to visually inspect.

We also provide a further illustration of CETD below, in the case where the target $\mu^\pi_\gamma$ lies on the boundary of $\Delta(\mathcal{X})^{\mathcal{X} \times \mathcal{A}}$, by modifying the transition matrix $P$ above to have a transient first state. Specifically, we set
\begin{align*}
    P =
    \begin{pmatrix}
        0.765830909 & 0.234148071 & 0.000021020 \\
        0 & 0.620945430 & 0.379054570 \\
        0 & 0.456168756 & 0.543831244
    \end{pmatrix} \, ,
\end{align*}
and use the same initialisation for $\phi_0$ as described above.
The results are shown in Figure~\ref{fig:cetd2}. One interesting observation is that when the collection of true state visitation distributions lies on the boundary of the product of simplices, the convergence of the algorithm appears to be particularly slow. An intuition as to why this might be the case is that the sample-based CETD update is limited to decreasing logit values only by the magnitude of the current probability corresponding to the logit. Because of this, fitting zero (or near-zero) probabilities requires many gradient updates. This hints at the utility of further work to develop a finer-grained understanding of the asymptotic performance of this algorithm (such as asymptotic covariance and/or convergence rate), as well as approaches for variance reduction that may improve the convergence rate, either practically or empirically.

\begin{figure}[ht]
    \centering
    \includegraphics[keepaspectratio,width=.9\textwidth]{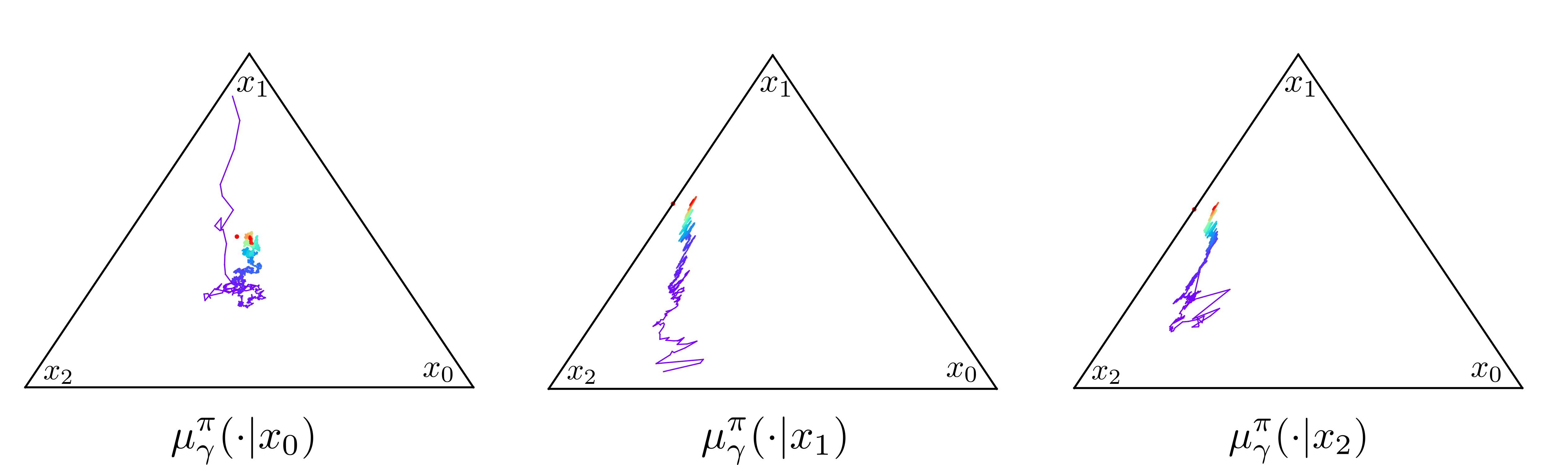}
    \caption{CETD dynamics for an MDP with $\mu^\pi_\gamma$ lying on the boundary of $\Delta(\mathcal{X})^{\mathcal{X} \times \mathcal{A}}$.}
    \label{fig:cetd2}
\end{figure}

\section{Further details and examples relating to core algorithms}\label{sec:appendix:examples}

\subsection{\ggpi counterexamples}
\label{sec:gnmpi-counterexamples}

We address several questions about the necessity of conditions for results appearing in the main paper through a set of counterexamples. Specifically, these questions and their resolutions are:
\begin{itemize}
    \item Is the Q-function of a non-Markov policy always equal to that of a Markov policy? No: See Example~\ref{example:non-markov-q}.
    \item Can we do GPI with the Q-functions of any collection of non-Markov policies? No: See Example~\ref{example:gpi-non-markov}.
    \item If we restrict to \GSPs, do we need the closure condition? Yes: See Example~\ref{example:gpi-non-markov-switching}.
\end{itemize}

\begin{example}\label{example:non-markov-q}
    Consider the two-state, two-action MDP in Figure~\ref{fig:gpi-example1}, and consider a non-Markov policy $\pi$ of the following form. When initialised in the left state, the agent seeks to take the action sequence $bb$, leading it to transition immediately to the right state, and then to terminate in the right state, having attained $0$ reward. When initialised in the right state, the agent seeks to take the action sequence $aab$, attaining a reward of $2$. In order for $Q^\pi(\text{L}, b)$ to be equal to $Q^{\pi'}(\text{L}, b)$ for some Markov policy $\pi'$, it must be the case that $\pi'$ takes action $b$ in state $\text{R}$ with probability 1. However, this is incompatible with the requirement $Q^{\pi}(\text{R}, a) = Q^{\pi'}(\text{R}, a)$, since we would require $\pi'(a|\text{R}) > 0$.
\end{example}

\begin{example}\label{example:gpi-non-markov}
    As a very basic example of why non-Markov policies cannot in general be used within GPI, consider the one-state MDP in Figure~\ref{fig:gpi-example2}. Consider a non-Markov policy $\pi$ specified as follows:
    \begin{itemize}
        \item Initially, the policy randomises uniformly between actions $a$ and $b$.
        \item If action $a$ is selected at the first time-step, then the full sequence of actions is deterministically specified as $abbbb\ldots$.
        \item If action $b$ is selected at the first time-step, then the full sequence of actions is deterministically specified as $baaaa\ldots$.
    \end{itemize}
    We therefore have $Q^\pi(a) = 1$, and $Q^\pi(b) = \gamma/(1-\gamma)$. So if $\gamma > \frac{1}{2}$, the greedy Markov policy obtained prefers action $b$, which is clearly worse for performance than the non-Markov policy $\pi$, meaning the GPI guarantee does not hold in this case.
\end{example}

\begin{example}\label{example:gpi-non-markov-switching}
    Consider the MDP with a depth-3 binary tree transition structure displayed in Figure~\ref{fig:gpi-example3}. Consider two policies $\pi_\text{L}$ and $\pi_\text{R}$ which always take the `left' and `right' actions in the tree. We consider the \GSP $\nu$ that follows $\pi_\text{L}$ for two steps and then switches to $\pi_\text{R}$ (this is a \GSP with two switches, and probability of switching equal to 1 in both cases). The value of this policy at the root node (in the undiscounted case) is +1 (obtained at the red leaf node), since the sequence of actions taken from the root node $x_0$ is LLR.
    
    We now consider the Markov policy obtained by acting greedily with respect to $Q^\nu$. To begin with, consider the Q-values $Q^\nu(x_0, \text{L})$ and $Q^\nu(x_0, \text{R})$. These are +1 and +2 respectively (obtained from the red and green leaf nodes), since these correspond to sequences of actions LLR and RLR respectively. So the greedy policy with respect to this Q-function takes the `right' action at the base state $x_0$. Next, at state $x_1$, we have $Q^\nu(x_1, \text{L}) = -1$ and $Q^\nu(x_1, \text{R}) = 0$ (obtained at the blue and grey nodes), since these correspond to the action sequences LL and RL from $x_1$, meaning that the greedy policy takes the `right' action in state $x_2$. This is enough to deduce that the greedy policy, executed from $x_0$, attains a return of 0, in contrast to the return of $+1$ obtained by the initial \GSP $\nu$; the greedy policy performs worse than the initial policy.
    
    To see how the closure condition deals with this, note that the condition would require that we include the value functions for (i) the \GSP that executes $\pi_\text{L}$ for one step and then switches to $\pi_\text{R}$ and (ii) $\pi_\text{R}$ itself, in the GPI procedure.
    Performing \ggpi over this collection of three policies then leads to an improved policy which when executed from $x_0$, obtains a return of +2, improving over all initial policies considered.
\end{example}

\begin{figure}[ht]
    \centering
    \includegraphics[keepaspectratio,width=.3\textwidth]{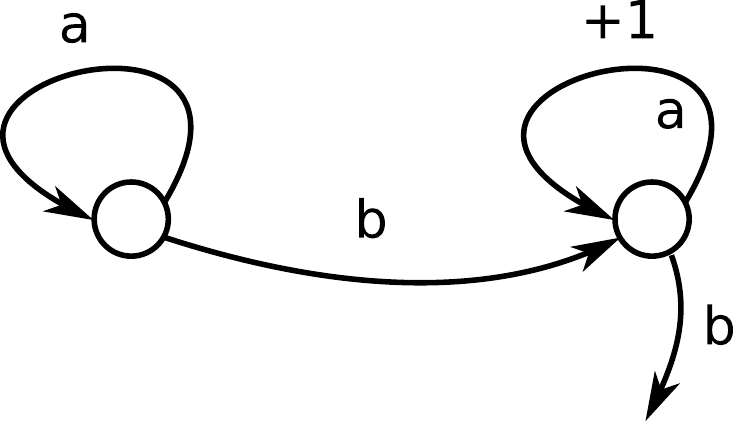}
    \caption{Example MDP showing that the Q-function of a non-Markov policy cannot necessarily be written as the Q-function of a Markov policy.}
    \label{fig:gpi-example1}
\end{figure}

\begin{figure}[ht]
    \centering
    \includegraphics[keepaspectratio,width=.15\textwidth]{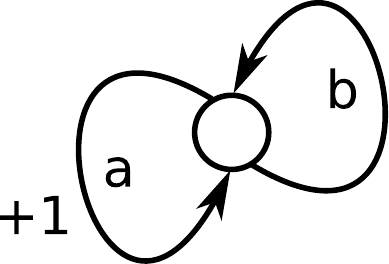}
    \caption{Example MDP illustrating that acting greedily with respect to a non-Markov policy may lead to detrimental performance.}
    \label{fig:gpi-example2}
\end{figure}

\begin{figure}[ht]
    \centering
    \includegraphics[keepaspectratio,width=.45\textwidth]{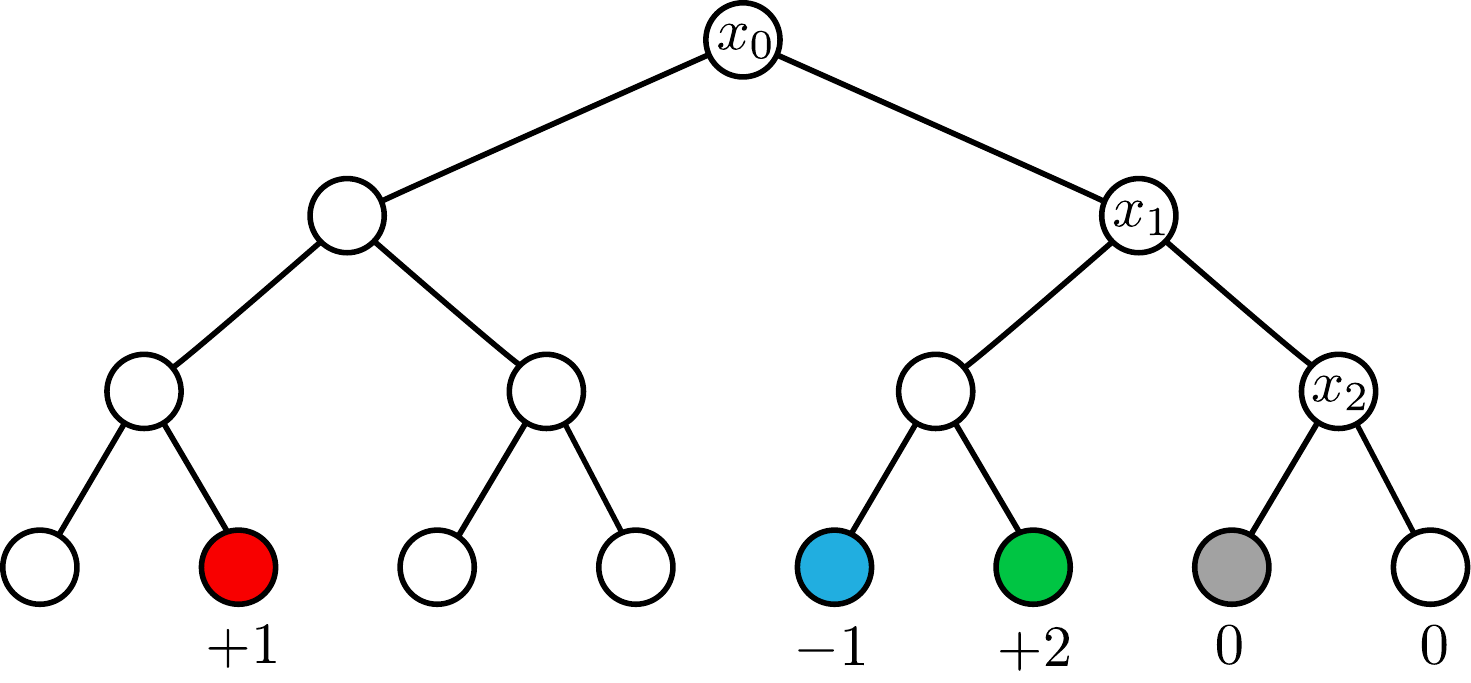}
    \caption{Example MDP illustrating that acting greedily with respect to a \GSP may lead to detrimental performance.}
    \label{fig:gpi-example3}
\end{figure}

\subsection{Further examples of \ggpi for policy iteration}
\label{sec:appendix:policy_iteration_details}

We first give a full description in Algorithm~\ref{alg:policy_iteration} of the combination of \ggpi with policy iteration that we consider in the paper. A possible optimisation is for each step to consider only those \GSPs that end in the most recent policy  $\pi_k$, as this dominates any \GSP comprised of the same initial policies and ending in any other prior policy.

The example that follows then gives a granular sense of when \ggpi delivers benefits over standard policy iteration in the tabular setting, to complement the four-rooms experiments in the main paper.

\begin{algorithm}
    \caption{\ggpi for sample-based policy iteration}
    \label{alg:policy_iteration}
    \begin{algorithmic}
        \REQUIRE Number of iterations $n_{\text{iter}}$, sample budget $n_{\text{samples}}$ 
        \STATE $i \leftarrow 0$
        \STATE $\pi' \leftarrow$ arbitrary policy
        \STATE $\Pi \leftarrow \emptyset$ \COMMENT{Set of policies seen so far}
        \STATE $\mathcal{M} \leftarrow \emptyset$ \COMMENT{Set of GHMs associated with policies in $\Pi$}
        \REPEAT
        \STATE $\pi \leftarrow \pi'$
        \STATE $\Pi \leftarrow \Pi \cup \{ \pi \}$ 
        \STATE Learn GHMs $\mu^{\pi}_\gamma$ and $\mu^{\pi}_\beta$, with $\beta = \gamma(1-\switchprob)$.
        \STATE $\mathcal{M} \leftarrow \mathcal{M} \cup \{\mu^{\pi}_\gamma, \mu^{\pi}_\beta \}$
        \STATE Define a set $\mathcal{V}$ of \GSPs using base policies in $\Pi$ and switching probability $\switchprob$. 
        \FOR{each state $x \in \mathcal{X}$}
                    \STATE For each $\nu \in \mathcal{V}$, estimate $Q^\nu_\gamma(x, \cdot)$ by composing $n_{\text{samples}}$ GHM samples via Equation~\eqref{eq:smp-estimator}.
            \STATE $\pi'(x) \leftarrow \mathcal{G}\left(\{ Q^\nu_\gamma(x, \cdot) | \nu \in \mathcal{V}\}\right)$
        \ENDFOR
        \STATE $i \leftarrow i + 1$
        \UNTIL $\pi' = \pi$ or $i \ge n_{\text{iter}}$
    \end{algorithmic}
\end{algorithm}

\begin{example}
    Figure~\ref{fig:ggpi-ex-chain} illustrates a chain environment with a large reward at one end of the chain, and `distractor' rewards along the chain that may cause a myopic agent to prefer a sub-optimal action. An initial policy $\pi_0$ that is able to make some progress towards the optimal side of the chain will have this progress `wiped out' by myopic greedy improvements in standard policy iteration. In contrast, with \ggpi, this initial progress can be used to deliver a stronger improvement over the first greedy policy $\pi_1$, leading to an optimal policy in fewer iterations. The figure illustrates an extreme setting where standard policy iteration would need $k+1$ improvement steps to reach the optimal policy, while \ggpi reaches it in two single improvement steps. Note also that the set $\{\pi_0 \overset{\alpha}{\rightarrow} \pi_1, \pi_1\}$ is suffix-closed in the sense of Definition~\ref{def:suffix-closed}, so improvement is guaranteed by Theorem~\ref{prop:gnmpi}.
\end{example}

\begin{figure}
    \centering
    \includegraphics[keepaspectratio,width=.5\textwidth]{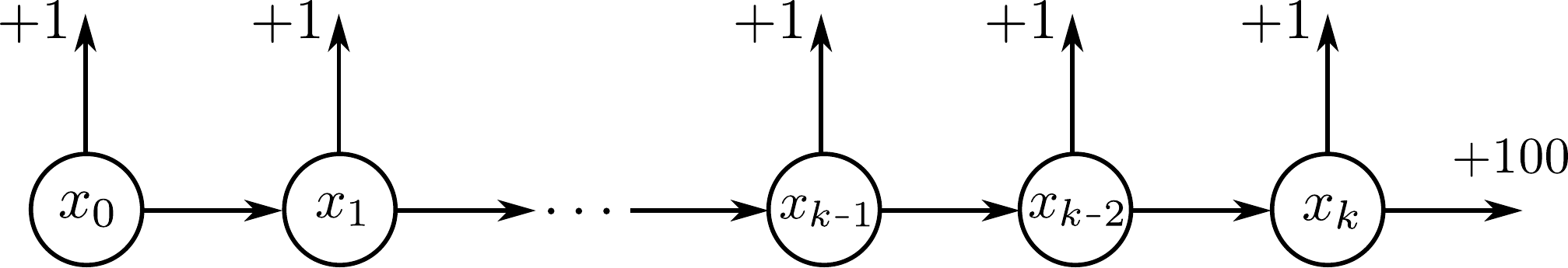}
    
    \vspace{0.5cm}
    \begin{tabular}{c|cccccc}
    \toprule
    Policy & $x_0$ & $x_1$ & $\cdots$ & $x_{k\indexminus2}$ & $x_{k\indexminus1}$ & $x_k$ \\ \midrule
    $\pi_0$  & \cellcolor{lightblue} $\rightarrow$ & \cellcolor{lightblue} $\rightarrow$ & \cellcolor{lightblue} $\cdots$ & \cellcolor{lightblue} $\rightarrow$ & \cellcolor{lightblue} $\rightarrow$ & $\uparrow$ \\
    $\pi_1 = \mathcal{G}(\pi_0)$  & $\uparrow$ & $\uparrow$ & $\cdots$ & $\uparrow$ & $\uparrow$ & \cellcolor{lightblue} $\rightarrow$ \\ 
    $\mathcal{G}(\pi_1)$ & $\uparrow$ & $\uparrow$ & $\cdots$ & $\uparrow$ &  \cellcolor{lightblue} $\rightarrow$ & \cellcolor{lightblue} $\rightarrow$ \\
    $\mathcal{G}(\{\pi_0, \pi_1\})$ & $\uparrow$ & $\uparrow$ & $\cdots$ & $\uparrow$ & \cellcolor{lightblue} $\rightarrow$ &\cellcolor{lightblue} $\rightarrow$ \\
    \cellcolor{lightblue}$\mathcal{G}(\{\pi_0 \overset{\switchprob}{\rightarrow} \pi_1, \pi_1\})$\!\! & \cellcolor{lightblue} $\rightarrow$ & \cellcolor{lightblue} $\rightarrow$ & \cellcolor{lightblue} $\cdots$ & \cellcolor{lightblue} $\rightarrow$ & \cellcolor{lightblue} $\rightarrow$ & \cellcolor{lightblue} $\rightarrow$ \\
    \bottomrule
    \end{tabular}
    
    \caption{Example chain environment, initial policy $\pi_0$, and policies generated by a variety of policy improvement steps. Light-blue shaded cells indicate optimal actions/policies. In this case, since $\pi_0$ encodes optimal behaviour in some states, it is useful to include switching policies beginning with $\pi_0$ in the policy improvement step, and $\mathcal{G}(\{\pi_0 \overset{\switchprob}{\rightarrow} \pi_1, \pi_1\})$ is indeed optimal in this case.}
    \label{fig:ggpi-ex-chain}
\end{figure}

\subsection{Full algorithmic description of \ggpi for transfer}
\label{sec:appendix:transfer-alg}
In Algorithm~\ref{alg:transfer}, we give algorithmic pseudocode to describe the use of \GSP evaluation with GHMs and \ggpi for transfer. There are several steps to the process: a set of Markov policies is given, which may be obtained through learning about prior reward signals, exploration objectives, imitation learning, etc. The agent learns GHMs for these models in a reward-free manner. A novel reward function is then revealed, and \ggpi can be used to derive an improved policy in a zero-shot manner.

\begin{algorithm}
    \caption{\ggpi for sample-based transfer.}
    \label{alg:transfer}
    \begin{algorithmic}
        \REQUIRE Markov policies $\pi_1,\ldots,\pi_k$, switching probability $\switchprob \in (0, 1]$, sampling budget $n_{\text{samples}}$, discount $\gamma$.
        \STATE Learn GHMs $\mu^{\pi_i}_\gamma$ and $\mu^{\pi_i}_\beta$, with $\beta = \gamma(1-\switchprob)$.
        \STATE Novel reward function $r$ is revealed
        \STATE Select a suffix-closed set $\Pi$ of \GSPs
        \FOR{each state $x$ encountered}
            \STATE For each $\nu \in \Pi$, estimate $Q^\nu_\gamma(x, \cdot)$ by composing $n_{\text{samples}}$ GHM samples via Equation~\eqref{eq:smp-estimator}.
            \STATE Act greedily according to the output of $\mathcal{G}$ applied to these estimated Q-functions.
        \ENDFOR
    \end{algorithmic}
\end{algorithm}

\subsection{Geometric horizon models with $\beta = 0$ and $\beta > 0$}
\label{sec:appendix:ghm_b0_illustration}

\begin{figure}
    \centering
    \includegraphics[keepaspectratio,width=0.9\textwidth]{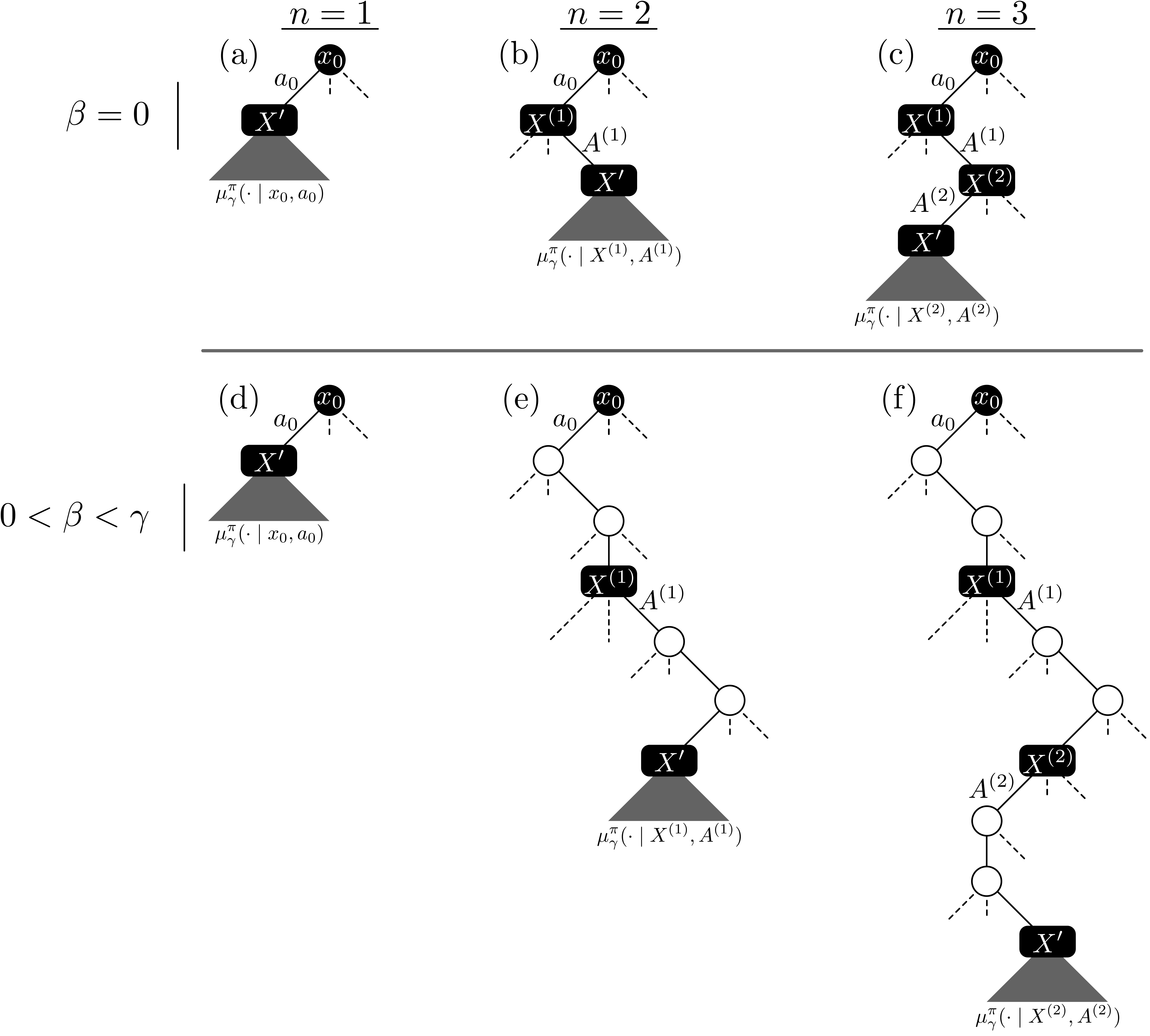}
    \caption{Illustrating samples for GHMs with (a)-(c) $\beta = 0$ and (d)-(f) $\beta > 0$, with $n \in \{1, 2, 3\}$ respectively. Starting state-action given by ($x_0, a_0$) shown as the root node and first branch. States that are sampled or directly observed are shown in black and actions explicitly conditioned on are denoted with solid line branches. Meanwhile, states shown in white are not sampled, but would have been produced while following the policy connecting two observed states. Action branches not followed are shown with dashed lines.}
    \label{fig:diagram_beta0}
\end{figure}

As noted in the main text, there is a close connection between GHMs with $\beta = 0$ (equivalently taking the switching probability as $\alpha = 1$) and one-step forward models traditionally used in planning. Concretely, we can say that $\mu^\pi_{\beta = 0}(\cdot \mid x, a)$, for any policy $\pi$, is exactly identical to one-step forward models commonly used in planning. In Figure~\ref{fig:diagram_beta0} (top row), we show examples of samples generated in this setting while varying the number of model compositions. For any $n \ge 1$, this corresponds to planning with a one-step model for $n - 1$ steps and `bootstrapping' afterwards with a value estimate of $V^\pi_\gamma(X')$ using $\mu^\pi_\gamma$. When we set $\beta > 0$, the practice of planning with this model is much the same but each sample from the model may move one or more trajectory steps into the future, and now the policies $(\pi_i)_{i=1}^{n-1}$ determines the nature of this evolution. In Figure~\ref{fig:diagram_beta0} (bottom row), the corresponding illustration for $\beta > 0$ is given. This again can be thought of as planning for $n - 1$ steps and bootstrapping afterwards with a value estimate of $V^\pi_\gamma(X')$, however now each step of unrolling the model is temporally extended and thus potentially moves through many intermediate trajectory states (while following some Markov policy $\pi_i$). This allows for much deeper planning with fewer unroll steps than in the $\beta = 0$ (one-step model) case, potentially reducing problems with error accumulation common in planning with learned models. Finally, note that in the case of $n = 1$, the value of $\beta$ has no effect, and we always sample directly from $\mu^\pi_\gamma(\cdot \mid x_0, a_0)$. 

\clearpage
\section{Further experiments and training details}
\label{sec:appendix:experiment_details}

All experiments were undertaken with Python, using NumPy \citep{harris2020array} and Matplotlib \citep{hunter2007matplotlib} for visualisation. Experiments involving deep neural networks were undertaken with Jax \citep{jax2018github}, specifically making use of the DeepMind Jax ecosystem \citep{deepmind2020jax}.

\subsection{Sparse-reward ant experiment details}
\label{sec:appendix:experiment_details_ant}

\textbf{Environment and task. }
The sparse-reward ant domain is implemented in the MuJoCo framework \citep{todorov2012mujoco} for realistic physics simulations. The agent is a quadrapedal bot resembling an ant, first introduced by \citet{schulman2015high}, with 8 controllable joints. The observation is a 35-dimensional representation of the true agent state, including information about its centre of mass position, velocity, joint angles and angular velocities, heading, etc. The environment has an 8-dimensional action space $[-1, 1]^8$ representing the torque to be applied at each joint. The ant is capable of moving about in an infinite unconstrained two-dimensional arena. Each episode starts with the agent randomly initialised at rest in a 20x20 square centered at the origin. A target is chosen randomly at a distance of between 2 to 4 units, at an angle that lies in the middle 30 degrees of each quadrant. Note that a policy trained to move consistently in a single direction typically progresses at around 0.2-0.4 units of distance per time step. The episode lasts for 150 time steps, or ends early if the ant reaches the target.

\textbf{Policy training.}
We pretrain policies to move consistently in the 4 axis-aligned directions in the arena. In order to train the policies to effectively turn directions when switching between these policies, we jointly pretrain them by randomly switching between them every Uniform(40) steps. Thus the pretraining conditions mimic how the policies will be used during \ggpi planning.
We found training policies jointly on data generated in this manner to be important in learning policies that composed well together. In contrast, training the base policies entirely from on-policy data typically led to poor compositions in preliminary experiments.

The policies are implemented as stochastic Gaussian policies, with a 4-layer MLP with 256 hidden units including layer normalisation~\citep{ba2016layer} and tanh non-linearities. The network outputs the mean/variance of the torque to be applied at each of the 8 joints independently. The policies are pretrained using the component of the velocity in the desired direction as a reward signal, using MPO~\citep{abdolmaleki2018maximum}. The critic network used for MPO is a similar 5-layer MLP with 256 hidden units at each layer, with Layer Norm and tanh non-linearities. The policies are pretrained for 1 million update steps, using the Adam optimiser~\citep{kingma2014adam} with a learning rate of 0.0003.

\textbf{GHM training.}
We implement the GHMs as conditional $\beta$-VAE models with a single latent dimension and a $\beta_{\text{VAE}}=20$. The encoder, prior, and decoder distributions are all assumed to be Gaussian, and implemented as a 3-layer MLP with 128 hidden units in each layer. They each take the concatenated representation of the current agent state and action $(x, a)$ as auxiliary input to be conditioned on. We slightly modify the modelling task to predict the change in the agent state rather than the future state directly, i.e. we model $X_{t+\text{Geom}(1-\beta)} - X_t$ and add this to the state $X_t$ to form a prediction, rather than directly modelling $X_{t+\text{Geom}(1-\beta)}$ itself. We found this to improve performance in terms of negative ELBO slightly.
For each pretrained policy, we train 2 separate GHMs, one with geometric horizon parameter $\beta = 0.8$ and one with $\beta = 0.9$. The GHMs are trained for 500,000 update steps with the CETD loss, using the Adam optimiser and a learning rate of 0.0003.

We performed a light hyperparameter search for: the learning rate between $\{3 \cdot 10^{-5}, 10^{-4}, 3 \cdot 10^{-4}, 10^{-3}\}$; the $\beta_{\text{VAE}}$-parameter in the $\beta$-VAE loss between $\{1, 20, 50, 100\}$; and the VAE latent dimension between $\{1, 8\}$. Performance in terms of negative ELBO was fairly robust across this range of hyperparameters.

Both the policy and GHM training use a distributed actor-learner setup communicating via a uniform replay buffer of size $10^6$. Each learner step uses a batch size of 256 and averages the loss over trajectories of length 20. These settings are conducted without a target specified, and episodes last for between 100 and 140 steps uniformly at random.

\textbf{\ggpi.}
When performing \ggpi to improve on this set of policies, we evaluate with a discount factor of $\gamma=0.9$. Thus, we can consider \gsps that switch with a probability of $\alpha = \nicefrac{1}{9}$, and estimate the GHM of such a policy using the GHMs of its constituent base policies with $\beta=0.8$ and $\beta=0.9$. We estimate the action-value function by sampling from the composed GHM, evaluating the sample under the known non-linear reward function $r$, and averaging over multiple samples per \gsp. For fairness, when comparing \ggpi with $n=1$ and $n=2$, we sample 4 times the sampling budget when evaluating $n=1$ --- thus, both $n=2$ and $n=1$ \ggpi are considering the same number of \emph{total} samples, with $n=1$ actually seeing more samples per policy.

Since this environment has a continuous action space, we cannot evaluate $Q^{\nu}(x, a)$ for all actions; thus, instead we estimate $V^{\nu}(x) = \mathbb{E}_{A \sim \pi_1(\cdot|x)}[Q^{\nu}(x, A)]$, i.e. consider only those actions that we obtain by sampling from the head-policy $\pi_1$ of the \GSP $\nu$, and use this to choose the best \GSP $\nu \in \Pi$ and act according to it per time step.

\subsection{Additional agent visualisations}\label{sec:agent-visualisations}

Figures~\ref{fig:log_viz}~and~\ref{fig:ant_results_viz} show more detailed visualisations of the GPI and depth-2 \ggpi agents.

\begin{figure*}[h]
    \centering
    \includegraphics[scale=0.5]{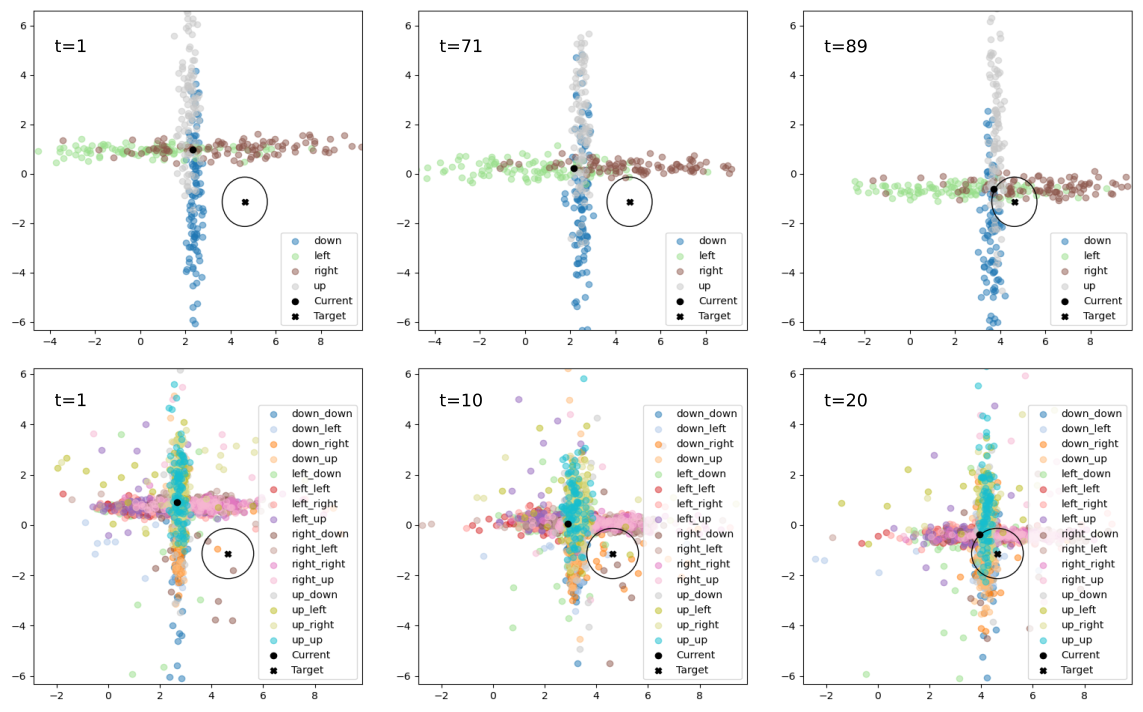}
    \caption{Visualisation of an agent performing GPI (top) and depth-2 \ggpi (bottom) for the same episode initialisation. While \ggpi is immediately able to plan to reach the target and reaches it within 20 steps, standard GPI is unable to do so and spends 70 steps moving randomly before happening to align with the target through pure chance, and then reaching the target. We show the agent and target locations, the boundary outside of which the reward signal is zero, and visualisation of the different agent plans.}
    \label{fig:log_viz}
\end{figure*}

\begin{figure*}[!h]
    \centering
    \begin{minipage}{.5\textwidth}
        \centering
        \includegraphics[width=0.85\linewidth]{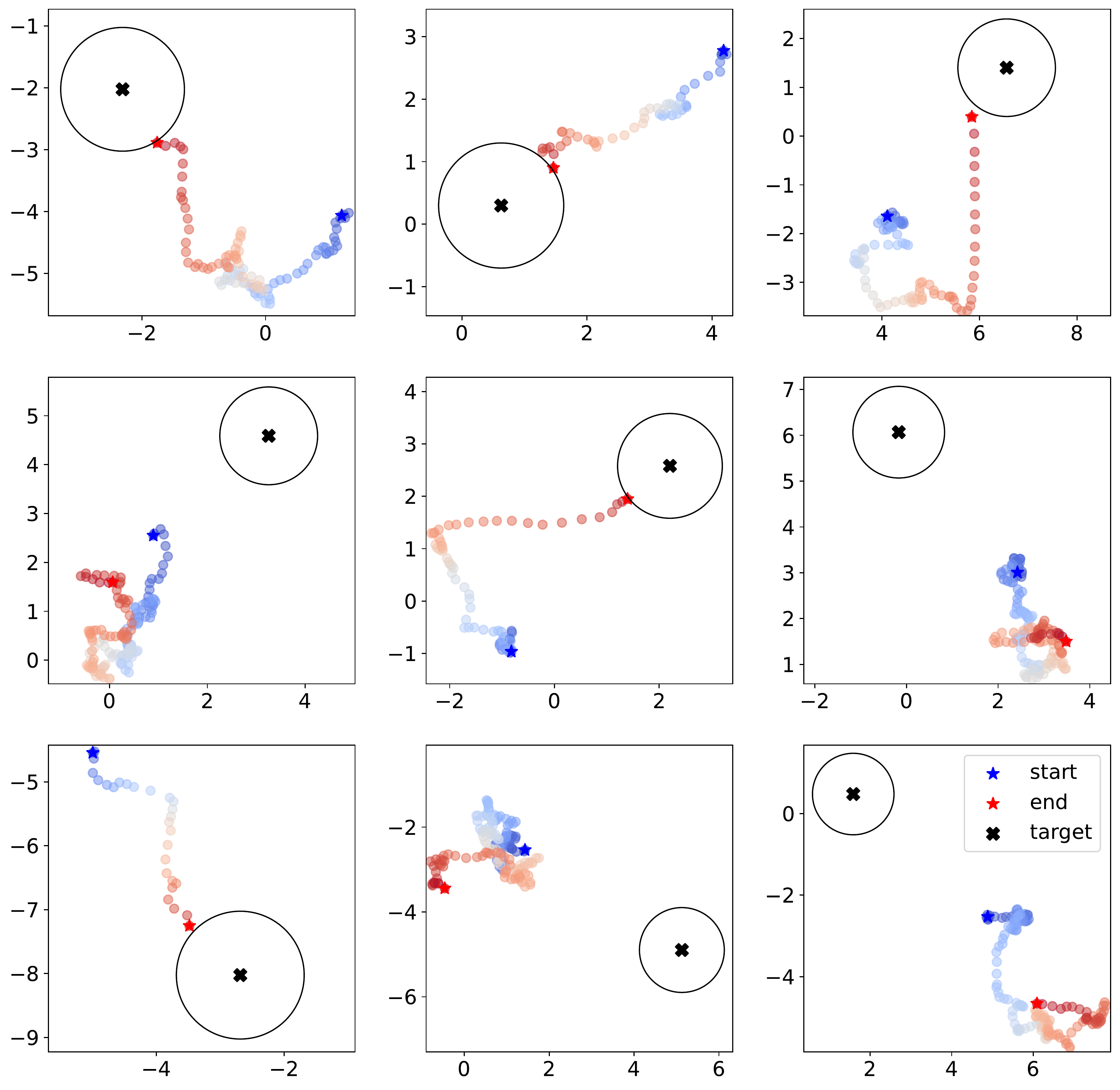}
    \end{minipage}%
    \begin{minipage}{.5\textwidth}
        \centering
        \includegraphics[width=0.85\linewidth]{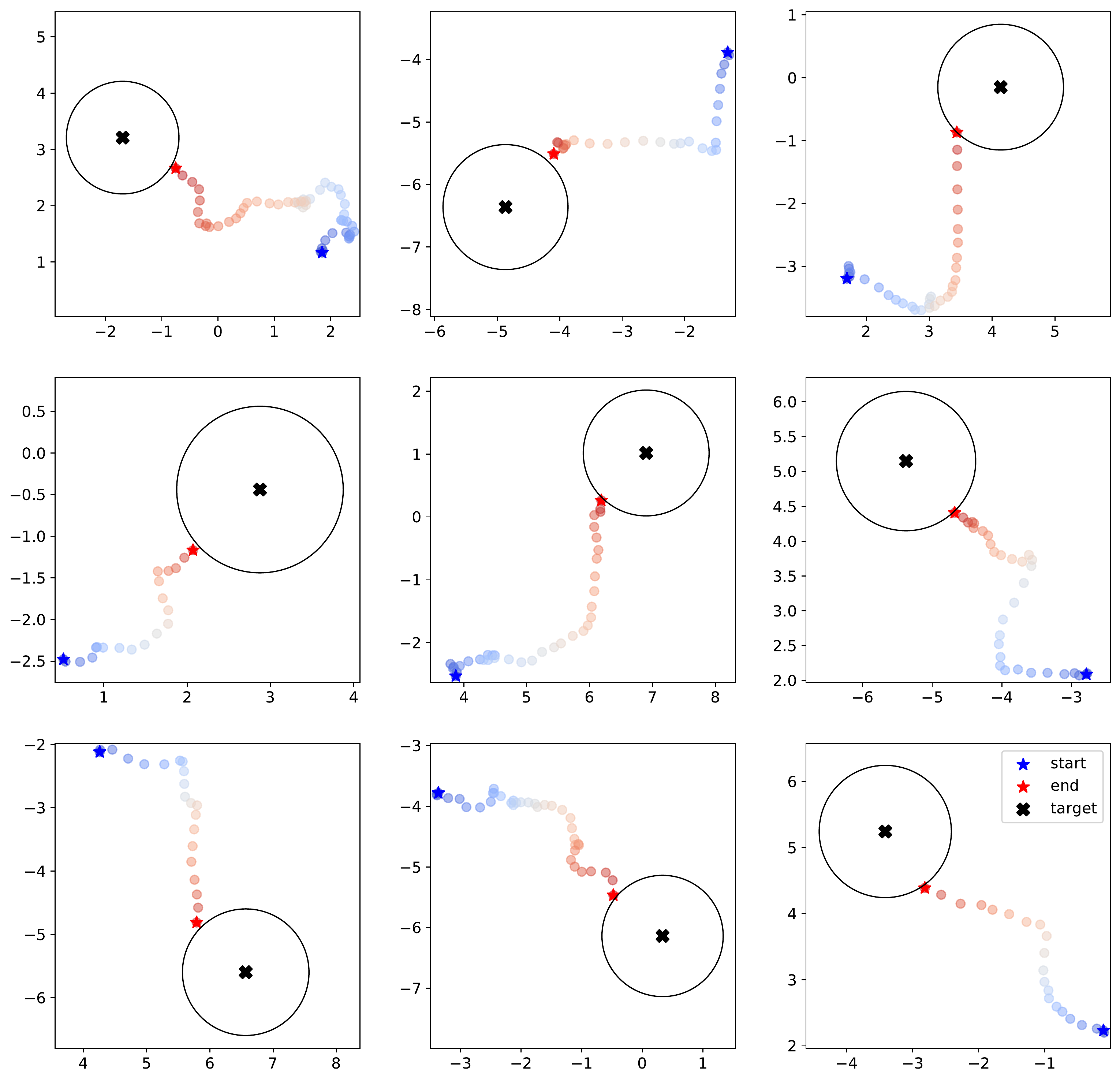}
    \end{minipage}
    \caption{Comparison of standard GPI (left) and $n=2$ \ggpi (right) to navigate towards a goal given sparse rewards. We plot the x-y coordinates of the agent centre of mass, and colour on a gradient from blue to red through the duration of the episode. The target is displayed along with the boundary outside of which the reward signal is zero.}
    \label{fig:ant_results_viz}
\end{figure*}

\subsection{Comparing GHMs to compositions of 1-step models}
\label{sec:appendix:ghm_vs_one_step}

We compare the use of GHMs against a 1-step model unrolled for multiple steps. Concretely, we compare our our VAE-based GHM($\beta$) against a one-step model unrolled for a Geometric($1-\beta$) number of steps. Note that for perfectly trained models, these two distributions would be identical; however, the GHM models show compounding error at train time due to the use of bootstrapping, while the one-step model shows compounding error at evaluation time due to the multi-step composition. We implement the one-step model with identical architecture as the GHM model, equivalent to a GHM($\beta=0$) model.

Figure~\ref{fig:ghm_vs_one_step_plots} shows a comparison of these models for $\beta=0.9$, versus the true geometrically discounted future state distribution obtained through sampling trajectories via simulating the policy in the environment.
The scatterplot on the left shows samples from these models compared with the true distribution, showing a high degree of overlap for the GHM with the true distribution and large errors for the one-step model. The plot on the left measures distance of a sample from these models versus the nearest simulated future state. Though the 1-step model has lower error for very near-term predictions, its error quickly compounds and increases steadily as the prediction horizon increases. Meanwhile, the GHM models make low-error predictions even far into the future.

\begin{figure*}[h]
    \centering
    \begin{minipage}{.5\textwidth}
        \centering
        \includegraphics[scale=0.5]{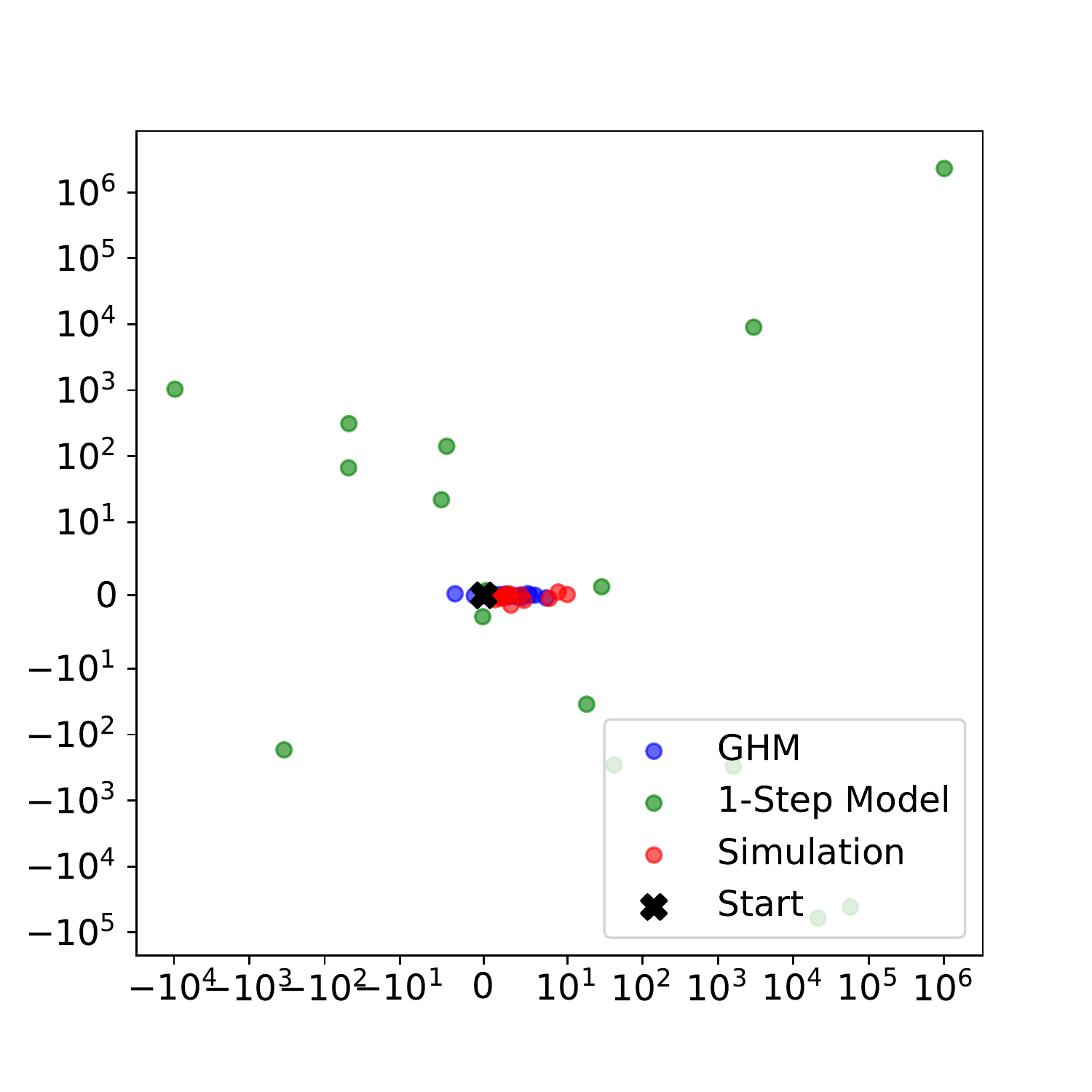}
    \end{minipage}%
    \begin{minipage}{0.5\textwidth}
        \centering
        \includegraphics[scale=0.5]{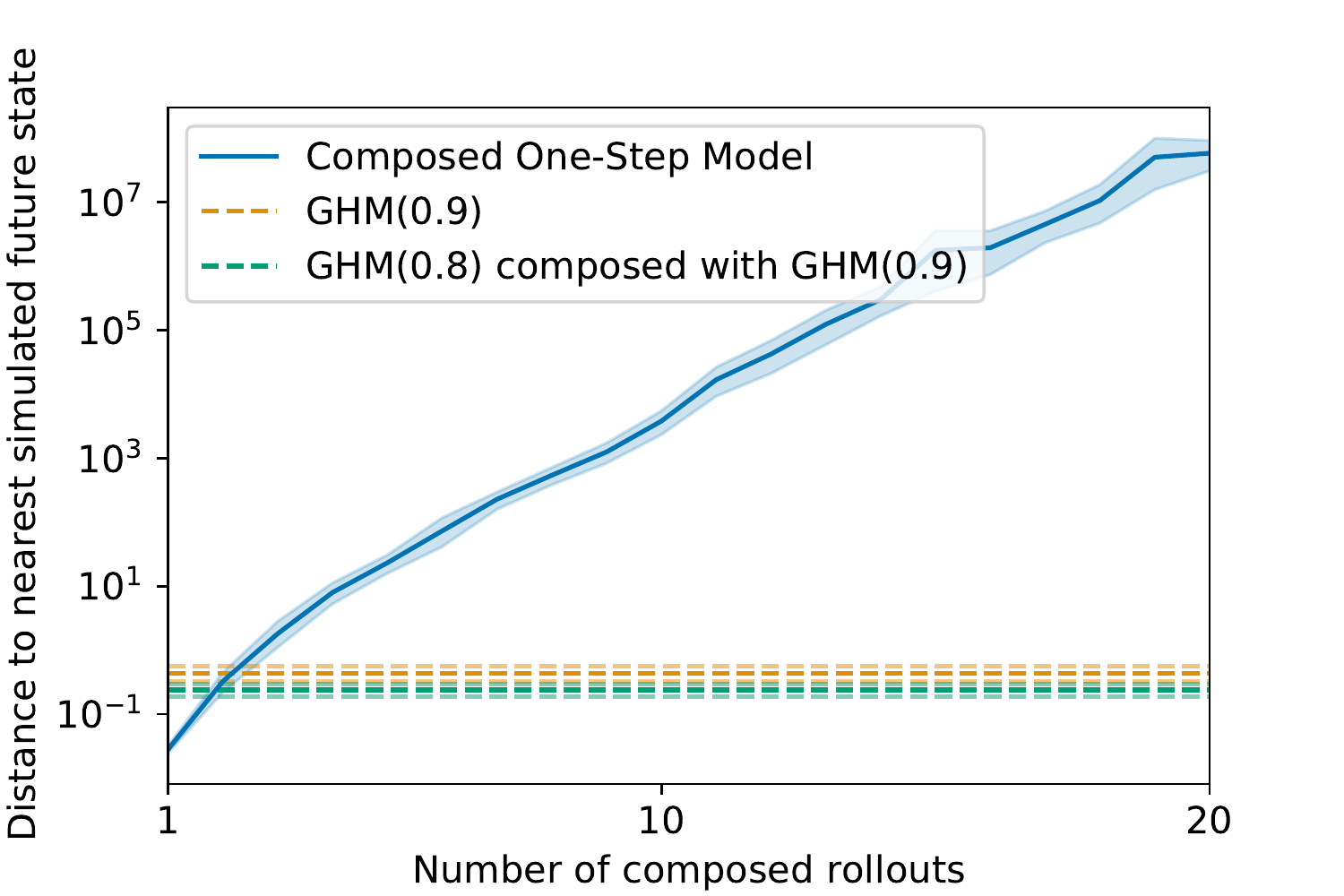}
    \end{minipage}
    \caption{Comparison of a true discounted future state visitation distribution, a learnt GHM($\beta$), and a 1-step model composed for a Geometric($1-\beta$) number of steps, for $\beta=0.9$. The plot on the right is averaged over 100 random seeds.}
    \label{fig:ghm_vs_one_step_plots}
\end{figure*}

\clearpage

\subsection{GHM training using normalising flows}
\label{sec:appendix:ghm_training_exps}

In this section we provide further details of GHM training experiments that inform our choice of VAE-GHMs and the CETD loss in the main paper. One of our main points of comparison is the $L^2$ loss on log-densities, introduced by \citet{janner2020gamma}.

\begin{definition}[Log-$L^2$ temporal-difference (\lltwotd) loss \citep{janner2020gamma}]
    Given an observed transition $(x, a, x')$, the log-$L^2$ bootstrap loss is defined by
    \begin{align*}
        ( \log(\mu(x''|x, a)) - \log((1-\beta) P(x''|x, a) +  \beta \bar{\mu}(x''|x', a')) )^2 \, ,
    \end{align*}
    where $a' \sim \pi(\cdot|x')$, $x'' \sim \mu(\cdot|x', a')$, $\bar{\mu}$ denotes a stop-gradient on $\mu$.
\end{definition}
\citet{janner2020gamma} focus on the \lltwotd bootstrap loss in their experiments. However, as they note, this loss generally leads to an incorrect minimiser, due to the presence of bias (specifically, the averaging of $x', a', x''$ \emph{outside} rather than inside the logarithm).

\begin{figure}[h]
\centering
\includegraphics[scale=0.5]{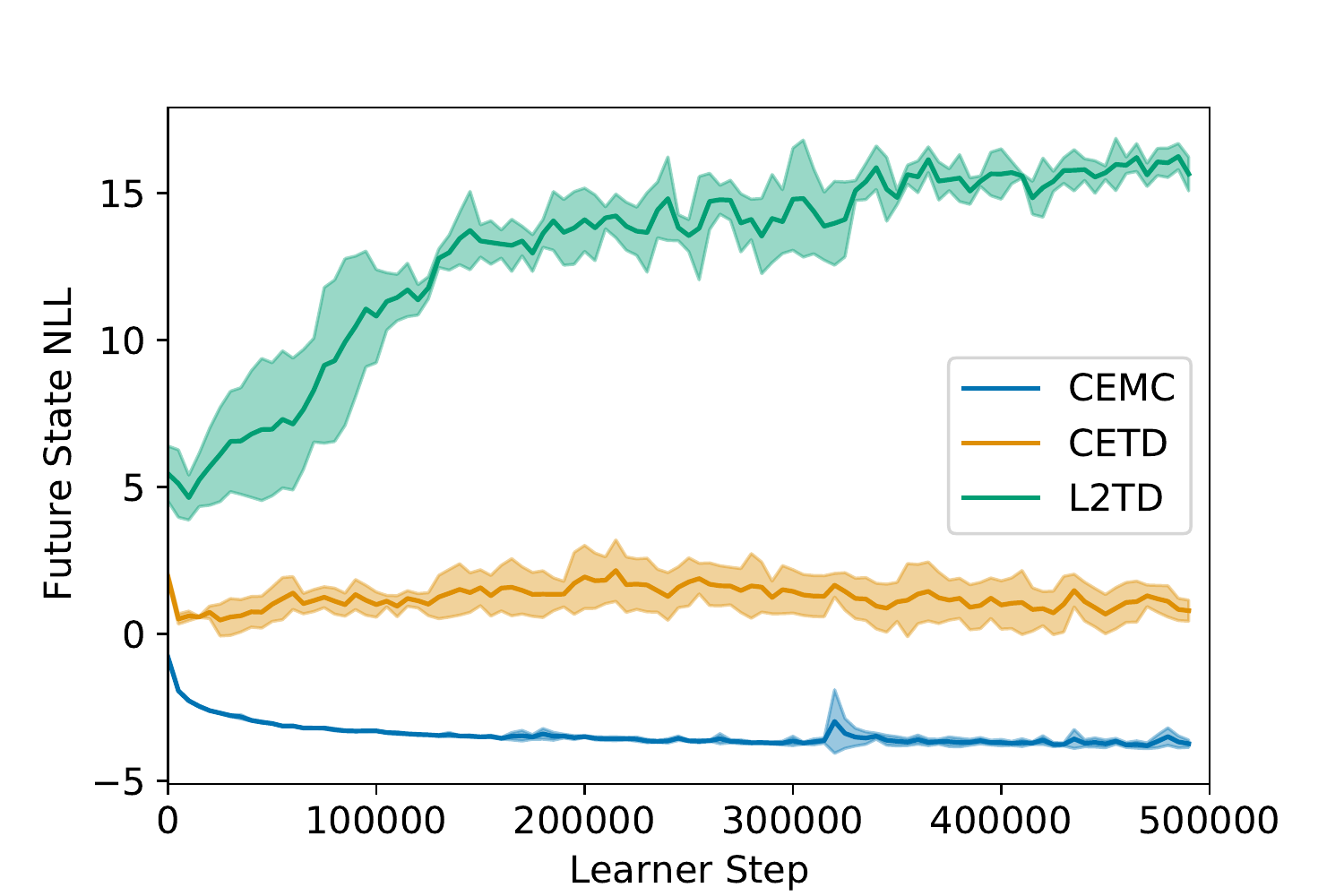}
\caption{Comparison of CEMC, CETD, and \lltwotd losses for training GHMs implemented as normalising flows (left) and $\beta$-VAEs (right). Performance (lower is better) is measured in terms of estimated negative log-likelihood of samples from the true geometric horizon distribution obtained by sampling states from on-policy trajectories. Results are shown averaged over 5 random seeds.}
\label{fig:vae_vs_flows}
\end{figure}

We provide an empirical comparison of the different methods for GHM training introduced in Section~\ref{sec:learning_ghm}. We primarily consider the \lltwotd loss proposed by~\citet{janner2020gamma}, against the CETD loss analysed previously. We also briefly consider training the GHM model $\mu$ by directly sampling a future state from on-policy trajectories at a time sampled according to a Geometric($1-\beta$) distribution into the future, following the CEMC loss introduced in the main paper. Note that the CEMC loss is straightforward to implement and does not require any bootstrapping, but cannot be learned from off-policy samples and thus is less desirable than the other methods.
    
In addition to the comparison using the much simpler VAE models in Section~\ref{sec:ghm_training_exp}, here we compare the losses using normalising flow models~\citep{rezende2015variational} of a similar architecture as suggested by \citet{janner2020gamma}. As these models admit exact density computation, we also can compare against the \lltwotd loss.
Figure~\ref{fig:vae_vs_flows} shows the performance of the different methods for $\beta=0.8$ in terms of the negative log-likelihood of a sample from the true geometric horizon distribution of the $\pi_{\text{right}}$ pretrained policy on the sparse-reward ant environment. Note that the CEMC loss is explicitly optimising for this metric and achieves very strong performance, while the CETD and \lltwotd losses perform much worse. Further, the \lltwotd loss is much less stable than CETD and actually diverges late in training.

When training GHMs using normalising flows, we use a similar architecture to that proposed by ~\citet{janner2020gamma}. We use a normalising flow consisting of 2 coupling layers, each including a batch norm flow~\citep{dinh2016density}, a 1x1 invertible convolution~\citep{kingma2018glow}, and a conditional neural spline~\citep{durkan2019neural}. The neural spline includes a rational quadratic spline with range between -5 to 5, 8 knots, and whose parameters are outputted by an MLP with a single hidden layer of size 256. When training using the \lltwotd loss, we use a target network with a target update period of 200 learner steps to generate the bootstrap targets as suggested by~\citet{janner2020gamma}.

\subsection{Evaluating \ggpi performance for varying GHM training budgets}
\label{sec:appendix:ghm_training_budget}

In our main experiments, we train GHMs using VAEs for 500000 learner steps, which is sufficient to plateau the training ELBO.
We now examine the sensitivity of our proposed method to the training budget afforded to GHM training.

\begin{figure}[h!]
    \centering
    \includegraphics[scale=0.5]{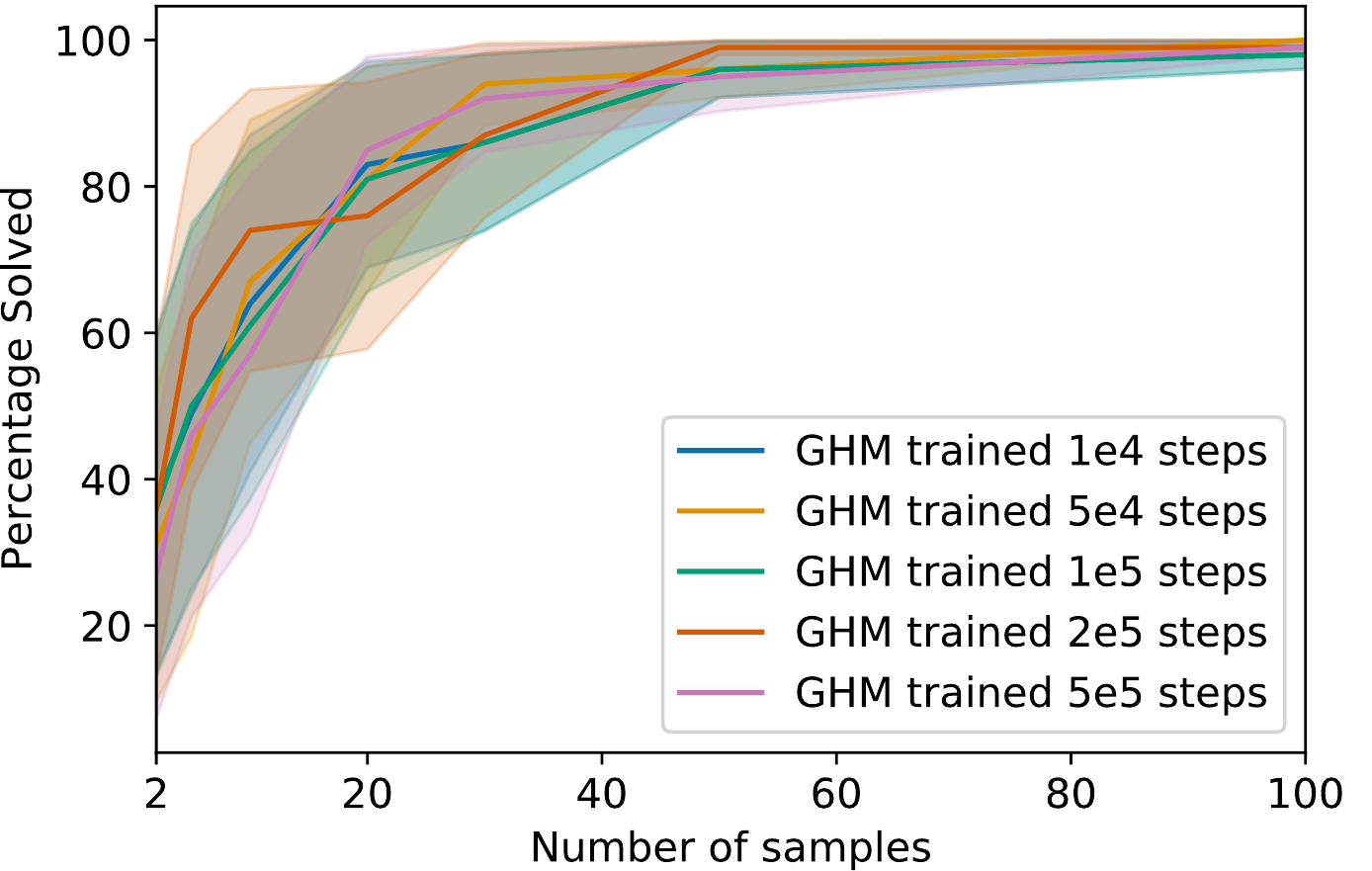}
    \caption{Performance of depth-2 \ggpi on the sparse-reward ant task with snapshots of GHM models taken at various stages through training.}
    \label{fig:ghm_training_steps}
\end{figure}

Figure~\ref{fig:ghm_training_steps} shows that GHMs trained with as few as $10^4$ learner steps (taking only 2 minutes wall-clock time on our distributed training setup described in Appendix~\ref{sec:appendix:experiment_details_ant}) are still successful in planning. Additional preliminary experiments with even fewer learner steps did not result in GHMs useful for planning.

\section{An extension of the main evaluation result}\label{sec:appendix:extensions}\label{sec:smp-eval-general}

For simplicity, in the main paper we presented Theorem~\ref{thm:prop-smp-eval} for action-independent rewards. There is a simple adaptation to this result that applies to general reward functions $r: \mathcal{X} \times \mathcal{A} \rightarrow \mathbb{R}$.

\begin{restatable}{theorem}{thmSMPEvalGeneral}\label{thm:prop-smp-eval-general}
    Consider an MDP with expected reward function $r : \mathcal{X} \times \mathcal{A} \rightarrow \mathbb{R}$, and let $\nu = \pi_1 \overset{\alpha}{\rightarrow} \cdots \overset{\alpha}{\rightarrow} \pi_n$. Writing $\beta =  \gamma(1-\switchprob)$, we have
    \begin{align*}
    & r(x, a) + \frac{\gamma}{1-\gamma} \times \Bigg\lbrack  \sum_{m=1}^{n-1} \frac{1-\gamma}{1 - \beta}\left(\frac{\gamma - \beta}{1-\beta}\right)^{m-1} \left( (1-\alpha)r^{\pi_m}(X^{(m)}) + \alpha r^{\pi_{m+1}}(X^{(m)}) \right)
     + \left(\frac{\gamma - \beta}{1-\beta}\right)^{n-1} r^{\pi_n}(X') \Bigg\rbrack \, ,
    \end{align*}
    where $(X^{(0)}, A^{(0)})=(x, a)$, $X^{(m)} \sim \mu^{\pi_m}_{\beta}(\cdot|X^{(m\indexminus 1)}, A^{(m \indexminus 1)})$, $A^{(m)} \sim \pi_{m+1}(\cdot|X^{(m)})$, $X' \sim \mu^{\pi^{(n)}}_\gamma(\cdot|X^{(n)}, A^{(n)})$, is an unbiased estimator for $Q^\nu_\gamma(x, a)$.
\end{restatable}
\begin{proof}
    The result can be proven as a straightforward corollary of Theorem~\ref{thm:prop-smp-eval}; the distribution of $X^{(m)}$ matches that of $X_T$ conditional on the geometric time $T$ falling between the times of the $(m-1)$\textsuperscript{th} and $m$\textsuperscript{th} switches, while the \GSP is executing policy $\pi_m$. Thus, to know what the distribution of actions should be when evaluating the reward function at this state, we need to know whether the switch happens at the current time step or not. From the memoryless property of the geometric distribution concerned, this probability is precisely $\alpha$. So with a weighting of $1-\alpha$, the reward is evaluated for the policy $\pi_m$, and with a weighting of $\alpha$, the reward is evaluated according to the distribution $\pi_{m+1}$ over policies that will be switched to at this time step.
\end{proof}

\end{document}